\documentclass[12pt]{article}

\RequirePackage[OT1]{fontenc}
\RequirePackage{amsthm,amsmath,amsmath,amsfonts,amssymb,mathabx}
\RequirePackage{natbib}
\RequirePackage[colorlinks,citecolor=blue,urlcolor=blue]{hyperref}
\usepackage{times}
\usepackage{appendix}
\usepackage{multirow}
\usepackage{enumitem}
\usepackage{footmisc}
\usepackage{xcolor}
\usepackage{xr}
\usepackage{url}
\usepackage{hyperref}

\usepackage{graphicx}

\usepackage{authblk}

\allowdisplaybreaks[4]

\newtheorem{definition}{Definition}

\newtheorem{theorem}{Theorem}
\newtheorem{lemma}{Lemma}
\newtheorem*{lemma*}{Lemma}
\newtheorem{corollary}{Corollary}
\newtheorem{assumption}{Assumption}
\newtheorem{proposition}{Proposition}

\theoremstyle{definition}

\usepackage[linesnumbered,boxed]{algorithm2e}
\SetKwRepeat{Do}{do}{while}%
\SetKwInput{KwInput}{Input}
\SetKwInput{KwOutput}{Output}
\SetKwInput{KwFunction}{Function}
\SetKwInput{KwParameters}{Parameters}
\SetKwComment{Comment}{$\triangleright$\ }{}
\SetCommentSty{itshape}

\let\oldnl\nl
\newcommand{\nonl}{\renewcommand{\nl}{\let\nl\oldnl}}
\newcommand{\mat}{\mathbf }
\newcommand{\vct}{\boldsymbol }

\newcommand{\argmax}{\mathrm{argmax}}

\newcommand{\tr}{\mathrm{tr}}

\newcommand{\supp}{\mathrm{supp}}

\newcommand\RR{\mathbb{R}}

\def\cA{\mathcal{A}}

\def\cN{\mathcal{N}}
\def\cO{\mathcal{O}}

\renewcommand{\hat}{\widehat}
\renewcommand{\tilde}{\widetilde}
\renewcommand{\bar}{\overline}

\definecolor{DSgray}{cmyk}{0,1,0,0}

\makeatletter
\def\singlespace{\def\baselinestretch{1}\@normalsize}

\newcommand{\blind}{1}

\begin{document}

\def\spacingset#1{\renewcommand{\baselinestretch}%
{#1}\small\normalsize} \spacingset{1}

\if1\blind
{
\title{Nearly Dimension-Independent Sparse Linear Bandit over Small Action Spaces via Best Subset Selection}
\author{Yining Wang, Yi Chen, Ethan X. Fang, Zhaoran Wang and Runze Li}
\date{}
\maketitle
\begin{singlespace}
\begin{footnotetext}[1]
{
Yining Wang is an assistant professor at Information Systems and Operations Management at the Warrington College of Business of University of Florida, Gainesville, FL 32611, USA. Yi Chen is a Ph.D. student at Department of Industrial Engineering and Management Sciences, Northwestern University, Evanston, IL 60208, USA. Ethan X. Fang is an assistant professor at Department of Statistics, Pennsylvania State University,
University Park, PA 16802-2111, USA. Zhaoran Wang is an assistant professor at Department of Industrial Engineering and Management Sciences, Northwestern University, Evanston, IL 60208, USA.
Runze Li is the Eberly Chair Professor at Department of Statistics, Pennsylvania State University,
University Park, PA 16802-2111, USA.
Fang and Li's research was supported by NSF DMS 1820702, DMS 1953196, and DMS 2015539.
}
\end{footnotetext}
\begin{footnotetext}[2]
    {Wang and Chen contribute equally to this manuscript.
    }
    \end{footnotetext}

\end{singlespace}
}
\fi
\if0\blind
{
\title{ Nearly Dimension-Independent Sparse Linear Bandit over Small Action Spaces via Best Subset Selection}
\author{}
\date{}
\maketitle
} \fi

\begin{abstract}
We consider the stochastic contextual bandit problem under the high dimensional linear model.
We focus on the case where the action space is finite and random, with each action associated with a randomly generated contextual covariate.  This setting finds essential applications such as personalized recommendation, online advertisement, and personalized medicine. However, it is very challenging as we need to balance  exploration and exploitation.
We propose  doubly growing epochs and estimating  the parameter using the best subset selection method, which is easy to implement in practice. This approach achieves $ \tilde{\mathcal{O}}(s\sqrt{T})$ regret with high probability, which is nearly independent in the ``ambient'' regression model dimension $d$. We further attain a sharper $\tilde{\mathcal{O}}(\sqrt{sT})$ regret by using the \textsc{SupLinUCB} framework and match the minimax  lower bound of low-dimensional linear stochastic bandit problems. Finally, we conduct extensive numerical experiments to demonstrate the applicability and robustness of our algorithms empirically.
\end{abstract}


\newpage
\spacingset{1.45} 

\pagestyle{plain}


\section{Introduction}

Contextual bandit problems receive significant attention over the
past years in different communities, such as statistics,
operations research, and computer science.  This class of problems
studies how to make optimal sequential decisions with new
information in different settings, where we aim to maximize our
accumulative reward, and we iteratively improve our policy given
newly observed results. It finds many important modern applications
such as personalized recommendation
\citep{li2010contextual,li2011unbiased}, online advertising
\citep{krause2011contextual}, cost-sensitive classification
\citep{agarwal2014taming}, and personalized medicine
\citep{goldenshluger2013linear,bastani2015online,keyvanshokooh2019contextual}.
In most contextual bandit problems, at each iteration, we first
obtain some new information. Then, we  take  action  based on a
certain policy and observe a new reward. Our goal is to maximize
the total reward, where we iteratively improve our policy. This
paper is concerned with the linear stochastic bandit model, one of the most fundamental models in
contextual bandit problems.
We model the expected reward at each time period
as a linear function of some random information depending on our
action.  This model receives considerable attention
\citep{auer2002using,abe2003reinforcement,dani2008stochastic,rusmevichientong2010linearly,chu2011contextual},
and as we will discuss in more detail, it naturally finds
applications in optimal sequential treatment regimes.

In a linear stochastic bandit model, at each time period $t\in\{1,2,\cdots,T\}$,
we are given some action space $\mathcal A_t$.
Here each feasible action  $i\in\mathcal A_t$ is  associated with a $d$-dimensional context covariate $X_{t,i}\in\mathbb \RR^d$ that is known before any action is taken.
Next, we take an  action $i_t\in\mathcal A_t$ and then observe a {reward} $Y_t\in\mathbb R$.
Assume that the reward follows   a linear model
\begin{equation}
Y_t = \langle X_{t,i_t},\theta^*\rangle + \varepsilon_t,
\label{eq:model}
\end{equation}
where $\theta^*\in\mathbb R^d$ is an unknown $d$-dimensional
parameter, and $\{\varepsilon_t\}_{t=1}^T$ are noises. Without
loss of generality, we assume that we aim to maximize the total
reward. We measure the performance of the selected sequence of
actions $\{i_t\}_{t=1}^T$ by the accumulated regret
\begin{equation}  \label{eqn:regret}
R_T(\{i_t\};\theta^*):=\sum_{t=1}^T \max_{i\in \mathcal A_t} \langle X_{t,i},\theta^*\rangle - \langle X_{t,i_t},\theta^*\rangle.
\end{equation}
Essentially, the regret measures the  difference between the
``best" we can achieve if we know the true parameter $\theta^*$
and the ``noiseless" reward we get. It is clear that regret is
always nonnegative, and our goal is to minimize the regret.



Meanwhile, due to the advance of technology, there are many modern
applications where we encounter the high-dimensionality issue,
i.e., the dimension $d$ of the covariate  is large. Note that here
the total number of iterations $T$ is somewhat equivalent to the
sample size, which is the number of pieces of information we are
able to ``learn" the true parameter $\theta^*$. Such a
high-dimensionality issue presents in practice. For example, given
the genomic information of some patients, we aim to find a policy
that assigns each patient to the best treatment for him/her.
It is usually the case that the number of patients is much smaller
than the dimension of the covariate. In this paper, we consider
the linear stochastic contextual bandit problem under such a
high-dimensional setting,  where the parameter $\theta^*$ is of
high dimension that $d\gg T$. We assume that $\theta^*$ is sparse
that only at most $s\ll d$ components of $\theta^*$ are non-zero.
This assumption is commonly imposed in high-dimensional statistics
and signal processing literature
\citep{donoho2006compressed,buhlmann2011statistics}.

In addition, for the  action spaces $\{\mathcal{A}_t\}_{t=1}^T$,
it is known in literature
\citep{dani2008stochastic,shamir2015complexity,szepesvari2016banditlog}
that if there are infinitely many feasible actions at each
iteration, the minimax lower bound is of order
$\tilde{\mathcal{O}}(\sqrt{sdT})$, which does not solve the curse
of dimensionality. We assume that the action spaces $\{\mathcal
A_t\}_{t=1}^T$ are {finite, small}, and {random}. In particular,
we assume that for all $t$, $|\mathcal A_t|=k\ll d$, and each
action in $\mathcal A_t$ is associated with independently randomly
generated contextual covariates. In most practical applications,
this finite action space setting is naturally satisfied. For
example, in the treatment example, there are usually only a small
number of feasible treatments available. We refer the readers to
Section~\ref{sec:assumptions} for a complete description and
discussion of the assumptions.

\subsection{Literature Review}

We  briefly review existing works on linear stochastic bandit
problems under both low-dimensional and high-dimensional settings.
Under the classical low-dimensional setting, \cite{auer2002using}
pioneers the use of upper-confidence-bound (UCB) type algorithms
for the linear stochastic bandit, which is one of the most
powerful and fundamental algorithms for this class of problems,
and is also considered in \cite{chu2011contextual}.
\cite{dani2008stochastic,rusmevichientong2010linearly} study
linear stochastic bandit problems with large or infinite action
spaces, and derive corresponding lower bounds. Under the
high-dimensional setting, where we assume that $\theta^*$ is
sparse, when the action spaces $\mathcal A_t$'s are hyper-cube
spaces $[-1,1]^d$, \cite{lattimore2015linear} develop the SETC
algorithm that attains nearly dimension-independent regret bounds.
We point out that this algorithm exploits the unique structure of
hyper-cubes and is unlikely to be applicable for general action
spaces including the ones of our interest where the action spaces
are finite. \cite{abbasi2012online} consider a UCB-type algorithm
for general action sets and obtain a regret upper bound of
$\widetilde \cO(\sqrt{sdT})$, which depends {polynomially} on the
ambient dimension $d$. Note that here the $\widetilde{\cO}(\cdot)$
notation is the big-O notation that ignores all logarithmic
factors. \cite{carpentier2012bandit} consider a different reward
model and obtain an $\widetilde \cO(\|\theta^*\|_2^2s\sqrt{T})$
regret upper bound.
\cite{goldenshluger2013linear,bastani2015online} study a variant
of the linear stochastic bandit problem in which only one
contextual covariate $X_t$ is observed at each time period $t$,
while each action $i_t$ corresponds to a different unknown model
$\theta_{i_t}^*$. We point out that this model is a special case
of our model \eqref{eq:model}, as discussed in
\cite{foster2018practical}.

Another closely related problem is the {online sparse prediction}
problem \citep{gerchinovitz2013sparsity,foster2016online}, in
which sequential predictions $\widehat Y_t$'s  of $Y_t=\langle
X_t,\theta^*\rangle + \varepsilon_t$ are of the interest, and the
regret is measured in mean-square error $\sum_t |\widehat
Y_t-Y_t|^2$. It can be further generalized to online
empirical-risk minimization \citep{langford2009sparse} or even the
more general {derivative-free/bandit convex optimization}
\citep{nemirovsky1983problem,flaxman2005online,agarwal2010optimal,
shamir2013complexity,besbes2015non,bubeck2017kernel,wang2017stochastic}.
{ Most existing works along this direction have continuous
(infinite) action spaces $\{\mathcal A_t\}$. They allow
small-perturbation type methods like estimating gradient descent.}
We summarize the existing theoretical results on stochastic linear
bandit problems in Table \ref{tab:summary}.

\begin{table}[t]
\caption{Summary of existing results on low- and high-dimensional linear stochastic bandit.
Contextual covariates $X$ regression model $\theta^*$ are in $\mathbb R^d$.
Logarithmic terms are dropped.}
\vskip 0.1in
\centering
\scalebox{0.9}{
\begin{tabular}{ccccc}
\hline
Example paper& Model& Action spaces& Model constrains& Regret\\
\hline
\cite{auer2002using}& $Y=\langle X,\theta^*\rangle+\varepsilon$& $|\mathcal A_t|\leq N$& $|\langle X,\theta^*\rangle|\leq 1$& $\widetilde \cO(\sqrt{dT})$\\
\cite{dani2008stochastic}& $Y=\langle X,\theta^*\rangle+\varepsilon$& $\mathbb R^d$& $|\langle X,\theta^*\rangle|\leq 1$& $\widetilde \cO(d\sqrt{T})$\\
\cite{lattimore2015linear}& $Y=\langle X,\theta^*\rangle+\varepsilon$& $[-1,1]^d$& $\|\theta^*\|_0\leq s,\|\theta^*\|_1\leq 1$& $\widetilde \cO(s\sqrt{T})$\\
\cite{abbasi2012online}& $Y=\langle X,\theta^*\rangle + \varepsilon$& $\mathbb R^d$& $\|\theta^*\|_0\leq s,|\langle X,\theta^*\rangle|\leq 1$& $\widetilde \cO(\sqrt{sdT})$ \\
\cite{carpentier2012bandit}& $Y=\langle X,\theta^*\rangle +\varepsilon $& $\{x: \|x\|_2\leq 1\}$& $\|\theta^*\|_0\leq s$& $\widetilde \cO(\|\theta^*\|_2^2s\sqrt{T})$\\
\textbf{This paper}& $Y=\langle X,\theta^*\rangle + \varepsilon$& $|\mathcal A_t|\leq k$\textsuperscript{$\dagger$}& $\|\theta^*\|_0\leq s,\|\theta^*\|_2\leq r$& $\widetilde \cO(\sqrt{sT})$\\
\hline
\end{tabular}
}
\begin{flushleft}
\vskip-0.1in
{\footnotesize \textsuperscript{$\dagger$}Contextual covariates $X_{t,i}$ associated with each $i\in\mathcal A_t$ are randomly generated.}
\end{flushleft}
\label{tab:summary}
\end{table}

From the application perspective of finding the optimal treatment
regime, existing literatures focus on achieving the optimality
through batch settings. { General approaches  include model-based
methods such as Q-learning \citep{watkins1992q,murphy2003optimal,
moodie2007demystifying,chakraborty2010inference,goldberg2012q,song2015penalized}
and A-learning \citep{robins2000marginal,murphy2005experimental}
and model-free policy search methods
\citep{robins2008estimation,orellana2010dynamic,orellana2010dynamic2,zhang2012robust,zhao2012estimating,zhao2014doubly}.}
These methods are all developed based on batch settings where we
use the whole dataset to estimate the optimal treatment regime.
This approach is applicable after the clinical trial is completed,
or when the observational dataset is fully available.  However,
the batch setting approach is not applicable when it is emerging
to identify the optimal treatment regime. For a contemporary
example, during the recent outbreak of coronavirus disease
(COVID-19), it is extremely important to quickly identify  the
optimal or nearly optimal treatment regime to assign each patient
to the best treatment among a few choices. However, since the
disease is novel, there is no or very little historical data.
Thus, the batch setting approaches mentioned above are not
applicable. On the other hand,  our model naturally provides a
``learn while optimizing" alternative approach to sequentially
improve the policy/treatment regime. This  provides an important
motivating application of the proposed method.



\subsection{Major Contributions}
In this paper, we propose new algorithms, which iteratively learn the parameter $\theta^*$ while optimizing the regret.  Our algorithms use the ``doubling trick" and  modern optimization techniques,
which carefully balance the randomization for exploration to fully learn the parameter and maximizing the reward to achieve the near-optimal regret. In particular, our algorithms fall under the general  UCB-type algorithms \citep{lai1985asymptotically,auer2002finite}. Briefly speaking, we start with some pure exploration periods that we take random actions uniformly, and we estimate the parameter. Then, for the next certain periods, which we call an epoch, we take the action at each time period by optimizing some upper confidence bands using the previous estimator. At the end of each epoch, we renew the estimator using new information. We then enter the next epoch using the new estimator and renew the estimator at the end of the next epoch. We repeat this until the $T$-th time period.

The high-dimensional regime (i.e., $d\gg T$) poses significant challenges in our setting, which cannot be solved by exiting works.
First, unlike in the low-dimensional regime where ordinary least squares (OLS) always admits closed-form solutions and error bounds, in the high-dimensional regime, most existing methods like the Lasso \citep{tibshirani1996regression} or the Dantzig selector \citep{candes2007dantzig}
require the sample covariance matrix to satisfy certain restricted eigenvalue''conditions \citep{bickel2009simultaneous},
which do not hold under our setting for sequentially selected covariates.
Additionally, our action spaces $\{\mathcal A_t\}$ are finite. This rules out several existing algorithms, including the SETC method \citep{lattimore2015linear}
that exploits the specific structure of hyper-cube actions sets
and finite-difference type algorithms in stochastic sparse convex optimization \citep{wang2017stochastic,balasubramanian2018zeroth}.
We adopt the best subset selector estimator \citep{miller2002subset} to derive valid confidence bands only using ill-conditioned sample covariance matrices. Note that while the optimization for best subset selection is  NP-hard in theory \citep{natarajan1995sparse}, by the tremendous progress of modern optimization, solving such problems is practically efficient as discussed in \cite{pilanci2015sparse,bertsimas2016best}. { In addition, the renewed estimator may correlate with the previous one. This decreases the efficiency. We let the epoch sizes grow exponentially, which is known as the ``doubling trick" \citep{auer1995gambling}. This ``removes" the correlation between recovered support sets by best subset regression.} Our theoretical analysis is also motivated from some known analytical frameworks such as the elliptical potential lemma \citep{abbasi2011improved} and the SupLinUCB framework
 \citep{auer2002using} in order to obtain sharp regret bounds.

We summarize our main theoretical contribution in the following
corollary, which is essentially a simplified version of
Theorem~\ref{thm:suplinucb} in Section 3.3.

\begin{corollary} \label{cor:main}
\textit{We assume that $|\mathcal A_t|=k$,
$X_{t,i}\overset{i.i.d.}{\sim}\mathcal N(0, I_d)$, and
$\varepsilon_t\overset{i.i.d.}{\sim}\mathcal N(0,1)$, for each
$1\le t\le T$. Then there exists a policy $\pi$ such that with
probability at least $1-\delta$, for all $\|\theta^*\|_0\leq s$
and $\|\theta^*\|\leq r$,
$$
R_T(\{i_t\};\theta^*) \lesssim \sqrt{sT}\cdot \mathrm{polylog}(k,d,T,r),
$$
where $\mathrm{polylog}$ denotes a polynomial of logarithmic
factors. }
\end{corollary}
Note that this corollary holds even if $T\ll d$.
Theorem~\ref{thm:suplinucb} provides more general results than
this corollary, and can be applied for a broader family of
distributions for contextual covariates $\{X_{t,i}\}$. A  simpler
and more implementable algorithm, with a weaker regret guarantee
of $\widetilde \cO(s\sqrt{T})$, is given in Algorithm
\ref{alg:main} and Theorem \ref{thm:main-upper}, respectively. The
form of regret guarantee in this Corollary \ref{cor:main} looks
particularly similar to its low-dimensional counterparts in Table
\ref{tab:summary}, with the ambient dimension $d$ replaced by the
intrinsic dimension~$s$.

\vspace{4pt}

\noindent {\bf Notations.} Throughout this paper, for an integer
$n$, we use $[n]$ to denote the set $\{1,2,\ldots,n\}$. We use
$\|\cdot \|_1, \| \cdot \|, \| \cdot \|_{\infty}$ to denote the
$\ell_1,\ell_2$ and $\ell_\infty$ norms of vector, respectively.
Given a matrix $\mat A $, we use $\| \cdot\|_{\mat A} $ to denote
the $\ell_2$ norm weighted by $\mat A$. Specifically, we have $\|X
\|_{\mat A} :=\sqrt{X^\top \mat A X}$. We also use $\langle
\cdot,\cdot \rangle $ to denote the inner product of two vectors.
Given a set $S \subseteq [d]$,  $ S^c$ is it complement and $ |S|
$ denotes the cardinality. Given a $d$-dimensional vector $X$, we
use $ [X]_i$ to denote its $i$-th coordinate. We also use
$\text{supp}(X)$ to represent the support of $X$, which is the
collection of indices corresponding to nonzero coordinates.
Furthermore, we use $[X]_S=([X]_i)_{i\in S}$ to denote the
restriction of $X$ on~$S$, which is a $|S|$-dimensional vector.
Similarly, for a $d\times d$ matrix $\mat A=([\mat
A]_{ij})_{i,j\in [d]} \in  \RR^{d\times d}$, we denote by $[\mat
A]_{SS'}=([\mat A]_{jk})_{j\in S,k \in S'}$ the restriction of $A$
on $S \times S'$, which is a $|S|\times |S'|$ matrix. When $S=S'$,
we further abbreviate $[\mat A]_S=[\mat A]_{SS} $  . In addition,
given two sequences of nonnegative real numbers $\{a_n\}_{n\geq
1}$ and $\{b_n\}_{n\geq 1}$, $ a_n \lesssim b_n $ and $a_n \gtrsim
b_n$  mean that there exists an absolute constant $0<C<\infty$
such that $a_n \le C b_n$ and $a_n \ge C b_n $ for all $n$,
respectively. We also abbreviate $ a_n \asymp b_n $, if $a_n
\lesssim b_n  $ and $a_n \gtrsim b_n $ hold simultaneously. {We
say that a random event $ \mathcal{E}$ holds with probability at
least $1-\delta$, if there exists some absolute constant $C$ such
that the probability of $\mathcal{E}$ is larger than $1-C\delta$.}
Finally, we remark that arm, action, and treatment all refer to
actions in different applications. We also denote by $i_t$ the
action taken in period $t$ and $X_t=X_{t,i_t}$ the associated
covariate.

\vspace{4pt}

\noindent{\bf Paper Organization.} The rest of this paper is
organized as follows. We  present our  algorithms in Section
\ref{method}. Then we establish our theoretical results in Section
\ref{theory}, under certain technical conditions. We  conduct
extensive numerical experiments using both synthetic and real
datasets to investigate the performance of our methods in Section
\ref{experiment}. We conclude the paper in Section \ref{sec:con}.


\section{Methodologies}  \label{method}
In this section, we present the proposed methods to solve the linear stochastic  bandit problem where we aim to minimize the regret defined in \eqref{eqn:regret}. In Section  \ref{algorithm}, we first introduce an algorithm called ``\textsc{Sparse-LinUCB}'' (SLUCB) as summarized in Alg.~\ref{alg1}, which can be efficiently implemented, and demonstrate the core idea of our algorithmic design.
The SLUCB algorithm is a variant of the celebrated {LinUCB} algorithm \citep{chu2011contextual} for classical linear contextual bandit problems.  The SLUCB algorithm is intuitive and easy to implement. However, we cannot derive the optimal upper bound for the regret.  To close this  gap, we further propose a more sophisticated algorithm called ``\textsc{Sparse-SupLinUCB}'' (SSUCB) (Alg.~\ref{alg:suplinucb}) in   Section \ref{reg_bound1}.  In comparison with the SLUCB algorithm, the SSUCB algorithm  constructs the upper confidence bands through sequentially  selecting  historical data and achieves the optimal rate (up to logarithmic factors).

\subsection{\textsc{Sparse-LinUCB} Algorithm} \label{algorithm}
As we mentioned above, our algorithm is inspired by the standard \textsc{LinUCB} algorithm, which balances  the tradeoff between exploration and exploitation following the principle of  ``optimism in the face of uncertainty''. In particular, the \textsc{LinUCB} algorithm repeatedly constructs upper confidence bands for the potential rewards of the actions. The upper confidence bands  are optimistic estimators. We  then pick the action associated with the largest upper confidence band. This leads to the optimal regret under the low-dimensional setting. However, under the high-dimensional setting, directly applying the \textsc{LinUCB} algorithm  incurs some suboptimal regret since we only get lose confidence bands under the high-dimensional regime. Thus,  it is desirable to construct tight confidence bands under the high-dimensional and sparse setting to achieve the optimal regret.

 Inspired by the remarkable success of the best subset selection (BSS) in  high-dimensional  regression problems, we propose to incorporate  this powerful tool into the \textsc{LinUCB} algorithm. Meanwhile, since the BSS procedure is computationally expensive, it is impractical  to execute the BSS method during every time period. In addition, as we will discuss later, it is not necessary.

 Before we present our algorithms, we briefly discuss the support of parameters. Given a $d$-dimensional vector $\theta$, we denote by $\supp(\theta)$  the support set of $\theta$, which is the collection of dimensions of $ \theta$ with nonzero coordinates that
$$
\supp(\theta)=\big \{j \in [d]: [\theta]_j \ne 0 \big\}.
$$
This definition agrees with that of most literatures. However, for the BSS procedure, it is desirable to generalize this definition. We propose the concept of ``generalized support'' as follows.

\begin{definition} [Generalized Support]
    Given a $d$-dimensional vector $\theta$, we call a subset $S\subseteq [d]$ the generalized support of $\theta$ and denote it by $ \supp^+(\theta)$, if
    $$
     [\theta]_j =0,\ \forall j \not\in S.
    $$
\end{definition}
The generalized support $ \supp^+(\theta) $ is a relaxation of the normal support, since any support is a generalized support (but not vice versa). Moreover, the generalized support is not unique. Any subset including the support is a valid generalized support.

We distinguish the difference between support and generalized support in order to define the best subset selection without causing confusion. For example, we consider a linear model $\theta^* \in \RR^d, X_t \in \RR^d$, and $Y_t=\langle X_t, \theta^* \rangle +\varepsilon_t $. Calculating the ordinary least square estimator restricted on the generalized support $S \in [d]$, which is denoted by
\begin{align*}
    \hat{\theta}= \operatorname*{argmin}_{\supp^+(\theta)=S} \Big\{ \sum_{t}  \big| Y_{t}-\langle X_{t},\theta \rangle  \big|^2  \Big \},
\end{align*}
means that we consider a low-dimensional model only using the information in  $S$ and set the coordinates of estimator except in $S$ as zeros. Formally, let ${[\hat\theta]}_S \in \RR^{|S|}$ be
$$
{[\hat\theta]}_S= \operatorname*{argmin}_{\phi \in \RR^{|S|} } \Big\{ \sum_{t}  \big| Y_{t}-\langle [X_t]_{S},\phi \rangle  \big|^2  \Big \}.
$$
Then we have
$$
\big[\hat \theta\big]_j=\big[ \hat{[\theta]}_S\big]_j,\ j\in S;\ \big[\hat \theta\big]_j=0,\ j\not\in S.
$$
Since we do not guarantee $[\hat \theta\big]_j \ne 0,\  \forall j \in S$, we call  $S$ the generalized support   instead of~support.

We are ready to present the details of SLUCB algorithm now. Our algorithm works as follows. We first apply the ``doubling trick'', which partitions the whole $T$ decision periods into several consecutive epochs such that the lengths of the epochs increase doubly. We only implement the BSS procedure at the end of each epoch to recover the support of  the parameter $\theta^*$. Within each epoch, we start  with a minimal epoch length of $n_0$ for pure exploration, which is called the pure-exploration-stage. At each time period $t$ of this stage, we pick an action within $\cA_t$ uniformly at random. Intuitively speaking, in pure-exploration-stage, we explore the unknown parameters at all possible directions through uniform randomization. From a technical perspective, the pure-exploration-stage is desirable since it maintains the smallest sparse eigenvalue of the empirical design matrix lower bounded away from zero and hence, help us obtain an efficient estimator for $\theta^*$. After the pure-exploration-stage, we keep the estimated support, say of $s$ nonzero components, from the previous epoch unchanged, and treat the problem as an $s$-dimensional regression. Specifically, at each time period, we use the ridge estimator with penalty weight $\lambda$ to estimate $\theta^*$ and construct corresponding confidence bands to help us make decisions in the next time period.

In summary, we partition the  time horizon $[T]$ into consecutive epochs $\{E_{\tau}\}_{\tau=1}^{\tilde{\tau}} $ such that
$$
[T]=\bigcup_{\tau=1}^{\tilde{\tau}} E_{\tau},\ |E_{\tau}|=\min \{ 2^{\tau},n_0\}.
$$
The length of the last epoch $E_{\tilde{\tau}}$ may be less than $2^{\tilde \tau}$ or $n_0$. By definition, the number of epochs $\tilde \tau \asymp \log(T)$. Hence, in the SLUCB algorithm, we run the BSS procedure $\tilde{\cO}(\log (T))$ times. In our later simulations studies, we find that this is practical for moderately large dimensions.

Next, we introduce the details of constructing upper confidence bands in the  SLUCB algorithm. We assume that at  period $t \in E_\tau$, we pick action $i_t\in \cA_t$ and observe the associated covariate $X_{t,i_t}$ and reward $Y_t$. We also abbreviate $ X_{t,i_t} $ as $ X_t $, if there is no confusion. We denote by $ \hat \theta_{\tau-1,\lambda} $ the BSS estimator for the true parameter $\theta^*$ at the end of previous epoch $E_{\tau-1}$, and let $$ S_{\tau-1}=\supp^+{(\hat \theta_{\tau-1,\lambda})}$$ be its generalized support, i.e., the generalized support recovered by epoch $E_{\tau-1}$. For periods in the pure-exploration-stage of $E_\tau$, we pick arms uniformly at random and do not update the estimator of $\theta^*$. Beyond the pure-exploration-stage, we estimate $ \theta^*$ by a ridge estimator. Let $ \hat \theta^{t-1}_{\tau,\lambda}$ be the most updated ridge  estimator of $\theta^*$ by $t \in E_\tau$, which is estimated by restricting its generalized support on $S_{\tau-1}$  using data $\{ X_{t'},Y_{t'}\}_{t'\in E^{t-1}_{\tau}}$. In particular, all components of $ \hat \theta^{t-1}_{\tau,\lambda}$ outside $S_{\tau-1}$ are set as zeros and
\begin{align*}
    \hat{\theta}^{t-1}_{\tau,\lambda}&=\operatorname*{argmin}_{\supp^+(\theta)=S_{\tau-1}} \Big\{ \sum_{t' \in E^{t-1}_{\tau}}  \big| Y_{t'}-\langle X_{t'},\theta \rangle  \big|^2 +\lambda \|\theta\|^2 \Big \}.
\end{align*}
Given $\hat \theta^{t-1}_{\tau,\lambda}$, we calculate the upper confidence band of potential reward $\langle X_{t,i}, \theta^* \rangle$ for each possible action $i \in \mathcal{A}_t$. In particular, let the tuning parameters $\alpha$ and $\beta$ be
\begin{align*}
    \alpha= \sigma \nu \sqrt{s\tilde{\tau}}\cdot \log(\sigma \rho^{-1}kTd/\delta) ,\ \beta=r\sigma\cdot \sqrt{s\log(kTd/\delta)},
\end{align*}
which correspond to the confidence level and upper estimate of potential reward, respectively. Then we calculate the upper confidence band associated with action $i$ as
$$
\min \Big\{  \beta,\big \langle X_{t,i}, \hat\theta_{\tau,\lambda}^{t-1}\big \rangle + \alpha \cdot \Big(\sigma \sqrt{{\log(kTd/\delta)} \big /{|E_{\tau-1}|}} + \sqrt{[X_{t,i}]_{S_{\tau-1}}^\top\big[\hat{\mat \Gamma}_{\tau-1,\lambda}^{t-1}\big]^{-1}[X_{t,i}]_{S_{\tau-1}}}\Big)\Big\},
$$
where
$$
\hat{\mat \Gamma}_{\tau-1,\lambda}^{t-1}=\lambda \textbf{I}_{|S_{\tau-1}|}+\sum_{t'\in E_{\tau}^{t-1}}[X_{t'}]_{S_{\tau-1}}[X_{t'}]^{\top}_{S_{\tau-1}}.
$$
After that, we pick the arm $i_t$ corresponding to the largest upper confidence band to play and observe the corresponding reward
$Y_t=\langle X_{t,i_t},\theta^*\rangle+ \varepsilon_t$. We repeat this process until the end of epoch $E_{\tau}$.

Then we run the BSS procedure using all data collected in $E_{\tau}$ to recover the support of $\theta^*$. We also enlarge the size of generalized support by $s$. To be specific, let $S_{\tau}$ be the generalized support recovered in this step. We require that $S_{\tau}$ satisfies constraints
$$
S_{\tau} \supseteq S_{\tau-1},\ |S_{\tau}|\le \tau s.
$$
and obtain the BSS estimator $\hat\theta_{\tau,\lambda}$ as
\begin{align}
\hat \theta_{\tau,\lambda}=\text{Proj}_r \Big\{ \operatorname*{argmin}_{S_{\tau-1}\subseteq \supp^+(\theta), |\supp^+(\theta)|\le \tau s} \Big \{\sum_{t' \in E_{\tau}}\big| Y_{t'}-\langle X_{t'},\theta \rangle  \big|^2 + \lambda \|\theta\|^2 \Big\} \Big\}, \label{eq:bss}
\end{align}
where $\text{Proj}_r\{\cdot\}$ denotes the projection on a centered $\ell_2$-ball with radius $r$, the $\ell_2$-norm of $\theta^*$. Note that in comparison with the standard BSS estimator, we project the minimizer for regularization. The boundedness also simplifies our later theoretical analysis. We also add the inclusion restriction $ S_{\tau} \supseteq S_{\tau-1}$ for some technical reasons, which does not lead to any fundamental difference. As a result, we need to consider  the sparsity $\tau s$ instead of  $s$. It boosts the probability of recovering the true support.

A pseudo-code description of the SLUCB algorithm is presented in Algorithm \ref{alg1}.
We remark that although computing the BSS estimator is computationally expensive,
many methods exist that can solve this problem very efficiently in practice.
See, for example, \citep{pilanci2015sparse,bertsimas2016best} for more details.
We also conduct extensive numerical experiments to demonstrate that our proposed algorithm is practical  in Section \ref{experiment}.

\begin{singlespace}
\begin{algorithm}[t]  \label{alg1}
    \footnotesize{\KwInput{sequentially arriving covariates $\{X_{t,i}\}_{t\in [T],i\in \cA_t}$, length of pure-exploration-stage $n_0$, confidence level $\alpha$, estimated upper bound of reward $\beta $, sparsity level $s$, ridge regression penalty $\lambda$.}
    \KwOutput{action sequence $\{i_t\}_{t\in [T]}$.}

    partition $[T]$ into consecutive epochs $E_1,E_2,\cdots, E_{\tilde \tau}$ such that $|E_\tau|=\max\{2^\tau,n_0\}$\;
    initialization: $\hat\theta_{0,\lambda} = 0$, $S=\emptyset$\;
    \For{$\tau=1,2,\cdots,\tilde{\tau}$}{
        \For{the first $n_0$ time periods $t \in E_\tau$}
        {
        select $i_t\in \cA_t$ uniformly at random\;
        set $\hat\theta^t_{\tau,\lambda}= \hat\theta_{\tau-1,\lambda}$\;
        }
        \For{the remaining time periods $t \in E_\tau$}
        {
            calculate matrix $$\hat{\mat\Gamma}_{\tau-1,\lambda}^{t-1}=\lambda \textbf{I}_{|S|} +\sum_{t'\in E_{\tau}^{t-1}}[X_{t'}]_{S}[X_{t'}]^{\top}_{S}\text{;}  $$ \\
            calculate the upper confidence band of reward for each arm $$\bar{r}(X_{t,i})=\min\Big \{\beta, \big\langle X_{t,i}, \hat\theta^{t-1}_{\tau,\lambda} \big\rangle + \alpha \cdot   \Big(\sigma \sqrt{{\log(kTd/\delta)} \big /{|E_{\tau-1}|}}  + \big \|[X_{t,i}]_{S} \big \|_{[\hat{\mat \Gamma}_{\tau-1,\lambda}^{t-1}]^{-1}}\Big) \Big\}\text{;} $$ \\
            select arm with the largest upper confidence band $$i_t=\operatorname*{argmin}_{i \in \cA_t} \big\{ \bar{r}(X_{t,i}) \big\} \text{;}$$\\
            observe reward $$Y_t = \langle X_{t,i_t}, \theta^*\rangle + \varepsilon_t\text{;}$$\\
            update the ridge  estimator: $$\hat\theta^t_{\tau,\lambda} =\operatorname*{argmin}_{\supp^+(\theta)=S}\Big \{ \sum_{t'\in E_\tau^t}\big| Y_{t'}-\langle X_{t'},\theta \rangle \big|^2 +\lambda \|\theta\|^2 \Big\}\text{;}$$\
        }
        update the best subset selection estimator
        $$\hat \theta_{\tau,\lambda}=\text{Proj}_r \Big\{ \operatorname*{argmin}_{S\subseteq \supp^+(\theta), |\supp^+(\theta)|\le \tau s} \Big \{\sum_{t' \in E_{\tau}}\big| Y_{t'}-\langle X_{t'},\theta \rangle  \big|^2 + \lambda \|\theta\|^2 \Big\} \Big\};$$

    update $S =\supp^+(\hat\theta_{\tau,\lambda}) \text{;}$
    }
    \caption{\textsc{Sparse-LinUCB} Algorithm}
    \label{alg:main}}

\end{algorithm}
\end{singlespace}

\subsection{\textsc{Sparse-SupLinUCB} Algorithm}  \label{algorithm2}
\begin{singlespace}
{   \begin{algorithm}[t]
        \footnotesize{\KwInput{epoch index $\tau$, sequential arriving covariates $\{X_{t,i}\}_{t\in E_{\tau}, i\in \cA_t}$,  length of pure-exploration-stage $n_0$, confidence level $\gamma$, estimated upper bound of reward $\beta $, support recovered in previous epoch $S_{\tau-1}$, sparsity level $s$, ridge regression penalty $\lambda$.}
            \KwOutput{action sequence $\{i_t\}_{t\in E_\tau}$.}
            \For{the first $n_0$ time periods $t \in E_\tau$}
                {
                select $i_t\in \cA_t$ uniformly at random\;
                }
        set $\tilde \zeta =\lceil \log(\beta T)\rceil$, $S=S_{\tau-1}$, and initialize sets $\{\Psi_{\tau}^{t,1}, \cdots,\Psi_{\tau}^{t,\tilde \zeta} \}$ by evenly partitioning the $n_0$ time periods of pure-exploration-stage\;
        \For{the remaining time periods $t$ in $E_\tau$}{
            initialize $\zeta=1, {\mathcal N}_{\tau}^{t-1,\zeta}= \cA_t$\;
            \Repeat{an arm $i_t\in \cA_t$ is selected}{
                compute restricted ridge  estimator
                $$
                \hat\theta_{\tau,\lambda}^{t-1,\zeta} = \operatorname*{argmin}_{\supp^+(\theta)=S}\Big\{ \sum_{t'\in \Psi_{\tau}^{t-1,\zeta}}\big|Y_{t'}-\langle X_{t'}, \theta\rangle\big|^2 +\lambda \|\theta\|^2 \Big\}\text{;}
                $$
                \

                compute matrix  $$ \hat{\mat{\Gamma}}^{t-1,\zeta}_{\tau-1,\lambda}=\lambda \textbf{I}_{|S|}+\sum_{t' \in \Psi^{t-1,\zeta}_{\tau}}[X_{t'}]_S[X_{t'}]_S^\top\text{;}$$\

                compute confidence band for each $i \in {\mathcal N}_{\tau}^{t-1,\zeta} $,   $$\omega_{\tau,\lambda}^{t-1,\zeta}(i) = \gamma \cdot \Big(  \sqrt{s/|E_{\tau-1}|}+ \big\|[X_{t,i}]_S\big\|_{[\hat{\mat\Gamma}^{t-1,\zeta}_{\tau-1,\lambda}]^{-1}} \Big)\text{;}$$ \

                \If{$\omega_{\tau,\lambda}^{t-1,\zeta}(i)\leq 1/\sqrt{T},\ \forall i\in {\mathcal N}_{\tau}^{t-1,\zeta}$}{
                     select $$i_t=\operatorname*{argmin}_{i \in {\mathcal N}_{\tau}^{t-1,\zeta} }\Big\{\beta, \big \langle X_{t,i},\hat\theta_\tau^{t-1,\zeta}\big\rangle +\omega_{\tau,\lambda}^{t-1,\zeta}(i)\Big\}\text{;}$$
                as the arm to play and update $\Psi_{\tau}^{t,\zeta}\gets\Psi_{\tau}^{t-1,\zeta}$ for all $\zeta \in [\tilde \zeta]$\;
                }
                \ElseIf{$\omega_{\tau,\lambda}^{t-1,\zeta}(i)\leq 2^{-\zeta}\beta, \forall i\in {\mathcal N}_{\tau}^{t-1,\zeta}$}{
                eliminate suboptimal arms as
                $${\mathcal N}_\tau^{t-1,\zeta+1} =\Big \{i\in {\mathcal N}_\tau^{t-1,\zeta}: \big \langle X_{t,i},\hat\theta_{\tau,\lambda}^{t-1,\zeta} \big\rangle
                     \geq \max_{j\in {\mathcal N}_{\tau}^{t-1,\zeta}} \big\langle X_{t,j}, \hat\theta_{\tau,\lambda}^{t-1,\zeta}\big \rangle - 2^{1-\zeta}\beta\Big \}\text{;}$$
                move to the next group and update $\zeta \gets \zeta + 1\text{;}$
                }
                \Else{
                    select $i_t\in {\mathcal N}_\tau^{t-1,\zeta}$ such that $\omega_{\tau,\lambda}^{t-1,\zeta}(i) > 2^{-\zeta}\beta$ as the arm to play\;
                    update $\Psi_\tau^{t,\zeta}\gets\Psi_{\tau}^{t-1,\zeta}\cup\{t\}$ and $\Psi_\tau^{t,\zeta'}\gets\Psi_\tau^{t-1,\zeta'}$ for all $\zeta'\neq\zeta$\;
                }
            }
        }}
        \caption{\textsc{Sparse-SupLinUCB} Subroutine}
        \label{alg:suplinucb}
        \end{algorithm} }
\end{singlespace}

Although the SLUCB algorithm is intuitive and easy to implement, we are unable to prove the optimal upper bound for its regret due to some technical reasons. Specifically, as discussed in the next section, we can only establish an $\tilde{\mathcal{O}}(s\sqrt{T})$ upper bound for the regret of the SLUCB algorithm, while the optimal regret should be $\tilde{\mathcal{O}}(\sqrt{sT})$. Here we omit all the constants and logarithmic factors and only consider the dependency on  horizon length and dimension parameters. The obstacle leading to sub-optimality is the dependency of covariates on random noises. Recall that at each period $t \in E_\tau$, the SLUCB algorithm constructs the ridge estimator $\hat{\theta}^{t-1}_{\tau,\lambda}$ using all historical data, where designs $ \{X_{t'}\}_{t'\in E^{t-1}_{\tau}}$ are correlated with noises $ \{\varepsilon_{t'}\}_{t'\in E^{t-1}_{\tau}}$ due to the UCB-type policy. Such a complicated correlation impedes us  from establishing tight confidence bands for  predicted rewards, which  results in suboptimal regret.

To close the aforementioned gap and achieve the optimality, we modify the seminal \textsc{SupLinUCB} algorithm \citep{auer2002using,chu2011contextual}, which is originally proposed to attain the optimal regret for classic stochastic linear bandit problem, as a subroutine in our framework. Then we propose the \textsc{Sparse-SupLinUCB} (SSUCB) algorithm. Specifically, we replace the ridge estimator and UCB-type policy with a modified \textsc{SupLinUCB} algorithm. The basic idea of the \textsc{SupLinUCB} algorithm is to separate the dependent designs into several groups such that within each group, the designs and noises are independent of each other. Then the ridge estimators of the true parameters are calculated based on group individually. Thanks to the desired independency, now we can derive  tighter confidence bands by applying sharper concentration inequality, which gives rise to the optimal regret in the final.

In the next, we present the details of \textsc{SupLinUCB} algorithm and show how to embed it in our framework. For each period $t\in E_\tau $ that belongs to the UCB-stage, the \textsc{SupLinUCB} algorithm partitions the historical periods $E^{t-1}_{\tau}$ into $\tilde \zeta $ disjoint groups
$$
E^{t-1}_{\tau}=\{ \Psi^{t-1,1}_\tau,\cdots, \Psi^{t-1,\tilde \zeta}_\tau \},
$$
where $\tilde \zeta= \lceil \log (\beta T)\rceil $. These groups are initialized by evenly partitioning periods from the pure-exploration-stage and are updated sequentially following the rule introduced as follows. For each period $t$, we screen the groups $ \{ \Psi^{t-1, \zeta}_\tau\} $ one by one (in an ascending order of index $\zeta $) to determine the action to take or eliminate some obvious sub-optimal actions.

Suppose that we are at the $\zeta$-th group now. Let ${\cN}^{t-1,\zeta}_{\tau}$ be the set of candidate actions that are still kept by the $\zeta$-th step, which is initialized as the whole action space $\cA_t$ when $\zeta=1$. We first calculate the ridge  estimator $\hat \theta^{t-1,\zeta}_{\tau,\lambda}$ restricted on the generalized support $S_{\tau-1}$, using data from group $ \Psi^{t-1,\zeta}_\tau$. Then for each action $ i \in  {\cN}^{t-1,\zeta}_{\tau} $,  we calculate $\omega_{\tau,\lambda}^{t-1,\zeta}(i) $, the width of confidence band of the potential reward. { Specifically, we have
\begin{align*}
    \hat \theta^{t-1,\zeta}_{\tau,\lambda}&=\operatorname*{argmin}_{\supp^+(\theta)=S_{\tau-1}}  \Big \{ \sum_{t' \in \Psi^{t-1,\zeta}_\tau}  \big| Y_{t'}-\langle X_{t',i_{t'}},\theta \rangle  \big|^2  + \lambda \| \theta\|^2 \Big\},\\
    \omega^{t-1,\zeta}_{\tau,\lambda}(i) & =  \gamma  \cdot  \Big(\sqrt{s/|E_{\tau-1}|}+\sqrt{[X_{t,i}]^\top_{S_{\tau-1}}[\hat{\mat \Gamma}_{\tau-1,\lambda}^{t-1,\zeta}]^{-1} [X_{t,i}]_{S_{\tau-1}}} \Big),
 \end{align*}
where the tuning parameter
\begin{align*}
     \gamma =\sigma(\lambda^{1/2}+\nu+\sigma)\log^{3/2}(kTd/\delta),
\end{align*}
corresponds to the confidence level. Our next step depends on the values of $\omega_{\tau,\lambda}^{t-1,\zeta}(i) $. If
$$\omega_{\tau,\lambda}^{t-1,\zeta}(i)\le 1/\sqrt{T},\ \forall i \in  {\cN}^{t-1,\zeta}_{\tau},$$
which means that the widths of confidence bands are uniformly small, we pick the action associated with the largest upper confidence band
$$\min\Big\{ \beta, \langle X_{t,i},\hat{\theta}_{\tau,\lambda}^{t-1,\zeta}\rangle + \omega_{\tau,\lambda}^{t-1,\zeta}(i) \Big\}.$$
where
$
\beta= r\sigma\cdot \sqrt{s\log(kTd/\delta)}
$
is an upper estimate of potential rewards.
In this case, we discard the data point $(X_{t},Y_t) $ and do not update any group, i.e., setting $\Psi^{t,\zeta}_{\tau}=\Psi^{t-1,\zeta}_{\tau}$, for all $ \zeta \in [\tilde \zeta ]. $

Otherwise, if there exists some $i \in {\cN}^{t-1,\zeta}_{\tau} $ such that $\omega_{\tau,\lambda}^{t-1,\zeta}(i)\ge 2^{-\zeta}\beta, $ which means that the width of confidence band is not sufficiently small, then we pick such an action $i$ to play for exploration. In this case, we add the period $t$ into the $\zeta$-th group while keeping all other groups unchanged, i.e.,
\begin{align*}
\Psi^{t,\zeta}_{\tau}=\Psi^{t-1,\zeta}_{\tau} \cup \{t\},\ \Psi^{t,\eta}_{\tau}=\Psi^{t-1,\eta}_{\tau},\ \text{if}\ \eta\neq \zeta.
\end{align*}

Finally, if neither one of the above scenarios happens, which implies that for all $i \in {\cN}^{t-1,\zeta}_{\tau} $ $\omega_{\tau,\lambda}^{t-1,\zeta}(i)\le 2^{-\zeta}\beta$, then we do not take any action for now. Instead, we eliminate some obvious sub-optimal actions and move to the next group $\Psi^{t-1,\zeta+1}_{\tau}$. Particularly, we update the set of candidate arms as
$${\mathcal N}_\tau^{t-1,\zeta+1} =\Big\{i\in{\mathcal N}_\tau^{t-1,\zeta}: \big \langle X_{t,i},\hat\theta_{\tau,\lambda}^{t-1,\zeta}\big \rangle
             \geq \max_{j\in {\mathcal N}_{\tau}^{t-1,\zeta}}\big \langle X_{t,j}, \hat\theta_{\tau,\lambda}^{t-1,\zeta}\big \rangle - 2^{1-\zeta}\beta\Big\}.
$$
We repeat the above procedure until an arm is selected. Since the number of groups is $\tilde \zeta=\lceil \log (\beta T) \rceil$ and $2^{-\tilde{\zeta}}\beta =1/T \le 1/\sqrt{T}$, the \textsc{SupLinUCB} algorithm stops eventually.  Finally, by replacing the direct ridge regression and UCB-type policy with the \textsc{SupLinUCB} algorithm above, we obtain the SSUCB algorithm.
The pseudo-code  is presented in Algorithm~\ref{alg:suplinucb}.


\section{Theoretical Results} \label{theory}         
In this section, we present the theoretical results of the SLUCB and SSUCB algorithms. We use the regret to evaluate the performance of our algorithms, which is a standard performance measure in literature. We denote by $\{i_t\}_{t\in [T]}$ the actions sequence generated by an algorithm. Then given the true parameter $ \theta^*$ and covariates $\{X_{t,i}\}_{t\in [T], i\in \cA_t } $, the regret of the sequence $\{i_t\}_{t\in [T]}$
 is defined as 
\begin{align} \label{reg_definition}
R_T\big(\{i_t\},\theta^*\big)=\sum_{t=1}^T \big \langle X_{t,i_t^*},\theta^* \big \rangle- \big\langle X_{i,i_t},\theta^* \big\rangle,
\end{align}
where $i_t^*=\argmax_{i \in \cA_t}\langle X_{t,i},\theta^* \rangle$ denotes the optimal action under the true parameter.  The regret measures the discrepancy in accumulated  reward between real actions and oracles where the true parameter is known to a decision-maker.  In what follows, in Section \ref{sec:assumptions}, we first introduce some technical   assumptions to facilitate our discussions. Then we study the regrets of the SLUCB  and SSUCB algorithms  in Section \ref{reg_bound1} and Section \ref{sec:reg2}, respectively.

\subsection{Assumptions}\label{sec:assumptions}

We present  the assumptions  in our theoretical analysis and discuss their relevance and implications. We consider  finite action spaces $\{\mathcal A_t\}_{t=1}^T$ and assume that there exists some constant $k$ such that $$|\mathcal A_t|=k,\ \forall t \in[T].$$ We also assume that for each period $t$, the covariates $\{ X_{t,i}\}_{i \in \cA_t}$ are sampled independently from an unknown  distribution $P_0$. We further impose the following assumptions on distribution $P_0$. 

\begin{assumption} \label{ass1}
Let random vector $X\in\RR^d$ follow the distribution $P_0$. Then $X$ satisfies:
\begin{enumerate}
\item[(A1)] \emph{(Sub-Gaussianity)}: Random vector $X \in \RR^d$ is centered and sub-Gaussian with parameter $\sigma^2\geq 1$, which means that $\mathbb E[X]=0$ and $$\mathbb E\big[\exp\{\sigma a^\top X\}\big]\leq \exp\big\{\sigma^2 \|a\|^2/2 \big\},\ \forall a \in \mathbb{R}^d;$$
\item[(A2)] \emph{(Non-degeneracy)}: There exists a constant $\rho\in(0,\sigma]$ such that $$\mathbb E\big[[X]_j^2\big]\geq \rho, \ \forall j\in[d];$$
\item[(A3)] \emph{(Independent coordinates)}: The $d$ coordinates of $X$ are independent of each other.
\end{enumerate}
\end{assumption}

We briefly discuss Assumption \ref{ass1}. First of all, (A1) is a standard assumption in literature, with sub-Gaussianity covering a broad family 
of distributions like Gaussian, Rademacher, and bounded distributions. Assumption (A2) is a non-degeneracy assumption which, together with (A3), implies that the smallest eigenvalue of the population covariance matrix  $\mathbb E[XX^\top]$ is lower bounded by some constant $\rho>0$.
Similar assumptions are also adopted in high-dimensional statistics literature in order to prove the ``restricted eigenvalue'' conditions of sample covariance matrices \citep{raskutti2010restricted},
which are essential in the analysis of penalized least square methods \citep{wainwright2009sharp,bickel2009simultaneous}.
However, we emphasize that in this paper the covariates indexed by the {selected} actions $\{i_t\}$ do {not} satisfy the 
restricted eigenvalue condition in general, and therefore, we need novel and non-standard analyses of high-dimensional M-estimators. For Assumption (A3), at a higher level, independence among coordinates enables relatively independent explorations in different dimensions, which is similar to the key idea of the SETC method \citep{lattimore2015linear}. Technically, (A3) is used to establish the key independence of sample covariance matrices restricted within and outside the recovered support. Due to such an independence, the rewards in the unexplored directions at each period are independent as well, which can be  estimated~efficiently. 

We also impose the following assumptions on the unknown $d$-dimensional true  parameter $\theta^*$.
\begin{assumption} \label{ass2}
\begin{enumerate}
\item[(B1)] \emph{(Sparsity)}: The true parameter $\theta^*$ is sparse. In other words, there exists an $s\ll d$ such that $|\supp(\theta^*)|=s $.  
\item[(B2)] \emph{(Boundedness)}: There exists a constant $r\geq 1$ such that $\|\theta^*\|\leq r$.
\end{enumerate}
\end{assumption}
Note that in Assumption \ref{ass2}, (B1) is the key sparsity assumption, which assumes that only $s\ll d$ components of the  true parameter $\theta^*$ are non-zero. 
Assumption (B2) is a boundedness condition on the $\ell_2$-norm of $\theta^*$. This assumption is
 often imposed, either explicitly or implicitly, in contextual bandit problems for deriving an upper bound for rewards {\citep{dani2008stochastic,chu2011contextual}}.

Finally, we impose the sub-Gaussian assumption on noises sequence $\{\varepsilon_t\}_{t=1}^T$, which is a standard assumption adopted in most statistics and bandit literatures.
\begin{assumption} \label{ass3}
\begin{enumerate}
\item[(C1)] \emph{(Sub-gaussian noise)}: The random noises $\{\varepsilon_t\}_{t=1}^T$ are independent, centered, and sub-Gaussian with parameter $\nu^2\geq 1$.
\end{enumerate}
\end{assumption}

\subsection{Regret Analysis of \text{Sparse-LinUCB}}  \label{reg_bound1}

In this section, we analyze the performance of the SLUCB algorithm. As discussed earlier,  we measure the performance via the  regret defined in~\eqref{reg_definition}.
We show that with a tailored choice of pure-exploration-stage length $n_0 $, tuning parameters $\alpha$, and $\beta$, the accumulated regret of the SLUCB algorithm is upper bounded by $\tilde{\cO}(s\sqrt{T})$ (up to logarithmic factors) with high probability. Formally, we have the following theorem. 
\begin{theorem} \label{theorem1} 
 For any $\delta\in(0,1)$, let
\begin{align*} 
n_0 &\asymp \rho^{-1} \cdot (\nu \sigma \tilde{\tau})^2\cdot s^{3} \cdot\log^4\big(kTd/(\delta \lambda)\big),\\
\alpha &= \sigma\nu\sqrt{s\tilde \tau}\cdot\log(\sigma\rho^{-1}dkT/\delta), \
\beta  = \sigma r\sqrt{s\log(kTd/\delta)}.
\end{align*}
Under Assumptions \ref{ass1}-\ref{ass3}, the regret of the actions sequence $\{i_t\}^T_{t=1} $ generated by the \textsc{Sparse-LinUCB} algorithm is upper bounded by  
\begin{align*}
    R_T\big(\{i_t\},\theta^*\big) \lesssim \Big( n_0\beta+ \alpha  \sqrt{sT\log(T)\cdot\big(\lambda+ \log(\sigma dTk/\delta) \big)} + \alpha \sigma \sqrt{T\log(kTd/\delta)}   \Big)\cdot \log(T),
\end{align*} \label{thm:main-upper}
with probability at least $1-\delta$.
\end{theorem}
Note that in Theorem \ref{theorem1}, if we omit all the constants and logarithmic factors, the dominating part in the accumulated regret is $$\tilde{\cO}(\alpha\cdot \sqrt{sT})=\tilde{\cO}(s\sqrt{T}).$$  Theorem \ref{theorem1} builds on a nontrivial combination of the UCB-type algorithm and the best subset selection method. To save space, we put the complete proof of Theorem \ref{theorem1} in Appendix \ref{appendix1}.

\subsection{Regret Analysis of \text{Sparse-SupLinUCB}}  \label{sec:reg2}
In comparison with the SLUCB algorithm, the SSUCB algorithm splits the historical data into several groups dynamically. In each period, we sequentially update the ridge estimator and corresponding confidence bands using data from a single group instead of the whole data. The motivation of only using a single group of data is to achieve the independence between the design matrix and random noises within each group, which leads to tighter confidence bands by applying a sharper concentration inequality. The tighter upper bounds of the predicted rewards lead to an improved regret. In particular, we have the following theorem.

\begin{theorem} \label{theorem2}
For any $\delta\in(0,1)$, let 
\begin{align*} 
    n_0 &\asymp \rho^{-1} \cdot (\nu \sigma \tilde{\tau})^2\cdot s^{3} \cdot\log^4\big(kTd/(\delta \lambda)\big),\\
\beta & = \sigma r \sqrt{s\log(kTd/\delta)},\ \gamma =\sigma(\lambda^{1/2}+\nu+\sigma)\log^{3/2}(kTd/\delta).
\end{align*}
Then under Assumptions \ref{ass1}-\ref{ass3}, the regret of actions sequence $\{i_t\}^T_{t=1} $ generated by the \textsc{Sparse-SupLinUCB} algorithm \vspace{4pt}
 is upper bounded by 
 \begin{align*}
    R_T\big(\{ i_t\},\theta^*\big) &\lesssim  \sqrt{T}+ \Big(n_0\beta+\gamma\sqrt{sT\cdot \log(T)\big(\lambda+ \log(\sigma dTk/\delta) \big)} \Big)\cdot  \log(T)\cdot \log(\beta T),
    \end{align*} \label{thm:suplinucb}
with probability at least $1-\delta$.
\end{theorem}
Note that in Theorem \ref{theorem2}, if we omit all constants and logarithmic factors, the dominating part in the regret upper bound is of order $\tilde{\cO}(\sqrt{sT})$. This improves the rate in Theorem \ref{theorem1} by an order of $\tilde{\cO}(\sqrt{s})$ and achieves the optimal rate (up to logarithmic factors). Theorem \ref{theorem2} builds on a tailored  analysis of the \textsc{SupLinUCB} algorithm. To save space, we also put the complete proof of Theorem \ref{theorem2} in Appendix~\ref{appendix2}.


\section{Numerical Experiments} \label{experiment}
In this section, we use extensive  numerical experiments to demonstrate our proposed algorithm's efficiency  and validate our theoretical results. To simplify implementation, we only test the SLUCB algorithm in our experiments, which can be implemented efficiently. Theoretically, the SLUCB algorithm only achieves a sub-optimal regret guarantee due to technical reasons. However, extensive numerical experiments imply that it attains the near-optimality in practice. Specifically, we first empirically verify the regret guarantee established in Section \ref{theory} via simulation under various settings. Then we apply the proposed method to a sequential treatment assignment problem using a real medical dataset. In simulation studies, we sample the contextual covariates from $d$-dimensional standard Gaussian distribution. For the real dataset test, we sample the covariates from the empirical distribution, which may not be Gaussian. Some assumptions in Section \ref{sec:assumptions} are also violated. However, in all scenarios, our algorithm performs well and achieves $\tilde{\mathcal{O}}(\sqrt{T})$ regrets, which demonstrates the robustness  and practicability. 

 \subsection{Simulation Studies}
In this section, we present the details of our simulation studies. We first show the $\tilde{\mathcal{O}}(\sqrt{T})$ growth rate of regret empirically. Then we fix time horizon length $T$ and dimension $d$ and study the dependency of accumulated regret on sparsity $s$. To demonstrate the power of best subset selection, we also compare our algorithm's performance with the oracle, where the decision-maker knows the true support of underlying parameters. Since the bottleneck of computing time in our algorithm is the best subset selection, which requires solving a mixed-integer programming problem, it is appealing to replace this step with other  variable  selection methods, such as Lasso and iterative hard thresholding (IHT) \citep{blumensath2009iterative}. So we test the performance of those variants. We fix the number of actions $k=60$ in all settings. 

\subsubsection{Experiment 1: Growth of regret}
In this experiment, we study the growth rate of regret. We run two sets of experiments, where in the first case $d=100, T=1300, s=5,10,15, $ and in the second case, $d=300,T=1970, s=15,20,25$. For each setup, we replicate $20$ times and then calculate corresponding mean and $90\%$-confidence interval. We present the results in Figure~\ref{fig1}. As we see, for each fixed $d$ and $s$, the growth rate of regret is about $\tilde{\mathcal{O}}(\sqrt{T})$, which validates our theory.

\begin{figure}   
	\begin{tabular}{cc}
	\includegraphics[height=0.310\textwidth]{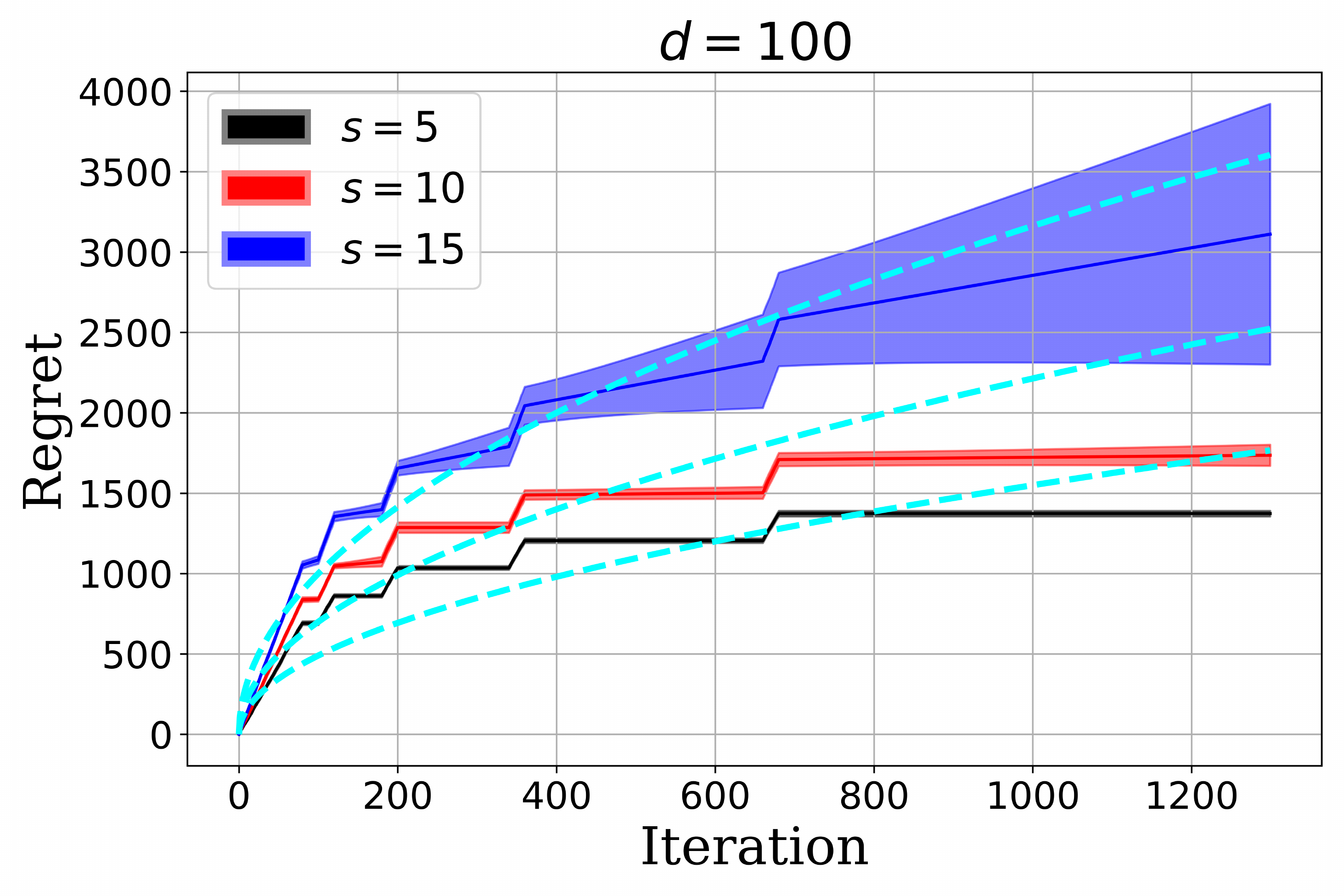}
	&
	\includegraphics[height=0.310\textwidth]{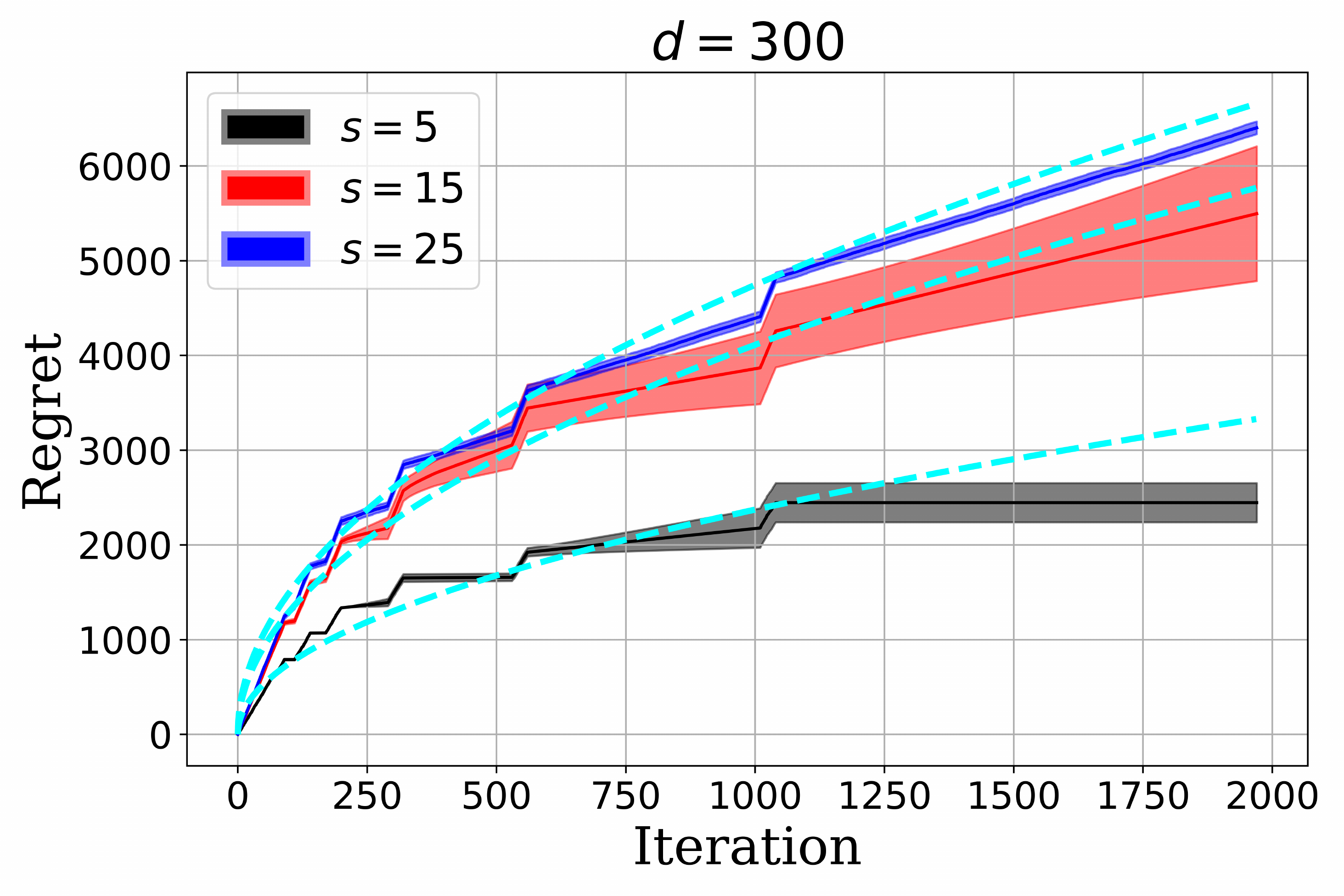}
	\\
	\end{tabular}
	\caption{\small Plot of regret  v.s. time periods. In (a), we set the dimension $d=100$, the horizon length $T=1300$, and the sparsity $s=5,10,15$. In (b), we set the dimension $d=300$, the horizon length $T=1970$, and the sparsity $s=5,15,25$. For each setting, we replicate $20$ times. Solid lines are the means of regret. Shadow areas denote corresponding empirical confidence intervals.}
	\label{fig1}	
	\vspace{-12pt}
\end{figure} 

\subsubsection{Experiment 2: Dependency on sparsity}
In this experiment, we fix the dimension $d$ and  horizon length $T$ and let sparsity $s$ varies. We calculate the accumulated regret at the end of horizon. We also run two sets of experiments, where in the first case $ d=100, T=1300, s=5,8,11,\cdots, 23$ and in the second case $ d=300, T=1970, s=5,8,11,\cdots, 23$. We present the results are presented in Figure \ref{fig2}. Although Theorem~\ref{thm:main-upper} only provides an  $\tilde{\mathcal{O}}(s\sqrt{T})$ regret guarantee for the SLUCB algorithm. The linear dependency of accumulated regret on $\sqrt{s}$ suggests that it actually attains the optimal $\tilde{\mathcal{O}}(\sqrt{sT})$ rate in practice. 

\begin{figure}   
	\centering
	\begin{tabular}{cc}
	\includegraphics[height=0.310\textwidth]{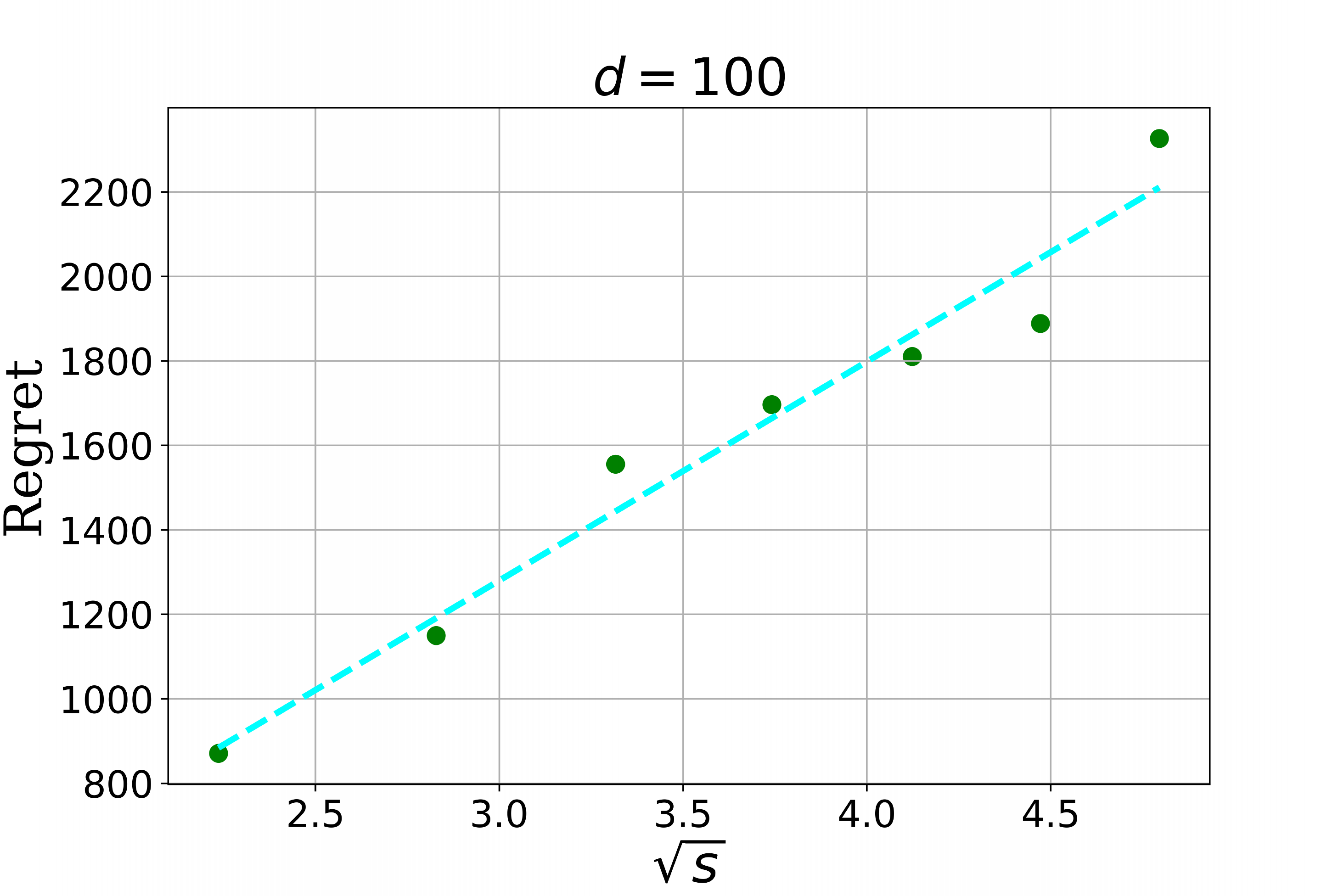}
	&
	\includegraphics[height=0.310\textwidth]{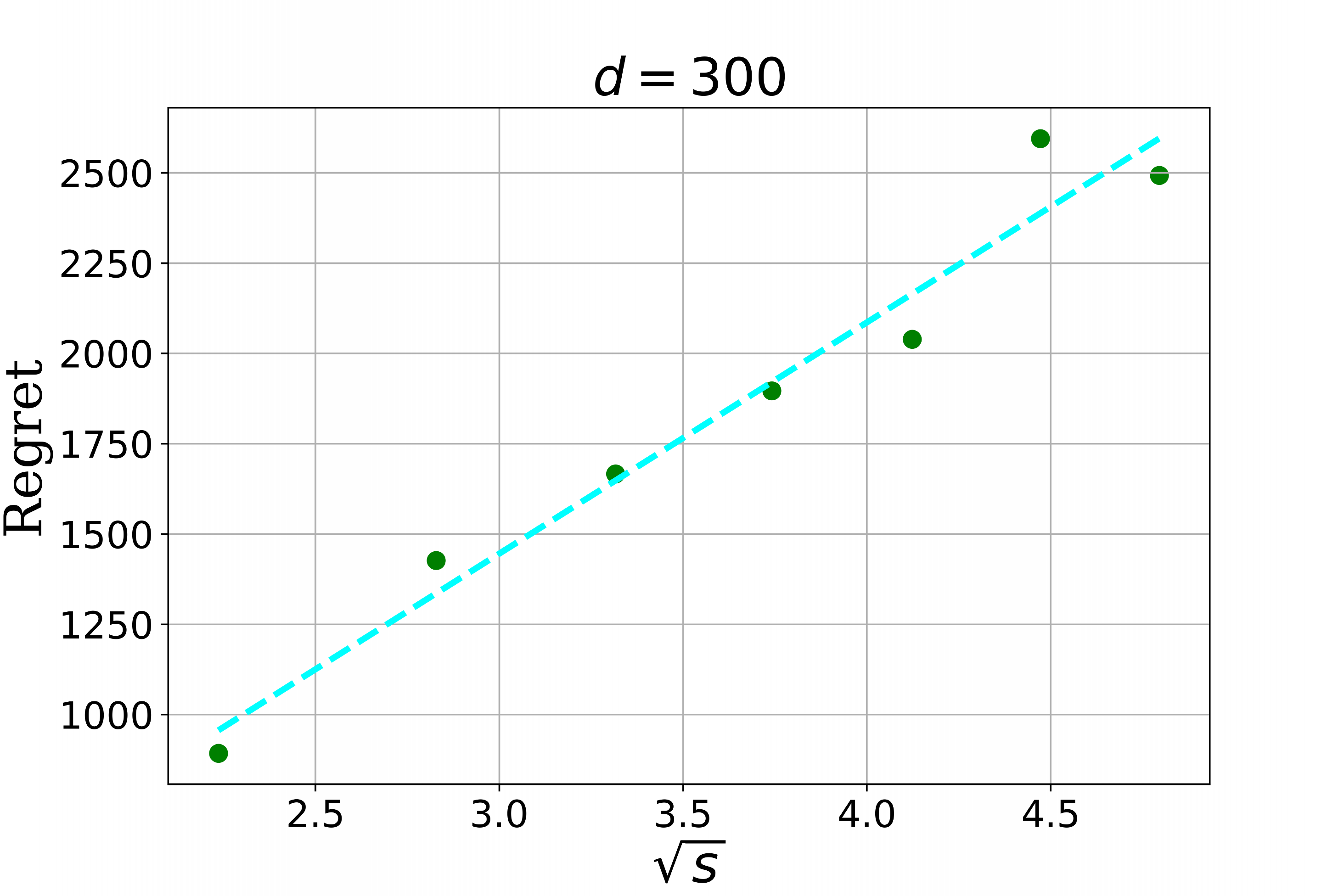}
	\\
	\end{tabular}
	\caption{\small Plot of accumulated regret v.s. $\sqrt{s}$ . In (a), we set the dimension $d=100$, the horizon length $T=1300$, and the sparsity $s=5,8,11,\cdots, 23$. In (b),  we set the dimension $d=300$, the horizon length $T=1970$, and the sparsity $s=5,8,11,\cdots, 23$.}
	\label{fig2}
	\vspace{-12pt}
\end{figure}

\subsubsection{Experiment 3: Comparison with variants of main algorithm and oracle}
{ In this experiment, we compare the performance and the computing time of our algorithm with several variants that substitute the best subset selection subroutine with Lasso and IHT.  We also compare with the oracle regret  where the decision-maker knows the true support of parameters. In more detail, for the first variant, we use Lasso to recover the support of the true parameter at the end of each epoch. We tune the $\ell_1$-penalty $\lambda$ such that the size of the support of the estimator is approximately equal to $s$, and then use it in the next epoch. For the second variant, we apply IHT to estimate the parameter and set the hard threshold to be $s$. 

We run two settings of experiments, corresponding to $d=100, s=15, T=1300 $, and $d=300, s=15, T=1300$. We also replicate $20$ times in each setting. For the first case, the averaged computing times are  $4.58$ seconds for Lasso,  $34.62$ seconds for IHT, and  $310.61$ seconds for best subset selection. For the second case, the average computing times are $9.10$ seconds for Lasso,  $39.47 $ seconds for IHT, and  $516.92$ seconds for best subset selection. We display the associated regret curves in Figure \ref{fig3}. As we see, although the computing time of Lasso is  shorter than the other two methods, the performance of Lasso is significantly weaker. The computing time of IHT is slightly longer than Lasso, but it achieves the similar performance as best subset selection, which suggests that IHT might be a good alternative in practice when the computing resource is limited. Finally, although the computing time of the best subset selection is the longest, its performance is the best. The regret almost achieves the same performance as the oracle. Such a result demonstrates the power of best subsection selection in high-dimensional stochastic bandit problems.}

\begin{figure}   
	\centering
	\begin{tabular}{cc}
	\includegraphics[height=0.310\textwidth]{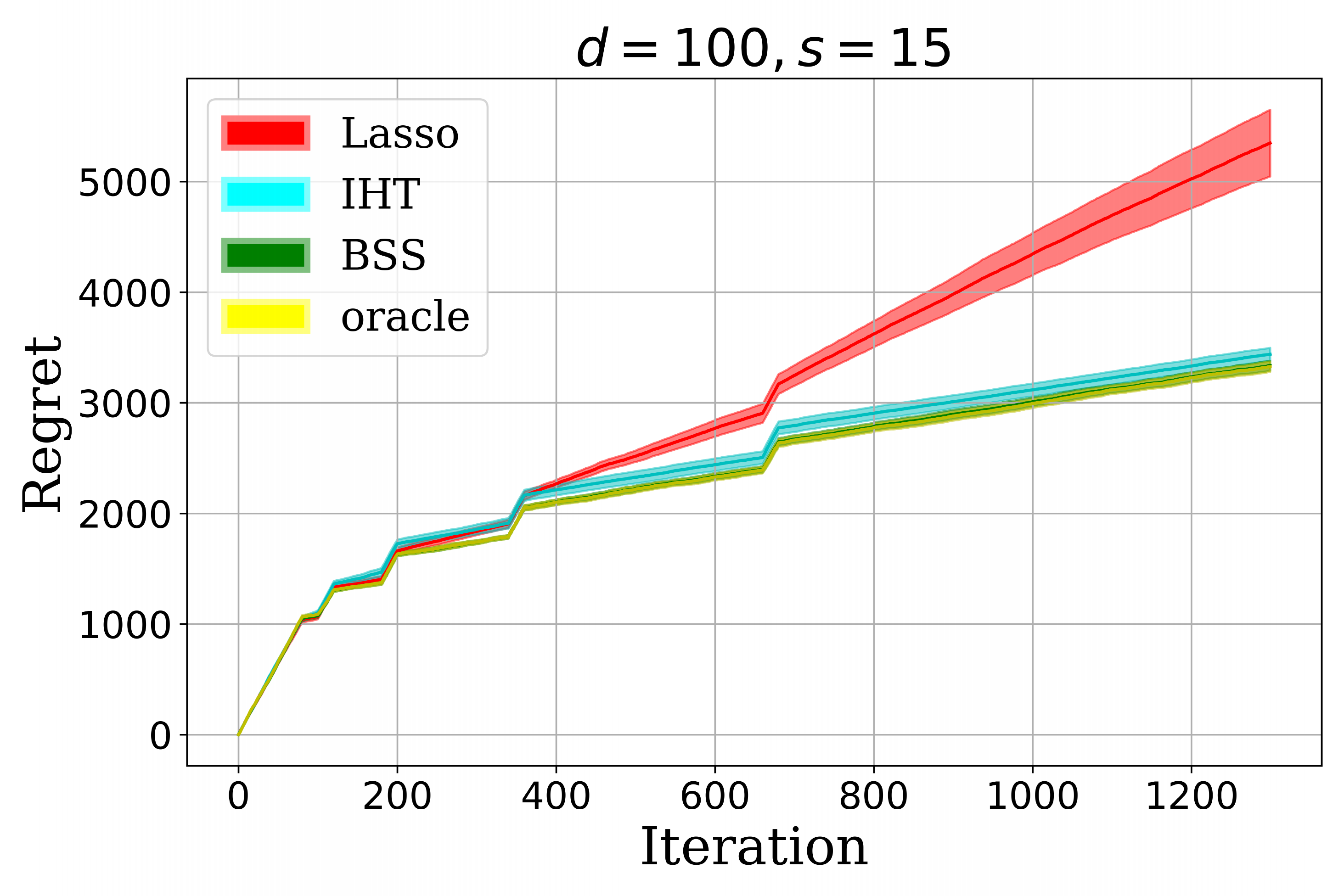}
	&
	\includegraphics[height=0.310\textwidth]{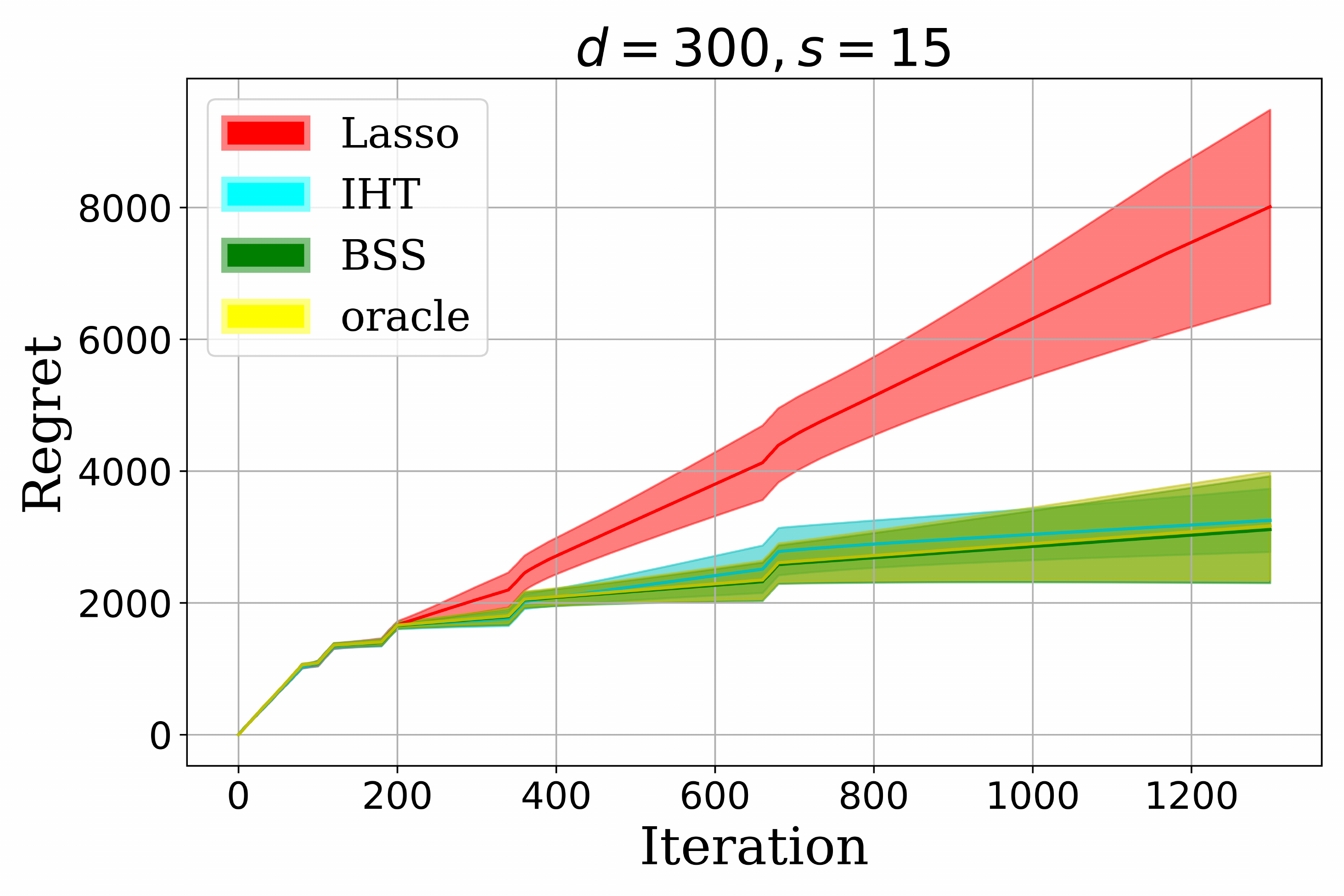}
	\\
	\end{tabular}
	\caption{\small Plot of regret curves of different algorithms. In (a), we set $d=100$, $s=15$, and $T=1300 $. In (b), we set $d=300$, $s=15$, and $T=1300 $. We test four algorithms: Lasso, IHT, BSS, and oracle.  We also replicate $20$ times in each setting. Solid lines are means of regret. Shadow areas denote corresponding confidence intervals. }
    \label{fig3}
	\vspace{-12pt}
\end{figure}


\subsection{Real Dataset Test}

In this section, we apply our methods in a sequential treatment assignment problem based on the ACTG175 dataset from the R package \textsl{``speff2trial"}. The scenario is that patients arrive at a clinic sequentially. Several candidate treatments are available, and different treatments may have different treatment effects. Hence, for each patient, the doctor observes his/her health information and then chooses a treatment for the patient, to maximize the total treatment effect.  

We first briefly introduce the ACTG175 dataset \citep{hammer1996trial}. This dataset records the results of a randomized experiment on 2139 HIV-infected patients, and the background information of those patients. The patients are randomly assigned one of the four treatments: zidovudine (AZT) monotherapy, AZT+didanosine (ddI), AZT+zalcitabine (ddC), and ddI monotherapy. The treatment effect is measured via the number of CD8 T-cell after 20 weeks of therapy. Besides, for each patient, an $18$-dimensional feature vector is provided as the background information, which includes age, weight, gender,  drug history, etc.

Similar to the simulation studies, we validate our algorithms by regret. However, in practice, the potential outcomes of unassigned treatments on each patient are unavailable, which prevents us from calculating the regret. To overcome this difficulty, we assume a linear model for the treatment effect of each therapy. Particularly, we assume that the effect of treatment $i$ (number of CD8 T cell) is $\langle X, \theta^*_i \rangle +\varepsilon$, where $X$ denotes the background information of each patient, $\varepsilon$ is a standard Gaussian noise, and $\theta_i^*$ is estimated via linear regression. { Similar assumptions are commonly adopted in statistics literatures to calculate the counterfactual, when real experiments are unavailable \citep{zhang2012estimating}}. Above sequential treatment assignment problem can be easily transformed to our formulation by considering the joint covariate and space. Given a patient with background $X$, let $X_1=(X,0,0,0),X_2=(0,X,0,0),X_3=(0,0,X,0),X_4=(0,0,0,X)$, and $\theta^*=(\theta^*_1,\theta^*_2,\theta^*_3,\theta^*_4)$. Since $ \langle X_i,\theta^* \rangle = \langle X, \theta^*_i\rangle $, assigning treatment to a patient is equivalent to pick an action from space $\mathcal{A}=\{X_1,X_2,X_3,X_4\}$. Note that with this formulation, Assumption \ref{ass1} is violated. However, the numerical results introduced below still demonstrate the good performance of our algorithm.

\begin{figure}   \label{fig4}   
	\centering
	\begin{tabular}{c}
	\includegraphics[height=0.320\textwidth]{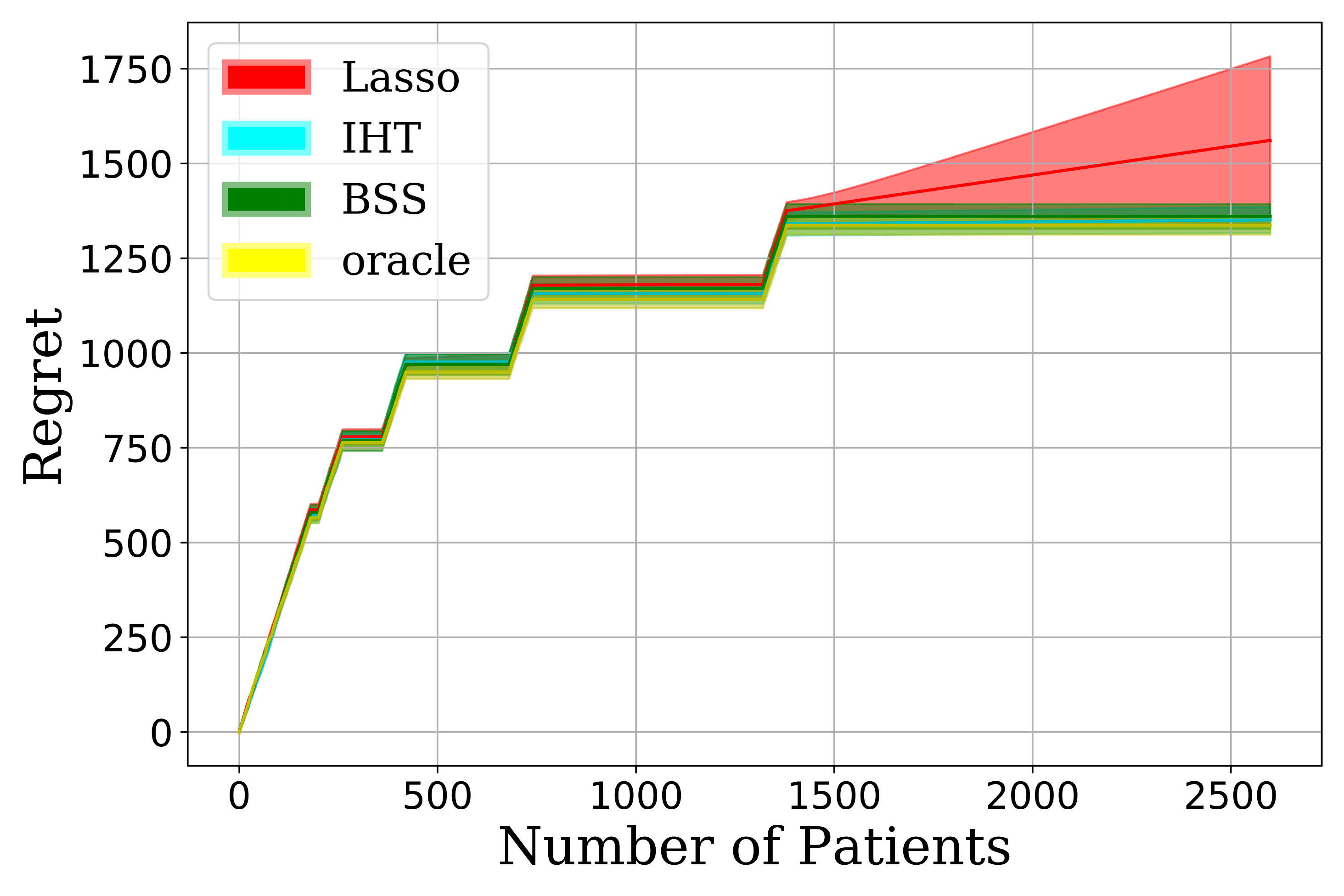}
	\\
	\end{tabular}
	\caption{\small Plot of regret in sequential treatment assignment problem based on ACTG175 dataset}
	\vspace{-12pt}
\end{figure}

In this experiment, for simplicity, we pick $9$ coordinates (age, weight, drugs in history, Karnofsky score, number of days of previously received antiretroviral therapy, antiretroviral history stratification, gender, CD4 T cell number at baseline, and CD8 T cell number at baseline), from the whole dataset as the covariate. We denote by it $X$. Then we use the real dataset to estimate the true parameters $\theta^*_i, i=1,2,3,4$, i.e., the treatment effect of each therapy. To test our algorithm in the high dimensional setting, we further concatenate $X$ with a $41$-dimensional random noise drawn from standard Gaussian distribution. The parameters corresponding to those injected noises are zeros. For each patient, we use the estimated parameters to calculate the potential treatment effect and regret. In addition, we choose the total number of patients $T=2600$. Like the simulation studies, we compare  our algorithm's performance with other variants (Lasso,IHT, and oracle). For each algorithm, we repeat $20$ times and calculate the mean regret curve. We present the result  in Figure~\ref{fig4}. Compared with the simulation studies, the performance of Lasso is still the worst. However, the discrepancy is much smaller. Although there is almost no difference in IHT and best subset selection, in real data tests, we may observe some instances where  the IHT algorithm fails to converge, which means a lack of robustness. Again, for our proposed algorithm, we observe an  $\tilde{\mathcal{O}}(\sqrt{T})$ growth rate of regret. It is pretty close to the oracle as well. So The BSS algorithm also achieves the optimality. Since in this experiment, technical assumptions in Section \ref{sec:assumptions} are  violated,  the numerical results demonstrate the robustness of our method.

{
\section{Conclusion} \label{sec:con}
In this paper, we first propose a method for the high-dimensional stochastic linear bandit problem by combining the best subset selection method and the \textsc{LinUCB} algorithm. It achieves the $\tilde{\mathcal{O}}(s\sqrt{T})$ regret upper bound and is nearly independent with the ambient dimension $d$ (up to logarithmic factors). In order to attain the optimal regret $\tilde{\mathcal{O}}(\sqrt{sT})$, we further improve our method by modifying the \textsc{SupLinUCB} algorithm. Extensive numerical experiments validate the performance and robustness of our algorithms. Moreover, although we cannot prove the $\tilde{\mathcal{O}}(s\sqrt{T})$ regret upper bound for the SLUCB algorithm due to some technical reasons, simulation studies show that the regret of SLUCB algorithm is actually $\tilde{\mathcal{O}}(\sqrt{sT})$ rather than our provable upper bound $\tilde{\mathcal{O}}(s\sqrt{T})$. A similar phenomenon is also observed in the seminal works \citep{auer2002using,chu2011contextual}, where low-dimensional stochastic linear bandit problems are investigated. It remains an open problem to prove the $\tilde{\mathcal{O}}(\sqrt{sT})$ upper bound for the \textsc{LinUCB}-type algorithms. In addition, it is also interesting to study the high-dimensional stochastic linear bandit problem under weaker assumptions, especially when the independent coordinates assumption does not hold. We view all of those problems as potential future research directions.}

\begin{singlespace}
\bibliographystyle{abbrvnat}
\bibliography{refs,bib}

\end{singlespace}

\

\spacingset{1.25} 

\begin{appendices}

\noindent{\LARGE{\bf Appendix}}
\section{Proof of Theorem \ref{theorem1}} \label{appendix1}

Before going to the proof, we present a proposition,  which upper bounds the scale of covariate $\{X_{t,i}\}$ uniformly with high probability. 

\begin{proposition} \label{X_uniform_bound}
    Consider all the sub-Gaussian covariates in the problem $\{X_{t,i}\}_{t\in [T],i \in \mathcal{A}_t}$. Then with probability at least $1-\delta$,
            \begin{align*}
                \max_{t\in [T],i \in \mathcal{A}_t, j \in [d]} \big | [X_{t,i}]_j \big | \lesssim \sigma \log(kTd/\delta).
            \end{align*}
\end{proposition}
\begin{proof} By Assumption \ref{ass1}, for all $t \in [T] $, $ i \in \cA_t$, and $ j\in[d] $, $\{[X_{t,i}]_j \} $ are independent and identically distributed centered sub-Gaussian random variables with parameter $ \sigma^2$. Then the result follows from the standard sub-Gaussian concentration inequality and union bound argument.
\end{proof}
Next, we present the proof of Theorem \ref{theorem1}. For ease of presentations, we defer the detailed proofs of the lemmas to the Supplementary Materials. Our proof  takes the benefits of results of the celebrated work on linear contextual bandit \citep{abbasi2011improved}, 
in which the authors establish a self-normalized martingale concentration inequality and the confidence region for the ridge  estimator under dependent designs. For the self completeness, we present these results in the Supplementary Materials Section~\ref{cite_lemma}.

\begin{proof} [Proof of Theorem \ref{theorem1}]
 The proof  includes four major steps. First, we quantify the performance of the BSS estimator obtained at the end of each epoch. In particular, we estimate the $\ell_2$-norm of the true parameter $\theta^*$ outside $S_{\tau}$, the generalized support recovered by epoch $E_\tau$. It measures the discrepancy between $S_{\tau}$ and the true support of $\theta^*$ and hence, the quality of BSS estimator. Then, within epoch $E_\tau$, for each period $t$, we analyze  the error of the restricted ridge estimator $\hat{\theta}^{t-1}_{\tau,\lambda}$. Based on that, we  upper bound the predictive errors of potential rewards, and  build up corresponding upper confidence bounds.  Next, similar to the classic analysis of UCB-type algorithm of linear bandit problem, we establish an elliptical potential lemma to upper bound the sum of confidence bands. Finally, we combine the results  to show that the regret of the SLUCB algorithm is  of order $\tilde{\cO}(s\sqrt{T})$. In what follows, we present the details of each step. 

\vspace{16pt}

\textbf{Step 1. Upper bound the error of the BSS estimator:}
In this step, we analyze the error of the BSS estimator obtained at the end of each epoch. Recall that $S_\tau=\supp^+(\hat \theta_{\tau,\lambda}) $, which is the generalized support of the BSS estimator $\hat \theta_{\tau,\lambda}$. We have the following lemma. 
\begin{lemma} \label{bss_error}
    If $$n_0 \gtrsim  \rho^{-1} \cdot (\nu \sigma \tilde{\tau})^2\cdot s^{3} \cdot\log^4\big(kTd/(\delta \lambda)\big), $$  we have that, with probability at least $1-\delta$,  for each epoch $ \tau \in [\tilde{\tau}]$, 
    \begin{equation}
    \big\|[\theta^*]_{S_{\tau}^c}\big\| \lesssim \sqrt{\frac{ \nu^2s\log^2\big({kTd}/{(\delta\lambda)}\big)+\lambda r^2 }{|E_{\tau}|}}.
    \end{equation}
    \label{lem:bss}
\end{lemma} 
Note that when $\tau$ is large, $ |E_\tau| \asymp T $. Hence, Lemma \ref{bss_error} states that the $\ell_2$-norm of the unselected part of  true parameter $\theta^*$ decays at a rate of $\tilde{\cO}(\sqrt{s/T})$. Recall that in the next epoch $\tau+1$, we restrict the estimation only on $S_\tau$. Failing to identify the exact support of $\theta^*$ incurs a regret that is roughly 
$$
\big\|[\theta^*]_{S_{\tau}^c}\big\|_2\cdot T=\tilde{\cO}(\sqrt{s/T}) \cdot T = \tilde{\cO}(\sqrt{sT}).
$$
We interpret this part as ``bias'' in the total regret. Formally, by Lemma \ref{lem:bss} and Proposition \ref{X_uniform_bound}, for each covariate $X_{t,i}$, we upper bound the inner product of $[\theta^*]_{S_\tau^c}$ and $[X_{t,i}]_{S_\tau^c}$ via the following corollary.  
\begin{corollary} \label{cor1}
    Under the condition of Lemma \ref{lem:bss}, with probability at least $1-\delta$, for all $ \tau \in [\tilde{\tau}]$, $t\in E_{\tau}$, and $i\in \cA_t$, we have
\begin{equation}
\Big| \big \langle [X_{t,i}]_{S_\tau^c},[\theta^*]_{S_\tau^c}\big \rangle\Big| \lesssim \sigma\sqrt{\log(kTd/\delta)} \cdot \sqrt{\frac{ \nu^2s\log^2\big({kTd}/{(\delta\lambda)}\big)+\lambda r^2 }{|E_{\tau}|}}.
\label{eq:bss-final}
\end{equation}
\label{cor:bss}
\end{corollary}
Essentially, Corollary \ref{cor1} implies that the regret incurred by inexact recovery of the support of true parameter at a single period is of order $\tilde{\cO}(\sqrt{s/T}) $. This result is critical for our regret analysis later.

\vspace{16pt}

\textbf{Step 2. Upper bound the prediction error:} In this step, for each time period $t\in E_{\tau}$ and arm $i\in \cA_t$, we establish an upper bound for the prediction error $\langle X_{t,i},\hat\theta_{\tau,\lambda}^{t-1}-\theta^*\rangle$. In particular, we have the following lemma.
\begin{lemma} \label{lem:ols}
    If $$ n_0 \gtrsim  \rho^{-1} \cdot (\nu \sigma \tilde{\tau})^2\cdot s^{3} \cdot\log^4\big(kTd/(\delta \lambda)\big), $$  with probability at least $1-\delta$, for all $\tau \in [\tilde{\tau}]$, $t\in E_{\tau}$, and $i\in \cA_t$, we have 
    \begin{equation*}
        \Big| \big \langle X_{t,i},\hat\theta_{\tau,\lambda}^{t-1}-\theta^* \big \rangle \Big| \lesssim  \nu  \Big( \sqrt{s}  \log\big({kTd}/{(\delta\lambda)}\big) + \lambda^{1/2}r \Big)   \Big(\sigma \sqrt{\frac{\log(kTd/\delta)}{|E_{\tau-1}|}}  + \big \|[X_{t,i}]_{S_{\tau-1}} \big \|_{[\hat{\mat \Gamma}_{\tau-1,\lambda}^{t-1}]^{-1}}\Big),
    \end{equation*}
    where  $$\hat{\mat \Gamma}_{\tau-1,\lambda}^{t-1}=\lambda\textbf{I}_{|S_{\tau-1}|} + \sum_{t'\in E_\tau^{t-1}}[X_{t'}]_{S_{\tau-1}}[X_{t'}]_{S_{\tau-1}}^\top,$$
    and $E_\tau^{t-1}=\{t': t'\le t-1,t'\in E_\tau \} $.
\end{lemma} 
The upper bound in Lemma \ref{lem:ols} includes two parts. The first term, which is of order $ \tilde{\mathcal{O}}(\sqrt{s/T})$, corresponds to the bias of the estimator $\hat{\theta}^{t-1}_{\tau,\lambda}$. Particularly, it is caused by the inexact recovery of $ \supp(\theta^*)$. In contrast, the second term, which is determined by the variability of $[X_{t,i}]_{S_{\tau-1}}$ and sample covariance matrix $ \hat{\mat \Gamma}_{\tau-1,\lambda}^{t-1}$, corresponds to the variance. 

\vspace{16pt}

\textbf{Step 3. Elliptical potential lemma:} In this this step, we establish the elliptical potential lemma, which a powerful tool in bandit problems \citep{rusmevichientong2010linearly,auer2002using,filippi2010parametric,li2017provably}, that upper bounds the sum of prediction errors in each period. Specifically, we upper bound the sum the right-hand side of the inequality in Lemma~\ref{lem:ols} for all periods. Recall that in each epoch, the arms are picked uniformly at random in the first $ n_0$ periods. 
\begin{lemma}  \label{lem:epl}
    Let two constants $p,q$ satisfy $ 1\le p \le q$. Then with probability at least $1-\delta$, we have
    \begin{equation*}
        \sum_{t\in E_\tau}\min\Big\{p,q \cdot \big \|[X_{t}]_{S_{\tau-1}} \big \|_{[\hat{\mat \Gamma}_{\tau-1,\lambda}^{t-1}]^{-1}} \Big\}  
    \lesssim n_0p + q\cdot \sqrt{|E_\tau| \cdot |S_{\tau-1}|\cdot \big(\lambda+ \log(\sigma dTk/\delta) \big) }.
    \end{equation*} 
\end{lemma}
The upper bound in Lemma \ref{lem:epl} consists of two terms. The first one is linear in $n_0$, the length of the pure-exploration-stage. The second part, which is of order $\tilde{\mathcal{O}}(\sqrt{sT})$, shows the trade-off of between exploration and exploitation. 
In later proof, we show that this part leads to the order $\tilde{\mathcal{O}}(s\sqrt{T})$ regret under a tailored choice of $p,q$.  

\vspace{16pt}

\textbf{Step 4. Putting all together:}
Based on Lemmas \ref{bss_error}, \ref{lem:ols}, and \ref{lem:epl}, we are ready to prove Theorem~\ref{theorem1}. For each $t \in [T]$ and $i \in \mathcal{A}_t $, let $$ r(X_{t,i})=\langle X_{t,i},\theta^*  \rangle$$ be the true expected reward of arm $i$ at period $t$. By Assumption \ref{ass3} and Proposition \ref{X_uniform_bound}, the inequality
 $$r(X_{t,i}) \le r \sigma\sqrt{\log(kTd/\delta)}$$ holds for all $t$ and $i$ with probability at least $1-\delta$. Then under the the following choice of $n_0$, 
\begin{align*} 
n_0 &\asymp  \rho^{-1} \cdot (\nu \sigma \tilde{\tau})^2\cdot s^{3} \cdot\log^4\big(kTd/(\delta \lambda)\big),
\end{align*} 
Lemma \ref{lem:ols} guarantees that 
\begin{align}  \label{ucb_r}
\bar r(X_{t,i})&:= \min\Big \{\beta, \big\langle X_{t,i}, \hat\theta_{\tau,\lambda}^{t-1} \big\rangle +    \alpha \cdot   \Big(\sigma \sqrt{{\log(kTd/\delta)} \big /{|E_{\tau-1}|}}  + \big \|[X_{t,i}]_{S_{\tau-1}} \big \|_{[\hat{\mat \Gamma}_{\tau-1,\lambda}^{t-1}]^{-1}}\Big) \Big\}
\end{align}
is an upper bound for $r(X_{t,i})$ for each $t\in[T]$ and $i\in \cA_t$ (up to an absolute constant),
where
\begin{align*}
\beta & = r \sigma \sqrt{s\log(kTd/\delta)}, \\
 \alpha&=\nu \cdot \big( \sqrt{s}  \log\big({kTd}/{(\delta\lambda)}\big) + \lambda^{1/2}r \big).
\end{align*}

Next, recall that at period $t$, our algorithm picks arm $i_t$. The optimal action is 
$$i_t^*= \arg\max_{i\in\cA_t}\langle X_{t,i},\theta^*\rangle. $$ 
Let $\mathcal{T}_e$ be the periods of pure-exploration-stage, in which the arms are picked uniformly at random. Correspondingly, we use $\mathcal{T}_g$ to denote the periods of UCB-stage, in which the UCB-type policy is employed. Then we have $$|\mathcal{T}_e| \lesssim n_0\log(T).$$ Since with probability at least $1-\delta$, $$ r(X_{t,i}) \le \max \big\{ \bar{r}(X_{t,i}), r\sigma\sqrt{\log(kTd/\delta)} \big\}$$ holds for all $t$ and $i$, by the definition of regret, we have
\begin{align} \label{reg1}
    R_T\big(\{i_t\},\theta^*\big) &= \sum_{t\in \mathcal{T}_e}  r(X_{t,i_t^*})-r(X_{t,i_t}) + \sum_{t\in \mathcal{T}_g} r(X_{t,i_t^*})-r(X_{t,i_t})  \notag \\ 
    & \lesssim  n_0 \log(T) \cdot 2r\sigma\sqrt{\log(kTd/\delta)}+ \sum_{t\in \mathcal{T}_g }\bar r(X_{t,i_t^*}) - r(X_{t,i_t}).
\end{align}
For the second part in the right-hand side of inequality \eqref{reg1}, we have 
\begin{align*}
    \sum_{t\in \mathcal{T}_g }\bar r(X_{t,i_t^*}) - r(X_{t,i_t}) &=  \sum_{t\in \mathcal{T}_g } \big( \bar r(X_{t,i_t^*}) - \bar r(X_{t,i_t})\big) + \big( \bar r(X_{t,i_t})-r(X_{t,i_t})\big) \notag \\
    & \le \sum_{t\in \mathcal{T}_g }  \bar r(X_{t,i_t})-r(X_{t,i_t})  \notag \\
    & \lesssim  \sum_{\tau=1}^{\tilde{\tau}}\sum_{t\in E_\tau} \min \Big\{ \beta,  \alpha \cdot   \Big(\sigma \sqrt{{\log(kTd/\delta)} \big /{|E_{\tau-1}|}}  + \big \|[X_{t}]_{S_{\tau-1}} \big \|_{[\hat{\mat \Gamma}_{\tau-1,\lambda}^{t-1}]^{-1}}\Big) \Big\}.
\end{align*}
Here the first inequality holds because in period  $ t \in \mathcal{T}_g$, the picked arm $i_t$ has the largest upper confidence band. The second inequality follows by the definition of $\bar{r}(X_{t,i})$ in \eqref{ucb_r}. By Lemma~\ref{lem:epl}, we have 
\begin{align} \label{reg2}
   & \sum_{\tau=1}^{\tilde{\tau}}\sum_{t\in E_\tau} \min \Big\{ \beta,  \alpha \cdot   \Big(\sigma \sqrt{{\log(kTd/\delta)} \big /{|E_{\tau-1}|}}  + \big \|[X_{t}]_{S_{\tau-1}} \big \|_{[\hat{\mat \Gamma}_{\tau-1,\lambda}^{t-1}]^{-1}}\Big) \Big\} \notag \\
     & \lesssim  \sum_{\tau=1}^{\tilde{\tau}}  \Big ( n_0\beta + \alpha\sqrt{|E_\tau| \cdot |S_{\tau-1}|\cdot \big(\lambda+ \log(\sigma dTk/\delta) \big) } \Big)  +  \sum_{\tau=1}^{\tilde{\tau}} \alpha \sigma \sqrt{{\log(kTd/\delta)} \cdot {|E_{\tau-1}|}}  \notag\\
    & \lesssim  \Big( n_0\beta+ \alpha  \sqrt{sT\log(T)\cdot\big(\lambda+ \log(\sigma dTk/\delta) \big)} + \alpha \sigma \sqrt{T\log(kTd/\delta)}   \Big)\cdot \log(T).
\end{align}
 Finally, by combining inequalities \eqref{reg1}-\eqref{reg2}, we obtain that with probability at least $1-\delta$, 
\begin{align*}
    R_T\big(\{i_t\},\theta^*\big) \lesssim \Big( n_0\beta+ \alpha  \sqrt{sT\log(T)\cdot\big(\lambda+ \log(\sigma dTk/\delta) \big)} + \alpha \sigma \sqrt{T\log(kTd/\delta)}   \Big)\cdot \log(T),
\end{align*}
which concludes the proof of Theorem \ref{theorem1}.    
\end{proof}

\section{Proof of Theorem \ref{theorem2}}  \label{appendix2}

\begin{proof}[Proof of Theorem \ref{theorem2}]
The proof  follows a similar routine as the proof of Theorem~\ref{thm:main-upper}.
The main difference lies in the upper bound of prediction error (Step 2 in the proof of Theorem~\ref{thm:main-upper}). Recall that in the SLUCB algorithm, we construct the ridge  estimator using all historical data. The complicated correlation between design matrices and random noises jeopardizes us applying the standard concentration inequality for independent random variables to obtain an order $\tilde{\cO}(\sqrt{s})$ upper bound for the prediction error. Instead, we use the self-normalized martingale concentration inequality, which only gives an order $\tilde{\cO}(s)$ upper bound and hence, a suboptimal regret. In comparison, in the SSUCB algorithm, the \textsc{Sparse-SupLinUCB} subroutine overcomes the aforementioned obstacle by splitting the historical data into disjoint groups such that within each group, the designs and random noises are independent. Due to such independences, we establish the desired order $\tilde{\cO}(\sqrt{s})$ upper bound for prediction errors. In particular, we have the following key lemma. 

\begin{lemma}  \label{improve_ols}
    When $$ n_0 \gtrsim  \rho^{-1} \cdot (\nu \sigma \tilde{\tau})^2\cdot s^{3} \cdot\log^4\big(kTd/(\delta \lambda)\big), $$ with probability at least $1-\delta$, we have that for every tuple $ (\tau,t,i,\zeta) $, it holds
    \begin{align*} 
        \Big | \big \langle X_{t,i}, \hat\theta_{\tau,\lambda}^{t-1,\zeta}-\theta^* \big\rangle \Big | & \lesssim    \sigma(\lambda^{1/2}+\nu+\sigma)\log^{3/2}(kTd/\delta)  \cdot  \Big(\sqrt{s/|E_{\tau-1}|}+\big \|[X_{t,i}]_{S_{\tau-1}} \big \|_{[\hat{\mat \Gamma}_{\tau-1,\lambda}^{t-1,\zeta}]^{-1}}\Big),
    \end{align*}
    where 
    $$\hat{\mat\Gamma}_{\tau-1,\lambda}^{t-1,\zeta} = \lambda \textbf{I}_{|S_{\tau-1}|}+ \sum_{t'\in\Psi_{\tau}^{t-1,\zeta}}[X_{t'}]_{S_{\tau-1}}[X_{t'}]_{S_{\tau-1}}^\top. $$     
\end{lemma}
The upper bound in Lemma \ref{improve_ols} includes two parts, corresponding to the bias and variance, respectively. The bias part is still of order $\tilde{\cO}(\sqrt{s})$. In comparison with Lemma \ref{lem:ols}, this lemma decreases an $\tilde{\mathcal{O}}(\sqrt{s})$ factor in the variance part. Note that by Lemma \ref{lem:epl}, we have 
$$
\big \|[X_{t,i}]_{S_{\tau-1}} \big \|_{[\hat{\mat \Gamma}_{\tau-1,\lambda}^{t-1,\zeta}]^{-1}}=\tilde{\cO}(\sqrt{s}).
$$
Such an improvement leads to an order $\tilde{\mathcal{O}}(\sqrt{s})$ improvement in the accumulated regret, which achieves the optimal rate. The detailed proof of Lemma \ref{improve_ols} is in the supplementary document. 

Same as the last step in the proof of Theorem \ref{theorem1}, by Lemma \ref{improve_ols}, we obtain that, with probability at least $1-\delta$, for all tuples $(\tau,t,i,\zeta)$, it holds 
\begin{align} \label{width_bound}
    {r}(X_{t,i}) &\lesssim 
\min \Big\{\beta, \big\langle X_{t,i},\hat\theta_{\tau,\lambda}^{t-1,\zeta} \big\rangle + \omega^{t-1,\zeta}_{\tau,\lambda}(i)  \Big\},
\end{align} 
 where  ${r}(X_{t,i})= \langle X_{t,i},\theta^*\rangle$  and 
\begin{align} \label{omega}
 \omega^{t-1,\zeta}_{\tau,\lambda}(i) & =  \gamma  \cdot  \Big(\sqrt{s/|E_{\tau-1}|}+\big \|[X_{t,i}]_{S_{\tau-1}} \big \|_{[\hat{\mat \Gamma}_{\tau-1,\lambda}^{t-1,\zeta}]^{-1}}\Big),\\
 \gamma & =\sigma(\lambda^{1/2}+\nu+\sigma)\log^{3/2}(kTd/\delta).
\end{align}

Without loss of generality, we assume that inequality \eqref{width_bound} holds for all tuples $(\tau,t,i,\zeta)$, which is guaranteed with probability at least $1-\delta$. We also assume that in period $t \in E_\tau$, the \textsc{SupLinUCB} algorithm terminates at the $\zeta_t$-th group with selected an arm $i_t$. Let $
i^*_t= \argmax_{i \in \mathcal{A}_t} \langle X_{t,i},\theta^* \rangle
$
be the optimal arm for period $t$. If  the first scenario in the \textsc{SupLinUCB} algorithm happens, i.e., 
$$\omega_{\tau,\lambda}^{t-1,\zeta_t}(i)\le 1/\sqrt{T},\ \forall i \in  {\cN}^{t-1,\zeta_t}_{\tau},$$
 we have 
\begin{equation} \label{inq001} 
\big \langle X_{t,i^*_t}-X_{t,i_t}, \theta^* \big \rangle \lesssim \big \langle X_{t,i^*_t}-X_{t,i_t}, \hat \theta_{\tau,\lambda}^{t-1,\zeta_t} \big \rangle + \omega^{t-1,\zeta_t}_{\tau,\lambda}(i^*_t)+\omega^{t-1,\zeta_t}_{\tau,\lambda}(i_t)  \lesssim 2/\sqrt{T}.
\end{equation}
Otherwise, we  have 
\begin{align} \label{inq003}
\omega^{t-1,\zeta_t}_{\tau,\lambda}(i_t) \ge  2^{-\zeta_t}\beta. 
\end{align} In this case, since the  \textsc{SupLinUCB} algorithm  does not stop before $\zeta_t$, then for all $\zeta<\zeta_t $, we have 
$$\omega_{\tau,\lambda}^{t-1,\zeta}(i) \le 2^{-\zeta}\beta, \ \forall i\in  \cN^{t-1,\zeta}_{\tau}.$$ 
Consequently, for all $ i\in  \cN^{t-1,\zeta}_{\tau} $, we have 
$$
	\big \langle X_{t,i_t^*},\hat \theta^{t-1,\zeta}_{\tau,\lambda} \big \rangle+\omega^{t-1,\zeta}_{\tau,\lambda}(i_t^*)  \gtrsim 
     \langle X_{t,i_t^*},\theta^*   \rangle  \ge  \langle X_{t,i},\theta^* \rangle \gtrsim  \big \langle X_{t,i},\hat \theta^{t-1,\zeta}_{\tau,\lambda}\big \rangle-\omega^{t-1,\zeta}_{\tau,\lambda}(i).
$$
Moreover, by setting $\zeta=\zeta_t-1 $, we have 
\begin{align} \label{inq002} 
	\big \langle X_{t,i^*_t}-X_{t,i_t}, \theta^* \big \rangle \lesssim	\big \langle X_{t,i^*_t}-X_{t,i_t}, \hat \theta^{t-1,\zeta_t-1}_{\tau,\lambda} \big \rangle+\omega^{t-1,\zeta_t-1}_{\tau,\lambda}(i^*_t)+\omega^{t-1,\zeta_t-1}_{\tau,\lambda}(i_t) \lesssim  2^{-\zeta_t}\beta.
\end{align}
By combining inequalities \eqref{inq001}-\eqref{inq002}, we obtain 
\begin{align} \label{int_reg}
    r(X_{t,i_t^*})-r(X_{t,i_t})& =\big \langle X_{t,i^*_t}-X_{t,i_t}, \theta^* \big \rangle \notag\\
    &\lesssim \max \big\{1/\sqrt{T},2^{-\zeta_t}\beta  \big\} \le \max \big\{1/\sqrt{T},\omega^{t-1,\zeta_t}_{\tau,\lambda}(i_t)  \big\} .
\end{align}

For the accumulated regret $ R_T( \{i_t\}, \theta^*)$, similar to inequality \eqref{reg1} in the proof of Theorem \ref{theorem1},  we have that, with probability at least $1-\delta$,
\begin{align*}
    R_T\big( \{i_t\}, \theta^*\big)  \lesssim n_0\log(T)\cdot 2r\sigma\sqrt{\log(kTd/\delta)} +\sum_{t\in \mathcal{T}_g} r(X_{t,i_t^*})-r(X_{t,i_t}),
\end{align*}
where $\mathcal{T}_g$ denotes the periods of UCB-stage. Then by inequality \eqref{int_reg} and Lemma \ref{improve_ols}, with probability at least $ 1-\delta$, we have 
\begin{align}
R_T\big( \{i_t\}, \theta^*\big) 
&\lesssim n_0\log(T)\cdot 2r\sigma\sqrt{\log(kTd/\delta)}+\sum_{t\in \mathcal{T}_g} \max\big \{1/\sqrt{T}, \omega^{t-1,\zeta_t}_{\tau,\lambda}(i_t) \big\} \notag\\
& \lesssim n_0\log(T)\cdot 2r\sigma\sqrt{\log(kTd/\delta)}+ \sqrt{T} + \sum_{t\in \mathcal{T}_g}\omega_{\tau,\lambda}^{t-1,\zeta_t}(i_t)\nonumber\\
&\lesssim \sqrt{T} + \sum_{\tau=1}^{\tilde \tau}\sum_{t\in E_\tau}\min \Big\{\beta,\gamma \cdot  \Big(\sqrt{s/|E_{\tau-1}|}+\big \|[X_{t,i}]_{S_{\tau-1}} \big \|_{[\hat{\mat \Gamma}_{\tau-1,\lambda}^{t-1,\zeta_t}]^{-1}}\Big)\Big\}  \notag \\
&= \sqrt{T} + \sum_{\zeta=1}^{\tilde \zeta}\sum_{\tau=1}^{\tilde \tau }\sum_{t\in \Psi^\zeta_\tau}\min \Big\{\beta,\gamma\cdot     \Big(\sqrt{s/|E_{\tau-1}|}+\big \|[X_{t,i}]_{S_{\tau-1}} \big \|_{[\hat{\mat \Gamma}_{\tau-1,\lambda}^{t-1,\zeta}]^{-1}}\Big)\Big\},
\label{eq:suplinucb-sum-final}
\end{align}
where $\omega_{\tau,\lambda}^{t,\zeta}(i) $ is defined in equation \eqref{omega}, and $\Psi^\zeta_\tau$ denotes the $\zeta$-th group of periods at the end of epoch $E_\tau$. 
For each fixed $\zeta$, we  have that, with probability at least $1-\delta$,
\begin{align} \label{reg_final}
    & \sum_{\tau=1}^{\tilde \tau }\sum_{t\in \Psi^\zeta_\tau}\min \Big\{\beta,\gamma\cdot     \Big(\sqrt{s/|E_{\tau-1}|}+\big \|[X_{t,i}]_{S_{\tau-1}} \big \|_{[\hat{\mat \Gamma}_{\tau-1,\lambda}^{t-1,\zeta}]^{-1}}\Big)\Big\}  \notag \\ & \lesssim  \sum_{\tau=1}^{\tilde{\tau}}  \Big ( n_0\beta + \gamma\sqrt{|E_\tau| \cdot |S_{\tau-1}|\cdot \big(\lambda+ \log(\sigma dTk/\delta) \big) } \Big)  +  \sum_{\tau=1}^{\tilde{\tau}} \gamma  \sqrt{s \cdot {|E_{\tau-1}|}} \notag \\
    & \lesssim   \Big(n_0\beta+\gamma\sqrt{sT\cdot \log(T)\big(\lambda+ \log(\sigma dTk/\delta) \big)} \Big)\cdot  \log(T).
\end{align}
Finally,  combining inequalities \eqref{eq:suplinucb-sum-final} and \eqref{reg_final}, we have that, with probability at least $1-\delta$, 
\begin{align*}
R_T\big(\{ i_t\},\theta^*\big) &\lesssim  \sqrt{T}+ \Big(n_0\beta+\gamma\sqrt{sT\cdot \log(T)\big(\lambda+ \log(\sigma dTk/\delta) \big)} \Big)\cdot  \log(T)\cdot \log(\beta T),
\end{align*}
which concludes the proof.
\end{proof}

\end{appendices}

\newpage

\spacingset{1.0} 

\begin{center}{\Large\bf Supplementary Materials to \\``Nearly Dimension-Independent Sparse Linear Bandit over Small Action Spaces via Best Subset Selection"}
\end{center}

\setcounter{figure}{0}
\setcounter{table}{0}
\setcounter{page}{1}


\section{Proof of Auxiliary Lemmas in Theorem \ref{theorem1}}  \label{app1}
\subsection{Proof of Lemma \ref{bss_error}}  \label{ple1}
\begin{lemma*}[Restatement of Lemma \ref{bss_error}]
    If $$n_0 \gtrsim  \rho^{-1} \cdot (\nu \sigma \tilde{\tau})^2\cdot s^{3} \cdot\log^4\big(kTd/(\delta \lambda)\big), $$  we have that, with probability at least $1-\delta$,  for each epoch $ \tau \in [\tilde{\tau}]$, 
    \begin{equation*}
    \big\|[\theta^*]_{S_{\tau}^c}\big\| \lesssim \sqrt{\frac{ \nu^2s\log^2\big({kTd}/{(\delta\lambda)}\big)+\lambda r^2 }{|E_{\tau}|}}.
    \end{equation*}
\end{lemma*}
\begin{proof} 

     The proof of Lemma \ref{bss_error} depends on the following lemma, which establishes an upper bound for the estimation error $\hat{\theta}_{\tau,\lambda}-\theta^* $ in an $\ell_2$-norm weighted by sample covariance matrix $$\hat{\mat \Gamma}_{\tau,\lambda}=\lambda \textbf{I}_d+\sum_{t\in E_\tau}X_tX_t^\top.$$ 
    \begin{lemma} \label{quadratic_upper_bound}
        With probability at least $1-\delta$, we have 
        $$
        \big\| \hat\theta_{\tau,\lambda}-\theta^* \big\|^2_{\hat{\mat \Gamma}_{\tau,\lambda}}\lesssim \nu^2s\log^2\big({kTd}/{(\delta\lambda)}\big)+\lambda r^2.
        $$
    \end{lemma}
    
    The proof of Lemma \ref{quadratic_upper_bound} is provided in Section \ref{qub_proof}. Now we continue the proof of Lemma \ref{bss_error}. To simplify notation, from now on, we fix the epoch $\tau$ and let $$ S=S_{\tau}=\supp^+(\hat{\theta}_{\tau,\lambda}).$$ Then by definition, $ [\hat\theta_{\tau,\lambda}-\theta^*]_{S^c}= -[\theta^*]_{S^c}$. We can decompose the error  $ \|\hat\theta_{\tau,\lambda}-\theta^* \|^2_{\hat{\mat \Gamma}_{\tau,\lambda}}$ as
\begin{align*}
    \big\|\hat\theta_{\tau,\lambda}-\theta^* \big \|^2_{\hat{\mat \Gamma}_{\tau,\lambda}}
   &= [\hat\theta_{\tau,\lambda}-\theta^*]_{S^c} ^\top\big[{\hat{\mat \Gamma}_{\tau,\lambda}}\big]_{S^c}[\hat\theta_{\tau,\lambda}-\theta^*]_{S^c}+ [\hat\theta_{\tau,\lambda}-\theta^*]_{S}^\top \big[{\hat{\mat \Gamma}_{\tau,\lambda}}\big]_{S}[\hat\theta_{\tau,\lambda}-\theta^*]_{S} \nonumber\\
    &\qquad   + [\hat\theta_{\tau,\lambda}-\theta^*]_{S}^\top\big[{\hat{\mat \Gamma}_{\tau,\lambda}}\big]_{SS^c}[\hat\theta_{\tau,\lambda}-\theta^*]_{S^c} 
    +[\hat\theta_{\tau,\lambda}-\theta^*]_{S^c}^\top\big[{\hat{\mat \Gamma}_{\tau,\lambda}}\big]_{S^cS}[\hat\theta_{\tau,\lambda}-\theta^*]_{S}\notag \\
    & \ge [\hat\theta_{\tau,\lambda}-\theta^*]_{S^c}^\top\big[{\hat{\mat \Gamma}_{\tau,\lambda}}\big]_{S^c}[\hat\theta_{\tau,\lambda}-\theta^*]_{S^c}+ [\hat\theta_{\tau,\lambda}-\theta^*]_{S}^\top\big[{\hat{\mat \Gamma}_{\tau,\lambda}}\big]_{SS^c}[\hat\theta_{\tau,\lambda}-\theta^*]_{S^c}\notag \\
    & \qquad  +[\hat\theta_{\tau,\lambda}-\theta^*]_{S^c}^\top\big[{\hat{\mat \Gamma}_{\tau,\lambda}}\big]_{S^cS}[\hat\theta_{\tau,\lambda}-\theta^*]_{S},
\end{align*}
where the inequality follows from the fact that the matrix $[{\hat{\mat \Gamma}_{\tau,\lambda}}]_{S} $ is positive definite. 
By rearranging above inequality, we obtain 
 \begin{align} \label{ineq0004}
    [\hat\theta_{\tau,\lambda}-\theta^*]_{S^c}^\top\big[\hat{\mat \Gamma}_{\tau,\lambda}\big]_{S^c}[\hat\theta_{\tau,\lambda}-\theta^*]_{S^c} &\le  -[\hat\theta_{\tau,\lambda}-\theta^*]_{S}^\top\big[\hat{\mat \Gamma}_{\tau,\lambda}\big]_{SS^c}[\hat\theta_{\tau,\lambda}-\theta^*]_{S^c} \notag \\
    & -[\hat\theta_{\tau,\lambda}-\theta^*]_{S^c}^\top\big[\hat{\mat \Gamma}_{\tau,\lambda}\big]_{S^cS}[\hat\theta_{\tau,\lambda}-\theta^*]_{S}  +  \big\|\hat\theta_{\tau,\lambda}-\theta^* \big\|^2_{\hat{\mat \Gamma}_{\tau,\lambda}}. 
 \end{align} 

    In the next, we upper bound the right-hand side of inequality \eqref{ineq0004} and then lower bound the  left hand side, which leads to a quadratic inequality with respect to the estimation error $\| [\hat \theta_{\tau,\lambda} -\theta^*]_{S_\tau^c}\| $. By solving this quadratic inequality, we finish the proof of Lemma \ref{bss_error}.
    
    We establish the upper bound first. According to H\"{o}lder's inequality, we have 
    \begin{align*}
    \Big|[\hat\theta_{\tau,\lambda}-\theta^*]_{S}^\top\big[\hat{\mat \Gamma}_{\tau,\lambda}\big]_{SS^c}[\hat\theta_{\tau,\lambda}-\theta^*]_{S^c} \Big|
    &\leq \big\|[\hat\theta_{\tau,\lambda}-\theta^*]_{S} \big\|_1\cdot \big\|\big[\hat{\mat \Gamma}_{\tau,\lambda}\big]_{SS^c}\big\|_{\infty}\cdot \big\|[\hat\theta_{\tau,\lambda}-\theta^*]_{S^c}\big \|_1\nonumber\\
    &\le \|\hat\theta_{\tau,\lambda}-\theta^* \|_1\cdot \big\|\big[\hat{\mat \Gamma}_{\tau,\lambda}\big]_{SS^c}\big\|_{\infty}\cdot \sqrt{s(\tau+1)}\cdot \big\|[\hat\theta_{\tau,\lambda}-\theta^*]_{S^c}\big \|,
    \end{align*}
    where the last inequality follows from the fact that $[\hat\theta_{\tau,\lambda}-\theta^*]_{S^c} $ has at most $(\tau+1)s$ non-zero coordinates. To continue the proof, we also need the following two lemmas, which upper bound $\|[\hat{\mat \Gamma}_{\tau,\lambda}]_{SS^c}\|_{\infty} $ and $\|\hat\theta_{\tau,\lambda}-\theta^* \|_1 $ individually. 
    \begin{lemma} \label{inf_norm_upper_bound}
        With probability at least $1-\delta$, we have 
        $$
        \big\|\big[\hat{\mat \Gamma}_{\tau,\lambda}\big]_{SS^c}\big\|_{\infty} \le  8\sigma{\log(d/\delta)}\sqrt{|E_{\tau}|}.
        $$
    \end{lemma}
   
    \begin{lemma} \label{l1_norm_upper_bound}
        For arbitrary constant $\varsigma>0$, when $$n_0\gtrsim \max\{\varsigma^2\rho^{-1}, \sigma^2\rho^{-2}{\log(d\tilde \tau/\delta)}\},$$ with probability at least $1-\delta$, we have  
        \begin{align*}
            \|\hat\theta_{\tau,\lambda}-\theta^*\|_1   \lesssim \nu s/\varsigma \cdot \sqrt{\tilde \tau}\cdot \log\big(kTd/(\delta \lambda)\big). 
        \end{align*}
    \end{lemma}
    The proofs of Lemma \ref{inf_norm_upper_bound}-\ref{l1_norm_upper_bound} are deferred to Section \ref{inf_proof}-\ref{l1_proof}. We continue the proof of Lemma \ref{bss_error} now. By applying Lemma \ref{inf_norm_upper_bound}-\ref{l1_norm_upper_bound} and setting
    \begin{align*} 
    \varsigma=\nu \sigma \tilde{\tau}\cdot s^{3/2} \cdot\log^2\big(kTd/(\delta \lambda)\big),
     \end{align*}
    we obtain the following inequality with probability at least $1-\delta$,
    \begin{align} \label{ineq0008}
    \Big|[\hat\theta_{\tau,\lambda}-\theta^*]_{S}^\top\big[\hat{\mat \Gamma}_{\tau,\lambda}\big]_{SS^c}[\hat\theta_{\tau,\lambda}-\theta^*]_{S^c} \Big| \lesssim  \sqrt{|E_{\tau}|}\cdot  \big\|[\hat\theta_{\tau,\lambda}-\theta^*]_{S^c}\big \|.
    \end{align}
    Note that with such a choice of $\varsigma$, we require
    $$
    n_0 \gtrsim \varsigma^2\rho^{-1}=\rho^{-1} \cdot (\nu \sigma \tilde{\tau})^2\cdot s^{3} \cdot\log^4\big(kTd/(\delta \lambda)\big).
    $$
    Similarly, we can establish the same upper bound for the second part in the right-hand side of inequality \eqref{ineq0004}. 
    Then by combing inequalities \eqref{ineq0004}  and \eqref{ineq0008}, we obtain 
    \begin{align} \label{ineq002}
        \Big|[\hat\theta_{\tau,\lambda}-\theta^*]_{S^c}^\top\big[\hat{\mat \Gamma}_{\tau,\lambda}\big]_{S^c}[\hat\theta_{\tau,\lambda}-\theta^*]_{S^c} \Big|\lesssim \big\|\hat\theta_{\tau,\lambda}-\theta^*\big\|_{\hat{\mat \Gamma}_{\tau,\lambda}}^2 + \sqrt{|E_{\tau}|}\cdot  \big\|[\hat\theta_{\tau,\lambda}-\theta^*]_{S^c}\big \|.
    \end{align}  
    
    Then we turn to the lower bound for the left-hand side of inequality \eqref{ineq002}. Recall that in the SLUCB algorithm, the selection of actions in epoch $E_\tau$ does not rely on $[X_t]_{S^c}$. Furthermore, by Assumption \ref{ass1} (A2, A3), for each $t$ and $j \in S^c$, $ [X_t]_{j} $ is sampled from distribution $P_0$ independently and satisfies
    $$ \mathbb{E}\big[ [X_t]_j^2 \big]\ge \rho.$$
    Subsequently, by matrix Bernstein inequality, when $$|E_\tau|\geq n_0 \gtrsim  \sigma^2\rho^{-2}\log (d\tilde \tau/\delta),$$ 
    with probability at least $1-\delta$, we have  
    \begin{equation}
   [\hat\theta_{\tau,\lambda}-\theta^*]_{S^c}^\top\big[\hat{\mat \Gamma}_{\tau,\lambda}\big]_{S^c}[\hat\theta_{\tau,\lambda}-\theta^*]_{S^c} \geq \rho/2 \cdot |E_\tau|\cdot \big\|[\hat\theta_{\tau,\lambda}-\theta^*]_{S^c}\big\|^2.
    \label{eq:bss-4}
    \end{equation}
    Then by combining inequalities \eqref{ineq0008}-\eqref{ineq002} and Lemma \ref{quadratic_upper_bound}, we obtain with probability at least $1-\delta$ that,
    \begin{equation*}
    \rho/2 \cdot |E_\tau|\cdot \big\|[\hat\theta_{\tau,\lambda}-\theta^*]_{S^c}\big\|^2 \lesssim  \sqrt{|E_\tau|}\cdot \big\|[\hat\theta_{\tau,\lambda}-\theta^*]_{S^c}\big\| +\nu^2s\log^2\big({kTd}/{(\delta\lambda)}\big)+\lambda r^2
    \end{equation*}
    Finally, we solve above inequality and obtain
    \begin{equation*}
    \big\|[\hat\theta_{\tau,\lambda}-\theta^*]_{S^c}\big\|_2 \lesssim\sqrt{\frac{\nu^2s\log^2\big({kTd}/{(\delta\lambda)}\big)+\lambda r^2}{|E_\tau|}}.
    \end{equation*}
    Then Lemma \ref{bss_error} follows from $[\hat\theta_{\tau,\lambda}]_{S^c}=0$ by the definition $S$. 
\end{proof}
    
\subsection{Proof of Lemma \ref{lem:ols}} \label{ple2}
 \begin{lemma*} [Resatement of Lemma \ref{lem:ols}]
    If $$ n_0 \gtrsim  \rho^{-1} \cdot (\nu \sigma \tilde{\tau})^2\cdot s^{3} \cdot\log^4\big(kTd/(\delta \lambda)\big), $$  with probability at least $1-\delta$, for all $\tau \in [\tilde{\tau}]$, $t\in E_{\tau}$, and $i\in \cA_t$, we have 
    \begin{equation*}
        \Big| \big \langle X_{t,i},\hat\theta_{\tau,\lambda}^{t-1}-\theta^* \big \rangle \Big| \lesssim  \nu  \Big( \sqrt{s}  \log\big({kTd}/{(\delta\lambda)}\big) + \lambda^{1/2}r \Big)   \Big(\sigma \sqrt{\frac{\log(kTd/\delta)}{|E_{\tau-1}|}}  + \big \|[X_{t,i}]_{S_{\tau-1}} \big \|_{[\hat{\mat \Gamma}_{\tau-1,\lambda}^{t-1}]^{-1}}\Big),
    \end{equation*}
    where  $$\hat{\mat \Gamma}_{\tau-1,\lambda}^{t-1}=\lambda\textbf{I}_{|S_{\tau-1}|} + \sum_{t'\in E_\tau^{t-1}}[X_{t'}]_{S_{\tau-1}}[X_{t'}]_{S_{\tau-1}}^\top,$$
    and $E_\tau^{t-1}=\{t': t'\le t-1,t'\in E_\tau \} $.
 \end{lemma*}

 \begin{proof}
    Recall that  $\hat\theta_{\tau,\lambda}^{t-1}$ is the ridge  estimator of $\theta^*$ using data $\{(X_{t'},Y_{t'})\}_{t'\in E_\tau^{t-1}}$ restricted on the generalized support  $S_{\tau-1}$. Then for each $i \in \mathcal{A}_t $, we have decomposition 
    \begin{align}  \label{ineq0009}
    \big \langle X_{t,i},\hat\theta_{\tau,\lambda}^{t-1}-\theta^* \big\rangle = -\big \langle [X_{t,i}]_{S_{\tau-1}^c}, [\theta^*]_{S_{\tau-1}^c}\big \rangle +\big \langle [X_{t,i}]_{S_{\tau-1}},  [\hat\theta_{\tau,\lambda}^{t-1}-\theta^*]_{S_{\tau-1}}\big \rangle.  
    \end{align} 
    We can upper bound the first term in the right-hand side of equation \eqref{ineq0009} by Corollary \ref{cor1}. Formally, with probability at least $1-\delta$, we have 
    \begin{align}   \label{ineq00011}
    \Big| \big \langle [X_{t,i}]_{S_{\tau-1}^c},[\theta^*]_{S_{\tau-1}^c}\big \rangle\Big| \lesssim \sigma\sqrt{\log(kTd/\delta)} \cdot \sqrt{\frac{ \nu^2s\log^2\big({kTd}/{(\delta\lambda)}\big)+\lambda r^2 }{|E_{\tau-1}|}}.
    \end{align}
    It remains to upper bound the second part. By Cauchy-Schwartz inequality, we have 
    \begin{align} \label{ineq00010}
      & \Big| \big\langle [X_{t,i}]_{S_{\tau-1}^c}, [\hat{\theta}_{\tau,\lambda}^{t-1} -\theta^*]_{S_{\tau-1}^c}\big\rangle\Big|  \le 
       \big\| [X_{t,i}]_{S_{\tau-1}} \big\|_{ [\hat{\mat \Gamma}_{\tau-1,\lambda}^{t-1}]^{-1}}  \cdot \big\| [\hat{\theta}_{\tau,\lambda}^{t-1} -\theta^*]_{S_{\tau-1}^c} \big\|_{ \hat{\mat \Gamma}_{\tau-1,\lambda}^{t-1}}.
    \end{align}

    The following lemma establishes an upper bound for the second term in the right-hand side of inequality \eqref{ineq00010}. The proof relies the self-normalized martingale concentration inequality (i.e., Lemma \ref{self_normalized_ineq} in Section \ref{cite_lemma}), which is similar to the proof of Lemma \ref{quadratic_upper_bound}. 

    \begin{lemma} \label{lem:ols-empirical-process}
        For each $\tau \in [\tilde\tau]$ and $t\in E_\tau$, we have 
        \begin{equation*}
            \big\| [\hat\theta_{\tau,\lambda}^{t-1}-\theta^*]_{S_{\tau-1}}  \big\|^2_{\hat{\mat \Gamma}^{t-1}_{\tau-1,\lambda}}   \lesssim \big(\nu^2+ \sigma^2\cdot \big\|  [\theta^*]_{S^c_{\tau-1}}\big\|^2 \big)\cdot \big(   s  \log^2\big({kTd}/{(\delta\lambda)}\big) + \lambda r^2  \big),
        \end{equation*}
        with probability at least $1-\delta$.
    \end{lemma}

    The proof of Lemma \ref{lem:ols-empirical-process} is deferred to Section \ref{pe}. We continue the proof of Lemma \ref{lem:ols} now. By combing inequalities \eqref{ineq0009}-\eqref{ineq00010} and Lemma \ref{lem:ols-empirical-process}, we obtain with probability at least $1-\delta$ that
    \begin{align*}
       \Big|  \big\langle X_{t,i},\hat\theta_{\tau,\lambda}^{t-1}-\theta^* \big\rangle \Big| & \lesssim  \sigma\sqrt{\log(kTd/\delta)} \cdot \sqrt{\frac{ \nu^2s\log^2\big({kTd}/{(\delta\lambda)}\big)+\lambda r^2 }{|E_{\tau-1}|}}  \\
       &  \ + \Big(\nu+ \sigma\cdot \big\|  [\theta^*]_{S^c_{\tau-1}}\big\| \Big) \Big( \sqrt{s}  \log\big({kTd}/{(\delta\lambda)}\big) + \lambda^{1/2}r \Big) \cdot \big \|[X_{t,i}]_{S_{\tau-1}}\big \|_{[\hat{\mat \Gamma}_{\tau-1,\lambda}^{t-1}]^{-1}}.    
    \end{align*}
    Recall that in Lemma \ref{bss_error}, we show that $$ \big\|  [\theta^*]_{S^c_{\tau-1}} \big\|=\tilde{\mathcal{O}}(\sqrt{s/T}).$$ Hence, when $ s \ll T$, we have  
    \begin{align*}
        \Big| \big \langle X_{t,i},\hat\theta_{\tau,\lambda}^{t-1}-\theta^* \big\rangle \Big| \lesssim  \nu  \Big( \sqrt{s}  \log\big({kTd}/{(\delta\lambda)}\big) + \lambda^{1/2}r \Big)   \Big(\sigma \sqrt{\frac{\log(kTd/\delta)}{|E_{\tau-1}|}}  + \big \|[X_{t,i}]_{S_{\tau-1}} \big \|_{[\hat{\mat \Gamma}_{\tau-1,\lambda}^{t-1}]^{-1}}\Big),   
    \end{align*}
    with probability at least $1-\delta$. Finally, by applying the union bound argument for all $t\in E_{\tau}$, we finish the proof of  Lemma \ref{lem:ols}.
    \end{proof}

\subsection{Proof of Lemma \ref{lem:epl}}  \label{ple3}

    \begin{lemma*}[Restatement of Lemma \ref{lem:epl}]
        Let two constants $p,q$ satisfy $ 1\le p \le q$. Then with probability at least $1-\delta$, we have
        \begin{equation*}
            \sum_{t\in E_\tau}\min\Big\{p,q \cdot \big \|[X_{t}]_{S_{\tau-1}} \big \|_{[\hat{\mat \Gamma}_{\tau-1,\lambda}^{t-1}]^{-1}} \Big\}  
        \lesssim n_0p + q\cdot \sqrt{|E_\tau| \cdot |S_{\tau-1}|\cdot \big(\lambda+ \log(\sigma dTk/\delta) \big) }.
        \end{equation*} 
    \end{lemma*}
    \begin{proof}
        For a fixed epoch $E_\tau$, let $t_1$ be the end of pure-exploration-stage and $t_2$ be the end of $E_\tau$. Then we have $t_1 \le n_0 \le t_2$. We also have 
        \begin{align}
        & \sum_{t\in E_\tau}\min\Big\{p,q \cdot \big \|[X_{t}]_{S_{\tau-1}} \big \|_{[\hat{\mat \Gamma}_{\tau-1,\lambda}^{t-1}]^{-1}} \Big\} \nonumber \\
        & \quad  \leq n_0p + q\cdot \sum_{t=t_1+1}^{t_2}\min\Big\{1,\big \|[X_{t}]_{S_{\tau-1}} \big \|_{[\hat{\mat \Gamma}_{\tau-1,\lambda}^{t-1}]^{-1}} \Big\}\nonumber\\
        & \quad \leq n_0p + q\sqrt{|E_\tau|}\cdot \sqrt{\sum_{t=t_1+1}^{t_2}\min\Big\{1, \big \|[X_{t}]_{S_{\tau-1}} \big \|_{[\hat{\mat \Gamma}_{\tau-1,\lambda}^{t-1}]^{-1}} \Big\}}\nonumber\\
        & \quad \leq  n_0p + q\sqrt{|E_\tau|}\cdot \sqrt{2\sum_{t=t_1+1}^{t_2}\log \Big(1+ \big \|[X_{t}]_{S_{\tau-1}} \big \|_{[\hat{\mat \Gamma}_{\tau-1,\lambda}^{t-1}]^{-1}} \Big )}.\label{eq:epl-1}
        \end{align}
        Here the first inequality holds because $q\geq p$ and $n_0\le t_1$. 
        The second inequality follows from Cauchy-Schwartz inequality.
        The last inequality is due to the basic inequality $$\min\{1,u\}\leq 2\log(1+u),\ \forall u>0.$$ On the other hand, we have the recursion 
       \begin{align*}
        \hat{\mat \Gamma}_{\tau-1,\lambda}^{t} & = \hat{\mat \Gamma}_{\tau-1,\lambda}^{t-1} + [X_t]_{S_{\tau-1}}[X_t]_{S_{\tau-1}}^\top  \\ 
        &=  \big[\hat{\mat \Gamma}_{\tau-1,\lambda}^{t-1}\big]^{1/2}\cdot \Big(I_{|S_{\tau-1}|} + \big[\hat{\mat \Gamma}_{\tau-1,\lambda}^{t-1}\big]^{-1/2}[X_t]_{S_{\tau-1}}[X_t]_{S_{\tau-1}}^\top \big[\hat{\mat \Gamma}_{\tau-1,\lambda}^{t-1}\big]^{-1/2}\Big)\cdot \big[\hat{\mat \Gamma}_{\tau-1,\lambda}^{t-1}\big]^{1/2}.
       \end{align*}
        Subsequently, we have 
        \begin{align}
         \log\big(\det(\hat{\mat \Gamma}_{\tau-1,\lambda}^{t})\big) =  \log\big(\det(\hat{\mat \Gamma}_{\tau-1,\lambda}^{t-1})\big)  + \log\big((1+ [X_t]_{S_{\tau-1}}[X_t]_{S_{\tau-1}}^\top )\big).\label{eq:epl-2}
        \end{align}
        By combining equations \eqref{eq:epl-1}-\eqref{eq:epl-2} and taking telescoping sum, we obtain
         \begin{equation} \label{ineq00012}
            \sum_{t\in E_\tau}\min\Big\{p,q \cdot \big \|[X_{t}]_{S_{\tau-1}} \big \|_{[\hat{\mat \Gamma}_{\tau-1,\lambda}^{t-1}]^{-1}} \Big\} 
         \leq n_0p + q\cdot \sqrt{2|E_\tau|\cdot \log\big(\det(\hat{\mat \Gamma}_{\tau-1,\lambda}^{t_2})\big /{\det(\hat{\mat \Gamma}_{\tau-1,\lambda}^{t_1})}\big)}.
        \end{equation}
    
        By Proposition \ref{X_uniform_bound}, we have with probability at least $1-\delta$ that, for all $t\in [T]$ and $ i \in  \mathcal{A}_t$,  
        \begin{equation*}
        \big\|X_{t,i}\big\|\lesssim \sigma\sqrt{d\log(kT/\delta)}.
        \end{equation*}
        Then by applying the determinant-trace inequality, we have 
        \begin{align*}
            \det(\hat{\mat \Gamma}_{\tau-1,\lambda}^{t_2})\le \Big(  \tr(\hat{\mat \Gamma}_{\tau-1,\lambda}^{t_2}) \big / |S_{\tau-1}| \Big)^{|S_{\tau-1}|} \le \Big( \lambda +\max_{t \in [T],i \in  \mathcal{A}_t}\|X_{t,i}\| \Big)^{|S_{\tau-1}|}.
        \end{align*}
        Hence, we have
        \begin{equation} \label{eq:logdet-ub}
        \log\big(\det(\hat{\mat \Gamma}_{\tau-1,\lambda}^{t_2})\big) \lesssim |S_{\tau-1}|\cdot \big(\lambda+ \log(\sigma dTk/\delta) \big)
        \end{equation}
        with probability at least $1-\delta$. On the other hand, by definition, we have $\hat{\mat \Gamma}_{\tau-1,\lambda}^{t_1}   \gtrsim \lambda \textbf{I}_{|S_{\tau-1}|}$, which implies 
        \begin{equation}
        \log\big(\det(\hat{\mat \Gamma}_{\tau-1,\lambda}^{t_1}  )\big) \gtrsim |S_{\tau-1}|\cdot \log(\lambda).
        \label{eq:logdet-lb}
        \end{equation}
        Finally, by combining inequalities \eqref{ineq00012}-\eqref{eq:logdet-lb}, we complete the proof of Lemma \ref{lem:epl}.
\end{proof}

\section{Proof of Auxiliary Lemmas in Theorem \ref{theorem2}}  \label{app2}
\subsection{Proof of Lemma \ref{improve_ols}}  \label{ple5}
\begin{lemma*}[Restatement of Lemma \ref{improve_ols}]
    When $$ n_0 \gtrsim  \rho^{-1} \cdot (\nu \sigma \tilde{\tau})^2\cdot s^{3} \cdot\log^4\big(kTd/(\delta \lambda)\big), $$ with probability at least $1-\delta$, we have that for every tuple $ (\tau,t,i,\zeta) $, it holds
    \begin{align*} 
        \Big | \big \langle X_{t,i}, \hat\theta_{\tau,\lambda}^{t-1,\zeta}-\theta^* \big\rangle \Big | & \lesssim    \sigma(\lambda^{1/2}+\nu+\sigma)\log^{3/2}(kTd/\delta)  \cdot  \Big(\sqrt{s/|E_{\tau-1}|}+\big \|[X_{t,i}]_{S_{\tau-1}} \big \|_{[\hat{\mat \Gamma}_{\tau-1,\lambda}^{t-1,\zeta}]^{-1}}\Big),
    \end{align*} 
    where 
    $$\hat{\mat\Gamma}_{\tau-1,\lambda}^{t-1,\zeta} = \lambda \textbf{I}_{|S_{\tau-1}|}+ \sum_{t'\in\Psi_{\tau}^{t-1,\zeta}}[X_{t'}]_{S_{\tau-1}}[X_{t'}]_{S_{\tau-1}}^\top. $$    
\end{lemma*}
\begin{proof}
    To simplify notation, for each fixed tuple $(\tau,t,\zeta)$, we denote by 
    \begin{align*}
        \vct{\varepsilon}^{t-1,\zeta}_{\tau} &= \{\varepsilon_{t'}\}_{t'\in\Psi_{\tau}^{t-1,\zeta}} \in\mathbb R^{m} ,\\ \mat Y^{t-1,\zeta}_{\tau} &= \{Y_{t'}\}_{t'\in \Psi_{\tau}^{t-1,\zeta}}\in\mathbb R^{m},\ \mat X^{t-1,\zeta}_{\tau} = \{X_{t'}\}_{t'\in\Psi_{\tau}^{t-1,\zeta}}\in\mathbb R^{d\times m},
    \end{align*}
    where  $ m $ is the number of elements in set $\Psi_{\tau}^{t-1,\zeta}$. Recall that 
    $$\hat{\mat\Gamma}_{\tau-1,\lambda}^{t-1,\zeta} = \lambda \textbf{I}_{|S_{\tau-1}|}+ \sum_{t'\in\Psi_{\tau}^{t-1,\zeta}}[X_{t'}]_{S_{\tau-1}}[X_{t'}]_{S_{\tau-1}}^\top, $$
    which is the sample covariance matrix corresponding to $\mat X^{t-1,\zeta}_{\tau} $ restricted on $S_{\tau-1}$. Recall that $\hat\theta_{\tau,\lambda}^{t-1,\zeta}$ is the ridge  estimator with generalized support $S_{\tau-1}$ using data $(\mat X^{t-1,\zeta}_{\tau},\mat Y^{t-1,\zeta}_{\tau})$. Particularly, we have 
    \begin{align*}
        \big[\hat\theta_{\tau,\lambda}^{t-1,\zeta}\big]_{S_{\tau-1}}= \hat{\mat \Gamma}_{\tau-1,\lambda}^{t-1,\zeta}\cdot \big[\mat X^{t-1,\zeta}_{\tau}\big]_{S_{\tau-1}} ^\top \mat Y^{t-1,\zeta}_{\tau}.
    \end{align*}
    Note that 
    $$
    \mat Y^{t-1,\zeta}_{\tau}= \big[\mat X^{t-1,\zeta}_{\tau}\big]_{S_{\tau-1}}[\theta^*]_{S_{\tau-1}}+\big[\mat X^{t-1,\zeta}_{\tau}\big]_{S^c_{\tau-1}}[\theta^*]_{S^c_{\tau-1}}+\vct{\varepsilon}^{t-1,\zeta}_{\tau}.
    $$
    Then we have   
	\begin{align*} 
     &\big[\hat\theta_{\tau,\lambda}^{t-1,\zeta} -\theta^*\big]_{S_{\tau-1}}=  \big[\hat{\mat \Gamma}_{\tau-1,\lambda}^{t-1,\zeta}\big]^{-1}\Big( -\lambda  [\theta^*]_{S_{\tau-1}} + [\mat X^{t-1,\zeta}_{\tau}]_{S_{\tau-1}}^\top\big(\vct{\varepsilon}^{t-1,\zeta}_{\tau} + [\mat X^{t-1,\zeta}_{\tau}]_{S^c_{\tau-1}}[\theta^*]_{S^c_{\tau-1}}\big)\Big). 
    \end{align*}
	Hence, for each $X_{t,i}$, we can upper bound the prediction error  $\langle X_{t,i}, \hat\theta_{\tau,\lambda}^{t-1,\zeta}-\theta^*\rangle$ as 
	\begin{align}
     \big | \langle X_{t,i}, \hat\theta_{\tau,\lambda}^{t-1,\zeta}-\theta^*\rangle \big | \le  \Big|\big\langle [X_{t,i}]_{S_{\tau-1}^c}, [\theta^*]_{S_{\tau-1}^c}\big\rangle\Big|  + \Big|\big\langle [X_{t,i}]_{S_{\tau-1}}, [\hat\theta_{\tau,\lambda}^{t-1,\zeta} -\theta^*]_{S_{\tau-1}} \big \rangle\Big|.
	\label{eq:ols-decomposition-suplinucb}
    \end{align}
	We can use Corollary \ref{cor:bss} to upper bound the first term in the right-hand side of inequality \eqref{eq:ols-decomposition-suplinucb} and obtain 
	\begin{equation}
	\Big|\big \langle [X_{t,i}]_{S^c_{\tau-1}}, [\theta^*]_{S^c_{\tau-1}}\big\rangle\Big| \lesssim \sigma\sqrt{\log(kTd/\delta)} \cdot \sqrt{\frac{ \nu^2s\log^2\big({kTd}/{(\delta\lambda)}\big)+\lambda r^2 }{|E_{\tau}|}},
	\label{eq:ols-suplinucb-decomp1}
	\end{equation}
    with probability at least $1-\delta$. It remains to  upper bound the second term.

    We first note that
    \begin{align}  \label{pred_error_part1}
         \Big \langle [X_{t,i}]_{S_{\tau-1}}, \big[\hat{\mat \Gamma}_{\tau-1,\lambda}^{t-1,\zeta}\big]^{-1} [\theta^*]_{S_{\tau-1}} \Big \rangle & \le  \big \| [\theta^*]_{S_{\tau-1}} \big \|_{[\hat{\mat \Gamma}_{\tau-1,\lambda}^{t-1,\zeta}]^{-1}} \cdot \big \|[X_{t,i}]_{S_{\tau-1}} \big \|_{[\hat{\mat \Gamma}_{\tau-1,\lambda}^{t-1,\zeta}]^{-1}} \notag \\
         & \le \lambda^{-1} \big \|  [\theta^*]_{S_{\tau-1}} \big\| \cdot \big \|[X_{t,i}]_{S_{\tau-1}} \big \|_{[\hat{\mat \Gamma}_{\tau-1,\lambda}^{t-1,\zeta}]^{-1}} \notag\\
         &\le \lambda^{-1} r \cdot \big \|[X_{t,i}]_{S_{\tau-1}} \big \|_{[\hat{\mat \Gamma}_{\tau-1,\lambda}^{t-1,\zeta}]^{-1}},
    \end{align}
    where the first inequality is Cauchy-Schwartz inequality and the second one follows from the fact 
    $$ \hat{\mat \Gamma}_{\tau-1,\lambda}^{t-1,\zeta} \succeq \lambda \textbf{I}_{|S_{\tau-1}|}.$$

   In the next, we claim that for each fixed tuple $ (\tau,t,\zeta)$, random vectors $\vct{\varepsilon}^{t-1,\zeta}_{\tau} $ and  $\mat X^{t-1,\zeta}_{\tau}$ are independent of each other. Recall the policies of picking arms and updating $\Psi_\tau^{t-1,\zeta}$  in \textsc{SupLinUCB} algorithm, which are specified in details in Section \ref{algorithm2}. We consider a fixed period $t'\le t-1 $ and the $ \zeta$-th group. Note that the random event that covariate $ X_{t'} $ enters $\mat X^{t-1,\zeta}_{\tau} $ only depends on the widths of confidence bands $\{\omega_{\tau,\lambda}^{t'-1,\zeta'}\}_{\zeta'\leq\zeta}$. Moreover, since $\{\omega_{\tau,\lambda}^{t'-1,\zeta'}\}_{\zeta'\leq\zeta}$ only relies on the values of design matrices $ \{\mat X^{t'-1,\zeta'}_\tau\}_{\zeta' \le \zeta} $ but not noises sequence $ \{\varepsilon_{t''} \}_{t'' \le t'}$, whether $ X_{t'} $ enters $\mat X^{t-1,\zeta}_{\tau} $ is also independent of $\{\varepsilon_{t'}\}_{t'\in\Psi_{\tau}^{t,\zeta}}$. Then we obtain the result. 
  
    The independency between $\vct{\varepsilon}^{t-1,\zeta}_{\tau} $ and $\mat X^{t-1,\zeta}_{\tau}$ enables us to establish a tighter upper bound for the second term in the right-hand side of inequality \eqref{eq:ols-decomposition-suplinucb}
    than  Lemma \ref{lem:ols}. Recall that by Assumption \ref{ass1}, given the information by epoch $\tau-1$, random variables $ [\mat X^{t-1,\zeta}]_{S_{\tau-1}} $ and $ [\mat X^{t-1,\zeta}]_{S^c_{\tau-1}}$ are independent. Moreover, $ [\mat X^{t-1,\zeta}]_{S^c_{\tau-1}}$ is also independent of $\vct{\varepsilon}^{t-1,\zeta}_{\tau} $. Hence, each element of  $$\vct{\varepsilon}^{t-1,\zeta}_{\tau} + [\mat X^{t-1,\zeta}_{\tau}]_{S^c_{\tau-1}}[\theta^*]_{S^c_{\tau-1}} $$ is a centered sub-Gaussian random variable with parameter $\nu^2+\sigma^2\|[\theta^*]_{S_{\tau-1}^c}\|_2^2$, which is also independent of covariates $ [\mat X^{t-1,\zeta}_{\tau}]_{S_{\tau-1}}$. By applying the sub-Gaussian concentration inequality and the union bound argument, we obtain with probability at least $1-\delta$ that
    \begin{align} \label{concen}
    \big \| \vct{\varepsilon}^{t-1,\zeta}_{\tau} + [\mat X^{t-1,\zeta}_{\tau}]_{S^c_{\tau-1}}^\top[\theta^*]_{S^c_{\tau-1}} \big \| \lesssim \sqrt{\big(\nu^2+\sigma^2 \big\|[\theta^*]_{S^c_{\tau-1}}\big\|^2\big)}\cdot \log(kTd/\delta). 
    \end{align}
    Here we emphasize that inequality \eqref{concen} is interpreted as a random event given the realization of $[\mat X^{t-1,\zeta}]_{S_{\tau-1}}$ for later proof. It also holds regardless the values of $[\mat X^{t-1,\zeta}]_{S_{\tau-1}}$. That is where the independency of $\vct{\varepsilon}^{t-1,\zeta}_{\tau}$ and $[\mat X^{t-1,\zeta}]_{S_{\tau-1}}$ is indispensable: otherwise, given $[\mat X^{t-1,\zeta}]_{S^c_{\tau-1}}$, elements of $\vct{\varepsilon}^{t-1,\zeta}_{\tau}$ are not i.i.d., which prevents us establishing concentration result similar to \eqref{concen}. Based on inequality \eqref{concen} and Cauchy-Schwartz inequality, with probability at least $1-\delta$, we have
    \begin{align} \label{pred_error_part2}
     & \Big \langle [X_{t,i}]_{S_{\tau-1}}, \big[\hat{\mat \Gamma}_{\tau-1,\lambda}^{t-1,\zeta}\big]^{-1} [\mat X^{t-1,\zeta}_{\tau}]_{S_{\tau-1}}^\top\big(\vct{\varepsilon}^{t-1,\zeta}_{\tau} + [\mat X^{t-1,\zeta}_{\tau}]_{S^c_{\tau-1}}[\theta^*]_{S^c_{\tau-1}}\big)\Big \rangle  \notag \\
     & \qquad \qquad \lesssim   \big \|[X_{t,i}]_{S_{\tau-1}} \big \|_{[\hat{\mat \Gamma}_{\tau-1,\lambda}^{t-1,\zeta}]^{-1}} \cdot \sqrt{\big(\nu^2+\sigma^2 \big\|[\theta^*]_{S^c_{\tau-1}}\big\|^2\big)}\cdot \log(kTd/\delta).
    \end{align}
    By combing inequalities \eqref{pred_error_part1}-\eqref{pred_error_part2}, we have 
    \begin{align} \label{pred_error_part3}
        & \Big|\big\langle [X_{t,i}]_{S_{\tau-1}}, [\hat\theta_{\tau,\lambda}^{t-1,\zeta} -\theta^*]_{S_{\tau-1}} \big \rangle\Big|  \notag   \\ 
        & \qquad \qquad \lesssim  \Big(r +  \sqrt{\big(\nu^2+\sigma^2 \big\|[\theta^*]_{S^c_{\tau-1}}\big\|^2\big)} \cdot \log(kTd/\delta)\Big)  \cdot  \big \|[X_{t,i}]_{S_{\tau-1}} \big \|_{[\hat{\mat \Gamma}_{\tau-1,\lambda}^{t-1,\zeta}]^{-1}},
    \end{align}
    with probability at least $1-\delta$. Recall that Lemma \ref{bss_error} shows that $$  \big \| [\theta^*]_{S^c_{\tau-1}} \big\|=\tilde{\mathcal{O}}(\sqrt{s/T}). $$ Hence, when $ s \ll T $, $ \| [\theta^*]_{S^c_{\tau-1}} \| \lesssim 1$. Subsequently, by combining inequalities \eqref{eq:ols-suplinucb-decomp1} and \eqref{pred_error_part3}, we obtain with probability at least $1-\delta$ that
    \begin{align*}
        \Big | \big \langle X_{t,i}, \hat\theta_{\tau,\lambda}^{t-1,\zeta}-\theta^* \big\rangle \Big | & \lesssim    \sigma(\lambda^{1/2}+\nu+\sigma)\log^{3/2}(kTd/\delta)  \cdot  \Big(\sqrt{s/|E_\tau|}+\big \|[X_{t,i}]_{S_{\tau-1}} \big \|_{[\hat{\mat \Gamma}_{\tau-1,\lambda}^{t-1,\zeta}]^{-1}}\Big),
    \end{align*}   
    which concludes the proof of Lemma \ref{improve_ols}.   
	\end{proof}

\section{Proof of Auxiliary Lemmas in Supplementary Document}
\subsection{Proof of Lemma \ref{quadratic_upper_bound}}   \label{qub_proof}
 \begin{lemma*} (Restatement of Lemma \ref{quadratic_upper_bound})
    With probability at least $1-\delta$, we have 
    $$
    \big\| \hat\theta_{\tau,\lambda}-\theta^* \big\|^2_{\hat{\mat \Gamma}_{\tau,\lambda}}\lesssim \nu^2s\log^2\big({kTd}/{(\delta\lambda)}\big)+\lambda r^2.
    $$
 \end{lemma*}

    \begin{proof}
        We use the optimality of $\hat\theta_{\tau,\lambda}$ to derive an upper bound of the quadratic form of $\hat\theta_{\tau,\lambda}-\theta^*.$ In the following, we use $S^*$ to denote the support of true parameter $ \theta^*$. Let $\bar{\theta}_{\tau,\lambda} $ be the projected ridge estimator with support $S^*$ using data $\{X_t,Y_t\}_{t\in E_\tau}$, i.e., 
        \begin{align*}
         \bar{\theta}_{\tau,\lambda}=\text{Proj}_r \Big \{\operatorname*{argmin}_{\supp^+{(\theta)}=S^*} \Big \{ \sum_{t\in E_\tau} (Y_t-X_t^\top\theta)^2+\lambda \|\theta\|^2 \Big\} \Big\},
        \end{align*}  
        where $\text{Proj}_r\{\cdot\} $ denotes the projection operator on a centered $\ell_2$-ball with radius $r$. We remark that $\bar{\theta}_{\tau,\lambda}$ is unavailable in practice since $S^*$ is unknown. However, it is not necessary in the  SLUCB algorithm and only serves as an auxiliary estimator in theoretical analysis. 
        Recall the definition of BSS estimator $ \hat{\theta}_{\tau,\lambda}$, which is given by
        \begin{align}  \label{BSS_def}
            \hat{\theta}_{\tau,\lambda} &=\operatorname*{argmin}_{ \theta} \Big \{  \sum_{t\in E_\tau} (Y_t-X_t^\top\theta)^2+\lambda \|\theta\|^2 \Big\}, \notag \\
            & \text{s.t.}\  \supp^+{(\theta)} \supseteq  S_{\tau-1},   |\supp^+{(\theta)}| \le \tau s,\|\theta\|\le r, 
        \end{align} 
        where $ S_{\tau-1} $ is the generalized support recovered by the $(\tau-1)$-th epoch. Since $\supp^+(\bar{\theta}_{\tau,\lambda})=S*$, we have $\supp^+(\bar{\theta}_{\tau,\lambda})=S^*\cup S_{\tau-1} \supseteq S_{\tau-1}$. We also have $|S^*\cup S_{\tau-1}| \le \tau s $.  Hence, $ \bar{\theta}_{\tau,\lambda}$ satisfies the constraints in equation \eqref{BSS_def}. Then by the optimality of $\hat{\theta}_{\tau,\lambda} $, we have 
         \begin{equation*}
         \sum_{t\in E_\tau}\big(Y_t-X_t^\top\hat{\theta}_{\tau,\lambda}\big)^2 +\lambda \|\hat{\theta}_{\tau,\lambda} \|^2 \leq \sum_{t\in E_\tau}\big(Y_t-X_t^\top\bar{\theta}_{\tau,\lambda}\big)^2+\lambda \|\bar{\theta}_{\tau,\lambda}\|^2,
         \end{equation*}
         and furthermore,
        \begin{equation} \label{ineq0001}
         \hat{\theta}_{\tau,\lambda}^\top  \hat{\mat \Gamma}_{\tau,\lambda}  \hat{\theta}_{\tau,\lambda} \le   \bar{\theta}_{\tau,\lambda}^\top  \hat{\mat \Gamma}_{\tau,\lambda}  \bar{\theta}_{\tau,\lambda} +2\sum_{t\in E_\tau} Y_tX_t^\top(\hat{\theta}_{\tau,\lambda}-\bar{\theta}_{\tau,\lambda}),
        \end{equation}
         where $ \hat{\mat \Gamma}_{\tau,\lambda}=\sum_{t\in E_\tau}X_tX_t^\top+\lambda \textbf{I}_d$.  
         Note that we have  
         $$Y_t=X_t^\top\theta^* + \varepsilon_t=X_t^\top\bar{\theta}_{\tau,\lambda} + \varepsilon_t+X_t^\top(\theta^*-\bar{\theta}_{\tau,\lambda}).$$
         After some algebra, we obtain
         \begin{align*}
             \sum_{t\in E_\tau} Y_tX_t^\top(\hat{\theta}_{\tau,\lambda}-\bar{\theta}_{\tau,\lambda})
            &=\bar{\theta}_{\tau,\lambda}^\top \hat{\mat \Gamma}_{\tau,\lambda} (\hat{\theta}_{\tau,\lambda}-\bar{\theta}_{\tau,\lambda}) +(\theta^*-\bar{\theta}_{\tau,\lambda})^\top\hat{\mat \Gamma}_{\tau,\lambda}(\hat{\theta}_{\tau,\lambda}-\bar{\theta}_{\tau,\lambda}) \\
            &+\Big(\sum_{t\in E_\tau}\varepsilon_tX_t^\top\Big)(\hat{\theta}_{\tau,\lambda}-\bar{\theta}_{\tau,\lambda})-\lambda (\hat{\theta}_{\tau,\lambda}-\bar{\theta}_{\tau,\lambda})^\top \theta^*. 
         \end{align*}
         Then inequality \eqref{ineq0001} becomes 
         \begin{align}  \label{ineq0002}
             &(\hat{\theta}_{\tau,\lambda}-\bar{\theta}_{\tau,\lambda})^{\top}\hat{\mat \Gamma}_{\tau,\lambda}(\hat{\theta}_{\tau,\lambda}-\bar{\theta}_{\tau,\lambda}) \notag \\
             &\quad  \le 2\Big(\sum_{t\in E_\tau}\varepsilon_tX_t^\top \Big)(\hat{\theta}_{\tau,\lambda}-\bar{\theta}_{\tau,\lambda}) + 2(\theta^*-\bar{\theta}_{\tau,\lambda})^\top\hat{\mat \Gamma}_{\tau,\lambda}(\hat{\theta}_{\tau,\lambda}-\bar{\theta}_{\tau,\lambda})-2\lambda(\hat{\theta}_{\tau,\lambda}-\bar{\theta}_{\tau,\lambda})^\top\theta^*.
         \end{align}

         In the next, we establish upper bounds for the three parts of the right-hand side of inequality \eqref{ineq0002} individually. We first introduce some notations. By the definition of $\hat{\theta}_{\tau,\lambda}$, we have
         \begin{align*}
            \hat{\theta}_{\tau,\lambda} &=\operatorname*{argmin}_{\theta} \Big\{\sum_{t\in E_\tau} (Y_t-X_t^\top\theta)^2+\lambda \|\theta\|^2 \Big\} \\
            & \text{s.t. these exists some subset}\ S\in[d]:\ \supp^+{(\theta)}=S,S\supseteq S_{\tau-1},|S|\le \tau s,\|\theta\|\le r.
         \end{align*} 
         where the optimization is taken over all feasible set $S$. From now on, for convenience, given a fixed subset $S\subseteq [d]$ with $|S|\le \tau s$, we denote by $$ \tilde{S}=S\cup \supp(\theta^*).$$ Particularly, let $$\tilde{S}_{\tau}=\supp^+(\hat{\theta}_{\tau,\lambda})\cup \supp(\theta^*). $$ Then $|\tilde{S}|\le (\tau+1)s$ by definition. We are ready to upper bound the right-hand side of inequality \eqref{ineq0002}. 
         
         For the first part, note that $\supp^+(\hat{\theta}_{\tau,\lambda}-\bar{\theta}_{\tau,\lambda})=\tilde{S}_\tau$. Then by Cauchy-Schwartz inequality, we have 
         \begin{align} \label{part1}
             \Big(\sum_{t\in E_\tau}\varepsilon_tX_t^\top \Big)(\hat{\theta}_{\tau,\lambda}-\bar{\theta}_{\tau,\lambda}) & \le  
             \Big \| \sum_{t\in E_\tau}\varepsilon_t \cdot [X_t]_{\tilde{S}_\tau} \Big\|_{ [\hat{\mat \Gamma}_{\tau,\lambda}]_{\tilde{S}_\tau}^{-1}} \cdot \Big\| [\hat{\theta}_{\tau,\lambda}-\bar{\theta}_{\tau,\lambda}]_{\tilde{S}_\tau} \Big \|_{[\hat{\mat \Gamma}_{\tau,\lambda}]_{\tilde{S}_\tau}}   
              \notag  \\
             &= \Big \| \sum_{t\in E_\tau}\varepsilon_t \cdot [X_t]_{\tilde{S}_\tau} \Big\|_{ [\hat{\mat \Gamma}_{\tau,\lambda}]_{\tilde{S}_\tau}^{-1}} \cdot \big\| \hat{\theta}_{\tau,\lambda}-\bar{\theta}_{\tau,\lambda} \big \|_{\hat{\mat \Gamma}_{\tau,\lambda}}   \notag \\
             & \le \max_{\tilde{S}} \Big\{ \Big \| \sum_{t\in E_\tau}\varepsilon_t \cdot[X_t]_{\tilde{S}} \Big\|_{ [\hat{\mat \Gamma}_{\tau,\lambda}]^{-1}_{\tilde{S}}} \Big\}  \cdot \big\| \hat{\theta}_{\tau,\lambda}-\bar{\theta}_{\tau,\lambda} \big\|_{\hat{\mat \Gamma}_{\tau,\lambda}}, 
         \end{align}
         where the maximization in the last step is taken over all possible $\tilde{S}$, whose number is at most $d^{(\tau+1) s}$. 
     For each fixed $\tilde{S}$, by applying the self-normalized martingale concentration inequality (i.e., Lemma \ref{self_normalized_ineq} provided in Section \ref{cite_lemma}), we obtain with probability at least $1-\delta$ that 
     \begin{align*}
         \Big \| \sum_{t\in E_\tau}\varepsilon_t \cdot[X_t]_{\tilde{S}} \Big\|_{ [\hat{\mat \Gamma}_{\tau,\lambda}]^{-1}_{\tilde{S}}} \le \nu \sqrt{2\log\Big(\frac{\det \big( [\hat{\mat \Gamma}_{\tau,\lambda}]_{\tilde{S}}\big)^{1/2}\det\big(\lambda [\textbf{I}_d]_{\tilde{S}}\big)^{-1/2} }{\delta}\Big)}.
     \end{align*}
      Since the number of possible $\tilde{S}$ is upper bounded by $d^{(\tau+1)s}$, by applying the union bound argument, we obtain with probability at least $1-\delta$ that,
      \begin{align} \label{Lambda_def}
         \Big \| \sum_{t\in E_\tau}\varepsilon_t \cdot[X_t]_{\tilde{S}} \Big\|_{ [\hat{\mat \Gamma}_{\tau,\lambda}]^{-1}_{\tilde{S}}} & \le \nu \sqrt{2\log\Big(\frac{\det \big( [\hat{\mat \Gamma}_{\tau,\lambda}]_{\tilde{S}}\big)^{1/2}\det\big(\lambda [\textbf{I}_d]_{\tilde{S}}\big)^{-1/2} }{\delta}\Big)+\tau s\log(d)} \notag\\
         &:=\Lambda_{\tilde{S}},
      \end{align} 
      hold for all possible $\tilde{S}$, which further implies that  
      \begin{align} \label{part1_bound}
         \max_{\tilde{S}} \Big \| \sum_{t\in E_\tau}\varepsilon_t \cdot[X_t]_{\tilde{S}} \Big\|_{ [\hat{\mat \Gamma}_{\tau,\lambda}]^{-1}_{\tilde{S}}} \le \max_{\tilde{S}} \big \{ \Lambda_{\tilde{S}}\big \},
      \end{align}
      holds with probability at least $1-\delta$.

     For the second part of the right-hand side of inequality \eqref{ineq0002}, by Cauchy-Schwartz inequality and the  confidence region of ridge estimator (i.e., Lemma \ref{confidence_region} provided in Section \ref{cite_lemma}), we obtain with probability at least $ 1-\delta$ that  
      \begin{align} \label{part2}
         2(\theta^*-\bar{\theta}_{\tau,\lambda})^\top\hat{\mat \Gamma}_{\tau,\lambda}(\hat{\theta}_{\tau,\lambda}-\bar{\theta}_{\tau,\lambda})& \le \big\| \theta^*-\bar{\theta}_{\tau,\lambda} \big\|_{\hat{\mat \Gamma}_{\tau,\lambda}}\cdot \big\|\hat{\theta}_{\tau,\lambda}-\bar{\theta}_{\tau,\lambda} \big\|_{\hat{\mat \Gamma}_{\tau,\lambda}} \notag \\
         &\lesssim (\Lambda_{S^*}+\lambda^{1/2}r) \cdot \big\|\hat{\theta}_{\tau,\lambda}-\bar{\theta}_{\tau,\lambda} \big\|_{\hat{\mat \Gamma}_{\tau,\lambda}},
      \end{align}
      where  $S^*=\supp({\bar{\theta}_\tau}) \cup \supp(\theta^*) $. 
     
      For the last part of the right-hand side of inequality \eqref{ineq0002}, since $\theta^*, \hat{\theta}_\tau,$ and $\bar{\theta}_\tau$ all are in the centered $\ell_2$-ball with radius $r$, we have  
      \begin{align} \label{part3}
       \big| 2\lambda(\hat{\theta}_{\tau,\lambda}-\bar{\theta}_{\tau,\lambda})^\top\theta^*\big| \le 4\lambda r^2. 
      \end{align}

      Finally, by combing inequalities \eqref{ineq0002}-\eqref{part3} and Lemma \ref{confidence_region} (provided in Section \ref{cite_lemma}), we obtain with probability at least $1-\delta$ that 
     \begin{align} \label{ineq0003}
         \big\| \hat{\theta}_{\tau,\lambda}-\theta^* \big\|_{\hat{\mat \Gamma}_{\tau,\lambda}} & \le \big\| \hat{\theta}_{\tau,\lambda}-\bar{\theta}_{\tau,\lambda} \big \|_{\hat{\mat \Gamma}_{\tau,\lambda}}+\big\| \bar{\theta}_{\tau,\lambda}-\theta^* \big\|_{\hat{\mat \Gamma}_{\tau,\lambda}} \notag \\
         & \lesssim \max_{\tilde{S}} \big\{\Lambda_{\tilde{S}}\big\}+\Lambda_{S^*}+\lambda^{1/2}r,  
     \end{align}
     where the maximization is taken over all subsets $\tilde{S} \in [d]$ with capacity at most $(\tau+1)s$. 
     
     In order to conclude the proof of Lemma \ref{quadratic_upper_bound}, it remains to establish an upper bound for $\Lambda_{\tilde{S}}$. By Proposition \ref{X_uniform_bound}, with probability at least $1-\delta $, for all $t\in E_\tau$, $i \in \mathcal{A}_t$,  and $j \in [d], $  
     $$
     \big| [X_{t,i}]_j \big| \lesssim \sigma \log(kTd/\delta).
     $$
     Then by determinant-trace inequality, we have
     \begin{align*}
         \det\big([\hat{\mat \Gamma}_{\tau,\lambda}]_{\tilde{S}}\big) \lesssim \Big (\lambda+\frac{T\sigma^2\log^2(kTd/\delta)}{(\tau+1)s} \Big)^{(\tau+1)s}.
     \end{align*}
     Recall the definition of $\Lambda_{\tilde{S}}$ in equation \eqref{Lambda_def}. Then with probability at least $1-\delta$, we have 
     \begin{align} \label{Lambda_bound}
         \Lambda_{\tilde{S}} &\le \nu\sqrt{(\tau+1) s \log\Big(\delta^{-1}+\frac{T\sigma^2\log^2(kTd/\delta)}{\delta\lambda(\tau+1)s}\Big)+\tau s \log(d)}  \notag \\
         & \lesssim \nu\sqrt{s}\cdot \log\big({kTd}/{(\delta\lambda)}\big),
     \end{align}
     which further implies that 
     $$  
     \max_{\tilde{S}}\big\{\Lambda_{\tilde{S}}\big\}, \Lambda_{S^*} \lesssim \nu\sqrt{s}\cdot \log\big({kTd}/{(\delta\lambda)}\big). 
     $$
     As a result, by inequality \eqref{ineq0003}, we finally obtain with probability at least $1-\delta$ that
     \begin{align*} 
           \big\| \hat\theta_{\tau,\lambda}-\theta^*\big \|^2_{\hat{\mat \Gamma}_{\tau,\lambda}} & \lesssim \Big(\nu\sqrt{s}\cdot \log\big({kTd}/{(\delta\lambda)}\big) +\lambda^{1/2}r\Big)^2 \lesssim \nu^2s\log^2\big({kTd}/{(\delta\lambda)}\big)+\lambda r^2 ,
     \end{align*}
     which concludes the proof of Lemma \ref{quadratic_upper_bound}.
     
    \end{proof}

\subsection{Proof of Lemma \ref{inf_norm_upper_bound} } \label{inf_proof}
\begin{lemma*}(Restatement of Lemma \ref{inf_norm_upper_bound})
        With probability at least $1-\delta$, we have 
        $$
        \big\|\big[\hat{\mat \Gamma}_{\tau,\lambda}\big]_{SS^c}\big\|_{\infty} \le  8\sigma{\log(d/\delta)}\sqrt{|E_{\tau}|}.
        $$
    \end{lemma*}
    \begin{proof}
        By the definition of matrix $[\hat{\mat \Gamma}_{\tau,\lambda}]_{SS^c}$, for every $i \in S$ and $j \in S^c$, we have 
    \begin{align*}
    \big[\hat{\mat \Gamma}_{\tau,\lambda}\big]_{ij}=\sum_{t\in E_{\tau}}[X_t]_i[X_t]_{j}^{\top}.
    \end{align*}  
     Recall that in epoch $E_\tau$, the SLUCB algorithm picks arm only based on the information in dimension $\supp^+(\hat \theta_{\tau-1,\lambda})$. Also recall that SLUCB algorithm requires $$S=\supp^+(\hat \theta_{\tau,\lambda}) \supseteq \supp^+(\hat \theta_{\tau-1,\lambda}), $$ which implies that
     $$
       S^c \subseteq \big[\supp^+(\hat \theta_{\tau-1,\lambda})\big]^c.
     $$
     Since all covariates $X_{t,i} $  have independent coordinates, for each $j\in S^c $, $ \{[X_t]_j\}_{t \in E_\tau} $ are independent centered sub-Gaussian random variables. As a result, conditional on $[X_t]_i$, the random variable $[\hat{\mat \Gamma}_{\tau,\lambda}]_{ij} $ is sub-Gaussian with parameter $\sigma^2\cdot \sum_{t\in E_{\tau}}[X_t]_i^2$. Note that
    \begin{align*}
    \big\|\big[\hat{\mat \Gamma}_{\tau,\lambda}\big]_{SS^c}\big\|_{\infty}=\max_{i\in S,j \in S^c}\big| \big[\hat{\mat \Gamma}_{\tau,\lambda}\big]_{ij}\big|,
    \end{align*}
    which is the maximum of $|S|\cdot|S^c|$ sub-Gaussian random variables whose parameters are uniformly upper bounded by $ \sigma^2\cdot \max_{i\in S} \sum_{t\in E_{\tau}}[X_t]_i^2 $. By applying the concentration inequality for the maximal of sub-Gaussian random variables, we obtain with probability at least $1-\delta$ that
    \begin{align*}
        \big\|\big[\hat{\mat \Gamma}_{\tau,\lambda}\big]_{SS^c}\big\|_{\infty}\le 4\sigma\sqrt{\max_{i\in S} \sum_{t\in E_{\tau}}[X_t]_i^2 \cdot \log(d/\delta)}\le 4\sigma\sqrt{ \sum_{t\in E_{\tau}} \big( \max_{i\in S} [X_t]_i \big)^2  \cdot \log(d/\delta)}.
    \end{align*} 
    Now we use the concentration inequality for the maximal of sub-Gaussian random variables again. Then with probability at least $1-\delta$, for each $t$,
     \begin{align*}
      \max_{i\in S}\big|[X_t]_i \big| \le 2\sigma\sqrt{\log(d/\delta)}.
     \end{align*}
    Hence, we obtain with probability at least $1-\delta$ that  
     \begin{align*}
        \big\|\big[\hat{\mat \Gamma}_{\tau,\lambda}\big]_{SS^c}\big\|_{\infty} \le  8\sigma{\log(d/\delta)}\sqrt{|E_{\tau}|}. 
     \end{align*}
    Similarly, we also have 
    \begin{align*}
        \big\|\big[\hat{\mat \Gamma}_{\tau,\lambda}\big]_{S^cS}\big\|_{\infty} \le  8\sigma{\log(d/\delta)}\sqrt{|E_{\tau}|}. 
    \end{align*}
       with probability at least $1-\delta$, which concludes the proof of Lemma \ref{inf_norm_upper_bound}.
    \end{proof}
\subsection{Proof of Lemma \ref{l1_norm_upper_bound} } \label{l1_proof}
    \begin{lemma*} (Restatement of Lemma \ref{l1_norm_upper_bound})
        For arbitrary constant $\varsigma>0$, when $$n_0\gtrsim \max\{\varsigma^2\rho^{-1}, \sigma^2\rho^{-2}{\log(d\tilde \tau/\delta)}\},$$ with probability at least $1-\delta$, we have  
        \begin{align*}
            \|\hat\theta_{\tau,\lambda}-\theta^*\|_1   \lesssim \nu s/\varsigma \cdot \sqrt{\tilde \tau}\cdot \log\big(kTd/(\delta \lambda)\big). 
        \end{align*}
    \end{lemma*}
    
    \begin{proof}
        To simplify notation, let 
        $$
        \tilde{S}_\tau = S \cup \supp(\theta^*) =\supp^+(\hat \theta_{\tau,\lambda}) \cup \supp(\theta^*)  . 
        $$
       Then we have $|\tilde{S}_\tau|\leq s({\tau}+1)$. Recall that in SLUCB algorithm, the first $n_0$ arms in epoch $E_\tau$ are chosen uniformly at random. Then by matrix Bernstein inequality \citep{tropp2012user} and Assumption \ref{ass1} (A2), for arbitrary constant $\varsigma>0$,when $n_0\gtrsim \max\{\varsigma^2\rho^{-1}, \sigma^2\rho^{-2}{\log(d\tilde \tau/\delta)}\},$
         we have  
       \begin{align} \label{partial_lower_bound}
       \big[\hat{\mat \Gamma}_{\tau,\lambda}\big]_{\tilde{S}_\tau}\succeq  \varsigma^2\cdot \textbf{I}_{|\tilde{S}_\tau|},
       \end{align}    
       {with probability at least $1-\delta$.}  Since $\supp(\hat\theta_{\tau,\lambda}-\theta^*)\subseteq \tilde{S}_\tau$, when inequality \eqref{partial_lower_bound} holds, we have
       $$ \label{eq:basic-ineq6}
       \sqrt{(\hat\theta_{\tau,\lambda}-\theta^*)^\top\hat{\mat \Gamma}_{\tau,\lambda}(\hat\theta_{\tau,\lambda}-\theta^*)} = 
       \sqrt{[\hat\theta_{\tau,\lambda}-\theta^*]_{\tilde{S}_\tau}^\top\big[\hat{\mat \Gamma}_{\tau,\lambda}\big]_{\tilde{S}_\tau }[\hat\theta_{\tau,\lambda}-\theta^*]_{\tilde{S}_\tau}}
       \geq \varsigma \cdot \big\|[\hat\theta_{\tau,\lambda}-\theta^*]_{\tilde{S}_\tau}\big\|.
       $$
       Hence, we obtain with probability at least $1-\delta$ that 
       \begin{align}
       \|\hat\theta_{\tau,\lambda}-\theta^*\|_1
       & = \big\|[\hat\theta_{\tau,\lambda}-\theta^*]_{\tilde{S}_\tau}\big\|_1 \leq \sqrt{s(\tilde\tau+1)}\cdot \big\|[\hat\theta_{\tau,\lambda}-\theta^*]_{\tilde{S}_\tau}\big\| \notag\\
       &\leq \big\|[\hat\theta_{\tau,\lambda}-\theta^*]_{\tilde{S}_\tau}\big\|_{[\hat{\mat \Gamma}_{\tau,\lambda}]_{\tilde{S}_\tau}}  \cdot \sqrt{s(\tilde \tau+1)}/\varsigma  \lesssim \nu s/\varsigma \cdot \sqrt{\tilde \tau}\cdot \log\big(kTd/(\delta \lambda)\big), \label{bss:1norm_b}
       \end{align}
       where the last inequality follows from Lemma \ref{quadratic_upper_bound}. Now we finish the proof of Lemma \ref{l1_norm_upper_bound}.
    \end{proof}

\subsection{Proof of Lemma \ref{lem:ols-empirical-process}} \label{pe}
\begin{lemma*} (Restatement of Lemma \ref{lem:ols-empirical-process})
        For each $\tau \in [\tilde\tau]$ and $t\in E_\tau$, we have 
        \begin{equation*}
            \big\| [\hat\theta_{\tau,\lambda}^{t-1}-\theta^*]_{S_{\tau-1}}  \big\|^2_{\hat{\mat \Gamma}^{t-1}_{\tau-1,\lambda}}   \lesssim \big(\nu^2+ \sigma^2\cdot \big\|  [\theta^*]_{S^c_{\tau-1}}\big\|^2 \big)\cdot \big(   s  \log^2\big({kTd}/{(\delta\lambda)}\big) + \lambda r^2  \big),
        \end{equation*}
        with probability at least $1-\delta$.
\end{lemma*}

\begin{proof}[Proof]
    For each $t' \in E_\tau$, we have decomposition 
    \begin{align*}
     Y_{t'}=\langle X_{t'},\theta^*\rangle+\varepsilon_{t'}&=\big \langle [X_{t'}]_{S_{\tau-1}},[\theta^*]_{S_{\tau-1}}\big\rangle+\big\langle [X_{t'}]_{S^c_{\tau-1}},[\theta^*]_{S^c_{\tau-1}}\big \rangle+\varepsilon_{t'}\\
     &=\big \langle [X_{t'}]_{S_{\tau-1}},[\theta^*]_{S_{\tau-1}}\big\rangle+\tilde \varepsilon_{t'}.
    \end{align*} 
    where $$\tilde \varepsilon_{t'}=\big \langle [X_{t'}]_{S^c_{\tau-1}},[\theta^*]_{S^c_{\tau-1}} \big\rangle+\varepsilon_{t'}.$$ Given the information by epoch $\tau-1$, for each $t' \in E_\tau$ and $i \in \mathcal{A}_{t'}$, by Assumption \ref{ass1} (A3), $[X_{t',i}]_{S^c_{\tau-1}}$ is independent of $ [X_{t',i}]_{S_{\tau-1}} $. Recall that in epoch $\tau$, all decisions on picking arms only depend on the information of covariates in dimension $S_{\tau-1}$. Hence, $\tilde{\varepsilon}_{t'}$ is independent of the design $[X_{t'}]_{S_{\tau-1}}$. Moreover, conditioning on the information by epoch $\tau-1$, $\{\tilde \varepsilon_{t'}\}_{t' \in E_\tau}$ are independent centered sub-Gaussian random variables with parameter 
    $$
    \nu^2+ \sigma^2\cdot \big\|  [\theta^*]_{S^c_{\tau-1}}\big\|^2.
    $$
    Recall the definition of $ \hat\theta_{\tau,\lambda}^{t-1} $, i.e., 
    $$
    \hat\theta_{\tau,\lambda}^{t-1}= \operatorname*{argmin}_{\supp^+(\theta)=S_{\tau-1}}  \sum_{t'\in E_{\tau}^{t-1}}\big|Y_{t'}- \langle X_{t'},\theta\rangle\big|^2 +\lambda   \| \theta \|^2,
    $$
    where $ E_{\tau}^{t-1}=\{t': t'\le t-1,t' \in E_\tau \} $. Then by using the confidence region of ridge regression estimator (i.e., Lemma \ref{confidence_region} provided in Section \ref{cite_lemma}), we have with probability at least $1-\delta$ that
    \begin{align*}
        & [\hat\theta_{\tau,\lambda}^{t-1}-\theta^*]_{S_{\tau-1}}^\top \hat{\mat \Gamma}^{t-1}_{\tau-1,\lambda} [\hat\theta_{\tau,\lambda}^{t-1}-\theta^*]_{S_{\tau-1}} = \big\| [\hat\theta_{\tau,\lambda}^{t-1}-\theta^*]_{S_{\tau-1}}  \big\|^2_{\hat{\mat \Gamma}^{t-1}_{\tau-1,\lambda}} \\
        & \lesssim \big(\nu^2+ \sigma^2\cdot \big\|  [\theta^*]_{S^c_{\tau-1}}\big\|^2 \big)\cdot \log\Big(\frac{\det(\hat{\mat \Gamma}^{t-1}_{\tau-1,\lambda})^{1/2}\det(\lambda \textbf{I}_{|S_{\tau-1}|})^{-1/2}}{\delta} \Big)+\lambda r^2
    \end{align*}
    where $$\hat{\mat \Gamma}^{t-1}_{\tau-1,\lambda}=\lambda \textbf{I}_{|S_{\tau-1}|}+\sum_{t' \in E^{t-1}_{\tau}} [X_{t'}]_{S_{\tau-1}} [X_{t'}]_{S_{\tau-1}}^\top. $$ 
    Similar to the proof of Lemma \ref{quadratic_upper_bound}, we have 
    \begin{align*}
        \det\big(\hat{\mat \Gamma}^{t-1}_{\tau-1,\lambda}\big) \lesssim \Big (\lambda+\frac{T\sigma^2\log^2(kTd/\delta)}{\tau s} \Big)^{\tau s}, 
    \end{align*}    
    with probability at least $1-\delta$.
    Hence, we obtain with probability at least $1-\delta$ that
    \begin{align*}
        \big\| [\hat\theta_{\tau,\lambda}^{t-1}-\theta^*]_{S_{\tau-1}}  \big\|^2_{\hat{\mat \Gamma}^{t-1}_{\tau-1,\lambda}}  &\lesssim \big(\nu^2+ \sigma^2\cdot \big\|  [\theta^*]_{S^c_{\tau-1}}\big\|^2 \big)\cdot \sqrt{\tau s \cdot \log\Big(\delta^{-1}+\frac{T\sigma^2\log^2(kTd/\delta)}{\delta\lambda\tau s}\Big)}  \notag \\
        & \lesssim \big(\nu^2+ \sigma^2\cdot \big\|  [\theta^*]_{S^c_{\tau-1}}\big\|^2 \big)\cdot \sqrt{s}\cdot \log\big({kTd}/{(\delta\lambda)}\big),
    \end{align*}
    which concludes the proof of Lemma \ref{lem:ols-empirical-process}.
    \end{proof}

\section{Other Auxiliary Lemmas} \label{cite_lemma}
In this section, we list the results established in \citep{abbasi2011improved}, which are used extensively in our proof. 
\begin{lemma} (Self-normalized Martingale Concentration inequality) \label{self_normalized_ineq}
    Let $\{X_t\}_{t\ge 1}$ be an $\mathbb{R}^d$-valued stochastic process and $\{\eta_t\}_{t\ge 1} $ be a real valued stochastic process. Let $$\mathcal{F}_t=\sigma(X_1,\ldots,X_{t+1},\eta_1,\ldots,\eta_t)$$ be the natural filtration generated by $\{X_t\}_{t\ge 1}$ and $\{\eta_t\}_{t\ge 1}$. Assume that for each $t$, $\eta_t$ is centered and sub-Gaussian with variance proxy $R^2$ given $\mathcal{F}_{t-1}$, i.e., 
    \begin{align*}
        \mathbb{E}[\eta_t | \mathcal{F}_{t-1}]&=0,\\
        \mathbb{E}\big[\exp\{\lambda \eta_t \} \big| \mathcal{F}_{t-1}\big]& \le \exp\{ \lambda^2R^2/2\},\ \forall \lambda \in \mathbb{R}.
    \end{align*}
    Also assume that $\mat V$ is a $d \times d$ positive definite matrix. For any $t\ge 0$. we define 
    \begin{align*} 
        \hat{\mat \Gamma}_t=\mat V+\sum_{u=1}^t X_uX_u^\top, \ \mat S_t=\sum_{u=1}^t\eta_uX_u.
    \end{align*}
    Then for any $\delta>0$, with probability at least $1-\delta$, for all $t\ge 0$, 
    \begin{align*}
         \| \mat S_t \|_{\hat{\mat \Gamma}_t^{-1}}\le R \sqrt{2\log\Big(\frac{\det(\hat {\mat \Gamma}_t)^{1/2}\det(\mat V)^{-1/2} }{\delta}\Big)}.
    \end{align*}
    \end{lemma}

    \begin{lemma} (Confidence Region of Ridge Regression Estimator under Dependent Design) \label{confidence_region}
    Consider the stochastic processes $\{(X_t,Y_t,\eta_t)\}_{t\ge 1 }$. Assume that $\{X_t\}_{t\ge 1 } $ and $\{\eta_t\}_{t\ge 1 } $ satisfy the condition in Lemma \ref{self_normalized_ineq} and $Y_t=\langle X_t,\theta^*\rangle+\eta_t$. Also assume that $ \|\theta^*\|_2 \le r$. For any $T\ge 0$, let 
    $$\bar{\theta}_t=\big(\mat X_{1:t}^\top\mat X_{1:t}+\lambda \textbf{I}_d\big)^{-1}\mat X_{1:t}^\top\mat Y_{1:t} $$ 
    be the ridge regression estimator, where $\mat X_{1:t}= (X_1,\ldots,X_t)^\top $, $ \mat Y_{1:t}= (Y_1,\ldots,Y_t)^\top $ and $I_d $ is the $d$-dimensional identity matrix. Then with probability at least $1-\delta$, for all $t\ge 0 $,
    \begin{align*}
    \| \bar{\theta}_t-\theta^* \|_{\hat{\mat \Gamma}_t} \le  R \sqrt{2\log\Big(\frac{\det(\hat{\mat \Gamma}_t)^{1/2}\det(\lambda \textbf{I}_d)^{-1/2} }{\delta}\Big)}+\lambda^{1/2}r. 
    \end{align*}
    where $\hat{\mat \Gamma}_t=\lambda \textbf{I}_d+\sum_{u=1}^t X_uX_u^\top$,
    \end{lemma}

\begin{singlespace}
\bibliographystyle{abbrvnat}
\bibliography{refs,bib}

\begin{thebibliography}{64}
\providecommand{\natexlab}[1]{#1}
\providecommand{\url}[1]{\texttt{#1}}
\expandafter\ifx\csname urlstyle\endcsname\relax
  \providecommand{\doi}[1]{doi: #1}\else
  \providecommand{\doi}{doi: \begingroup \urlstyle{rm}\Url}\fi

\bibitem[Abbasi-Yadkori et~al.(2011)Abbasi-Yadkori, P{\'a}l, and
  Szepesv{\'a}ri]{abbasi2011improved}
Y.~Abbasi-Yadkori, D.~P{\'a}l, and C.~Szepesv{\'a}ri.
\newblock Improved algorithms for linear stochastic bandits.
\newblock In \emph{Advances in Neural Information Processing Systems}, pages
  2312--2320, 2011.

\bibitem[Abbasi-Yadkori et~al.(2012)Abbasi-Yadkori, Pal, and
  Szepesvari]{abbasi2012online}
Y.~Abbasi-Yadkori, D.~Pal, and C.~Szepesvari.
\newblock Online-to-confidence-set conversions and application to sparse
  stochastic bandits.
\newblock In \emph{Proceedings of International Conference on Artificial
  Intelligence and Statistics (AISTATS)}, 2012.

\bibitem[Abe et~al.(2003)Abe, Biermann, and Long]{abe2003reinforcement}
N.~Abe, A.~W. Biermann, and P.~M. Long.
\newblock Reinforcement learning with immediate rewards and linear hypotheses.
\newblock \emph{Algorithmica}, 37\penalty0 (4):\penalty0 263--293, 2003.

\bibitem[Agarwal et~al.(2010)Agarwal, Dekel, and Xiao]{agarwal2010optimal}
A.~Agarwal, O.~Dekel, and L.~Xiao.
\newblock Optimal algorithms for online convex optimization with multi-point
  bandit feedback.
\newblock In \emph{Proceedings of annual Conference on Learning Theory (COLT)}.
  Citeseer, 2010.

\bibitem[Agarwal et~al.(2014)Agarwal, Hsu, Kale, Langford, Li, and
  Schapire]{agarwal2014taming}
A.~Agarwal, D.~Hsu, S.~Kale, J.~Langford, L.~Li, and R.~Schapire.
\newblock Taming the monster: A fast and simple algorithm for contextual
  bandits.
\newblock In \emph{Proceedings of International Conference on Machine Learning
  (ICML)}, 2014.

\bibitem[Auer(2002)]{auer2002using}
P.~Auer.
\newblock Using confidence bounds for exploitation-exploration trade-offs.
\newblock \emph{Journal of Machine Learning Research}, 3\penalty0
  (Nov):\penalty0 397--422, 2002.

\bibitem[Auer et~al.(1995)Auer, Cesa-Bianchi, Freund, and
  Schapire]{auer1995gambling}
P.~Auer, N.~Cesa-Bianchi, Y.~Freund, and R.~E. Schapire.
\newblock Gambling in a rigged casino: The adversarial multi-armed bandit
  problem.
\newblock In \emph{Proceedings of the IEEE Symposium on Foundations of Computer
  Science (FOCS)}, 1995.

\bibitem[Auer et~al.(2002)Auer, Cesa-Bianchi, and Fischer]{auer2002finite}
P.~Auer, N.~Cesa-Bianchi, and P.~Fischer.
\newblock Finite-time analysis of the multiarmed bandit problem.
\newblock \emph{Machine Learning}, 47\penalty0 (2-3):\penalty0 235--256, 2002.

\bibitem[Balasubramanian and Ghadimi(2018)]{balasubramanian2018zeroth}
K.~Balasubramanian and S.~Ghadimi.
\newblock Zeroth-order (non)-convex stochastic optimization via conditional
  gradient and gradient updates.
\newblock In \emph{Proceedings of Advances in Neural Information Processing
  Systems (NIPS)}, 2018.

\bibitem[Bastani and Bayati(2015)]{bastani2015online}
H.~Bastani and M.~Bayati.
\newblock Online decision-making with high-dimensional covariates.
\newblock 2015.
\newblock Available at SSRN: \url{https://ssrn.com/abstract=2661896}.

\bibitem[Bertsimas et~al.(2016)Bertsimas, King, and
  Mazumder]{bertsimas2016best}
D.~Bertsimas, A.~King, and R.~Mazumder.
\newblock Best subset selection via a modern optimization lens.
\newblock \emph{The Annals of Statistics}, 44\penalty0 (2):\penalty0 813--852,
  2016.

\bibitem[Besbes et~al.(2015)Besbes, Gur, and Zeevi]{besbes2015non}
O.~Besbes, Y.~Gur, and A.~Zeevi.
\newblock Non-stationary stochastic optimization.
\newblock \emph{Operations Research}, 63\penalty0 (5):\penalty0 1227--1244,
  2015.

\bibitem[Bickel et~al.(2009)Bickel, Ritov, and
  Tsybakov]{bickel2009simultaneous}
P.~J. Bickel, Y.~Ritov, and A.~B. Tsybakov.
\newblock Simultaneous analysis of lasso and dantzig selector.
\newblock \emph{The Annals of Statistics}, 37\penalty0 (4):\penalty0
  1705--1732, 2009.

\bibitem[Blumensath and Davies(2009)]{blumensath2009iterative}
T.~Blumensath and M.~E. Davies.
\newblock Iterative hard thresholding for compressed sensing.
\newblock \emph{Applied and computational harmonic analysis}, 27\penalty0
  (3):\penalty0 265--274, 2009.

\bibitem[Bubeck et~al.(2017)Bubeck, Lee, and Eldan]{bubeck2017kernel}
S.~Bubeck, Y.~T. Lee, and R.~Eldan.
\newblock Kernel-based methods for bandit convex optimization.
\newblock In \emph{Proceedings of the Annual ACM SIGACT Symposium on Theory of
  Computing (STOC)}, 2017.

\bibitem[B{\"u}hlmann and Van De~Geer(2011)]{buhlmann2011statistics}
P.~B{\"u}hlmann and S.~Van De~Geer.
\newblock \emph{Statistics for high-dimensional data: methods, theory and
  applications}.
\newblock Springer Science \& Business Media, 2011.

\bibitem[Candes and Tao(2007)]{candes2007dantzig}
E.~Candes and T.~Tao.
\newblock The dantzig selector: Statistical estimation when p is much larger
  than n.
\newblock \emph{The Annals of Statistics}, 35\penalty0 (6):\penalty0
  2313--2351, 2007.

\bibitem[Carpentier and Munos(2012)]{carpentier2012bandit}
A.~Carpentier and R.~Munos.
\newblock Bandit theory meets compressed sensing for high dimensional
  stochastic linear bandit.
\newblock In \emph{Proceedings of International Conference on Artificial
  Intelligence and Statistics (AISTATS)}, 2012.

\bibitem[Chakraborty et~al.(2010)Chakraborty, Murphy, and
  Strecher]{chakraborty2010inference}
B.~Chakraborty, S.~Murphy, and V.~Strecher.
\newblock Inference for non-regular parameters in optimal dynamic treatment
  regimes.
\newblock \emph{Statistical Methods in Medical Research}, 19\penalty0
  (3):\penalty0 317--343, 2010.

\bibitem[Chu et~al.(2011)Chu, Li, Reyzin, and Schapire]{chu2011contextual}
W.~Chu, L.~Li, L.~Reyzin, and R.~Schapire.
\newblock Contextual bandits with linear payoff functions.
\newblock In \emph{Proceedings of the International Conference on Artificial
  Intelligence and Statistics (AISTATS)}, 2011.

\bibitem[Dani et~al.(2008)Dani, Hayes, and Kakade]{dani2008stochastic}
V.~Dani, T.~P. Hayes, and S.~M. Kakade.
\newblock Stochastic linear optimization under bandit feedback.
\newblock In \emph{Proceedings of annual Conference on Learning Theory (COLT)},
  2008.

\bibitem[Donoho(2006)]{donoho2006compressed}
D.~L. Donoho.
\newblock Compressed sensing.
\newblock \emph{IEEE Transactions on Information Theory}, 52\penalty0
  (4):\penalty0 1289--1306, 2006.

\bibitem[Filippi et~al.(2010)Filippi, Cappe, Garivier, and
  Szepesv{\'a}ri]{filippi2010parametric}
S.~Filippi, O.~Cappe, A.~Garivier, and C.~Szepesv{\'a}ri.
\newblock Parametric bandits: The generalized linear case.
\newblock In \emph{Proceedings of Advances in Neural Information Processing
  Systems (NIPS)}, 2010.

\bibitem[Flaxman et~al.(2005)Flaxman, Kalai, and McMahan]{flaxman2005online}
A.~D. Flaxman, A.~T. Kalai, and H.~B. McMahan.
\newblock Online convex optimization in the bandit setting: gradient descent
  without a gradient.
\newblock In \emph{Proceedings of the annual ACM-SIAM symposium on Discrete
  algorithms (SODA)}, 2005.

\bibitem[Foster et~al.(2016)Foster, Kale, and Karloff]{foster2016online}
D.~Foster, S.~Kale, and H.~Karloff.
\newblock Online sparse linear regression.
\newblock In \emph{Proceedings of annual Conference on Learning Theory (COLT)},
  2016.

\bibitem[Foster et~al.(2018)Foster, Agarwal, Dud{\'\i}k, Luo, and
  Schapire]{foster2018practical}
D.~J. Foster, A.~Agarwal, M.~Dud{\'\i}k, H.~Luo, and R.~E. Schapire.
\newblock Practical contextual bandits with regression oracles.
\newblock In \emph{Proceedings of the International Conference on Machine
  Learning (ICML)}, 2018.

\bibitem[Gerchinovitz(2013)]{gerchinovitz2013sparsity}
S.~Gerchinovitz.
\newblock Sparsity regret bounds for individual sequences in online linear
  regression.
\newblock \emph{Journal of Machine Learning Research}, 14\penalty0
  (Mar):\penalty0 729--769, 2013.

\bibitem[Goldberg and Kosorok(2012)]{goldberg2012q}
Y.~Goldberg and M.~R. Kosorok.
\newblock Q-learning with censored data.
\newblock \emph{Annals of Statistics}, 40\penalty0 (1):\penalty0 529, 2012.

\bibitem[Goldenshluger and Zeevi(2013)]{goldenshluger2013linear}
A.~Goldenshluger and A.~Zeevi.
\newblock A linear response bandit problem.
\newblock \emph{Stochastic Systems}, 3\penalty0 (1):\penalty0 230--261, 2013.

\bibitem[Hammer et~al.(1996)Hammer, Katzenstein, Hughes, Gundacker, Schooley,
  Haubrich, Henry, Lederman, Phair, Niu, et~al.]{hammer1996trial}
S.~M. Hammer, D.~A. Katzenstein, M.~D. Hughes, H.~Gundacker, R.~T. Schooley,
  R.~H. Haubrich, W.~K. Henry, M.~M. Lederman, J.~P. Phair, M.~Niu, et~al.
\newblock A trial comparing nucleoside monotherapy with combination therapy in
  hiv-infected adults with cd4 cell counts from 200 to 500 per cubic
  millimeter.
\newblock \emph{New England Journal of Medicine}, 335\penalty0 (15):\penalty0
  1081--1090, 1996.

\bibitem[Keyvanshokooh et~al.(2019)Keyvanshokooh, Zhalechian, Shi, Van~Oyen,
  and Kazemian]{keyvanshokooh2019contextual}
E.~Keyvanshokooh, M.~Zhalechian, C.~Shi, M.~P. Van~Oyen, and P.~Kazemian.
\newblock Contextual learning with online convex optimization: Theory and
  application to chronic diseases.
\newblock \emph{Available at SSRN}, 2019.

\bibitem[Krause and Ong(2011)]{krause2011contextual}
A.~Krause and C.~S. Ong.
\newblock Contextual gaussian process bandit optimization.
\newblock In \emph{Proceedings of Advances in Neural Information Processing
  Systems (NIPS)}, 2011.

\bibitem[Lai and Robbins(1985)]{lai1985asymptotically}
T.~L. Lai and H.~Robbins.
\newblock Asymptotically efficient adaptive allocation rules.
\newblock \emph{Advances in Applied Mathematics}, 6\penalty0 (1):\penalty0
  4--22, 1985.

\bibitem[Langford et~al.(2009)Langford, Li, and Zhang]{langford2009sparse}
J.~Langford, L.~Li, and T.~Zhang.
\newblock Sparse online learning via truncated gradient.
\newblock \emph{Journal of Machine Learning Research}, 10\penalty0
  (Mar):\penalty0 777--801, 2009.

\bibitem[Lattimore et~al.(2015)Lattimore, Crammer, and
  Szepesv{\'a}ri]{lattimore2015linear}
T.~Lattimore, K.~Crammer, and C.~Szepesv{\'a}ri.
\newblock Linear multi-resource allocation with semi-bandit feedback.
\newblock In \emph{Proceedings of Advances in Neural Information Processing
  Systems (NIPS)}, 2015.

\bibitem[Li et~al.(2010)Li, Chu, Langford, and Schapire]{li2010contextual}
L.~Li, W.~Chu, J.~Langford, and R.~E. Schapire.
\newblock A contextual-bandit approach to personalized news article
  recommendation.
\newblock In \emph{Proceedings of the International Conference on World Wide
  Web (WWW)}. ACM, 2010.

\bibitem[Li et~al.(2011)Li, Chu, Langford, and Wang]{li2011unbiased}
L.~Li, W.~Chu, J.~Langford, and X.~Wang.
\newblock Unbiased offline evaluation of contextual-bandit-based news article
  recommendation algorithms.
\newblock In \emph{Proceedings of the ACM International Conference on Web
  Search and Data Mining (WSDM)}. ACM, 2011.

\bibitem[Li et~al.(2017)Li, Lu, and Zhou]{li2017provably}
L.~Li, Y.~Lu, and D.~Zhou.
\newblock Provably optimal algorithms for generalized linear contextual
  bandits.
\newblock In \emph{Proceedings of International Conference on Machine Learning
  (ICML)}, 2017.

\bibitem[Miller(2002)]{miller2002subset}
A.~Miller.
\newblock \emph{Subset selection in regression}.
\newblock Chapman and Hall/CRC, 2002.

\bibitem[Moodie et~al.(2007)Moodie, Richardson, and
  Stephens]{moodie2007demystifying}
E.~E. Moodie, T.~S. Richardson, and D.~A. Stephens.
\newblock Demystifying optimal dynamic treatment regimes.
\newblock \emph{Biometrics}, 63\penalty0 (2):\penalty0 447--455, 2007.

\bibitem[Murphy(2003)]{murphy2003optimal}
S.~A. Murphy.
\newblock Optimal dynamic treatment regimes.
\newblock \emph{Journal of the Royal Statistical Society: Series B (Statistical
  Methodology)}, 65\penalty0 (2):\penalty0 331--355, 2003.

\bibitem[Murphy(2005)]{murphy2005experimental}
S.~A. Murphy.
\newblock An experimental design for the development of adaptive treatment
  strategies.
\newblock \emph{Statistics in medicine}, 24\penalty0 (10):\penalty0 1455--1481,
  2005.

\bibitem[Natarajan(1995)]{natarajan1995sparse}
B.~K. Natarajan.
\newblock Sparse approximate solutions to linear systems.
\newblock \emph{SIAM Journal on Computing}, 24\penalty0 (2):\penalty0 227--234,
  1995.

\bibitem[Nemirovsky and Yudin(1983)]{nemirovsky1983problem}
A.~S. Nemirovsky and D.~B. Yudin.
\newblock \emph{Problem complexity and method efficiency in optimization.}
\newblock SIAM, 1983.

\bibitem[Orellana et~al.(2010{\natexlab{a}})Orellana, Rotnitzky, and
  Robins]{orellana2010dynamic}
L.~Orellana, A.~Rotnitzky, and J.~M. Robins.
\newblock Dynamic regime marginal structural mean models for estimation of
  optimal dynamic treatment regimes, part {I}: main content.
\newblock \emph{The International Journal of Biostatistics}, 6\penalty0 (2),
  2010{\natexlab{a}}.

\bibitem[Orellana et~al.(2010{\natexlab{b}})Orellana, Rotnitzky, and
  Robins]{orellana2010dynamic2}
L.~Orellana, A.~Rotnitzky, and J.~M. Robins.
\newblock Dynamic regime marginal structural mean models for estimation of
  optimal dynamic treatment regimes, part {II}: proofs of results.
\newblock \emph{The International Journal of Biostatistics}, 6\penalty0 (2),
  2010{\natexlab{b}}.

\bibitem[Pilanci et~al.(2015)Pilanci, Wainwright, and
  El~Ghaoui]{pilanci2015sparse}
M.~Pilanci, M.~J. Wainwright, and L.~El~Ghaoui.
\newblock Sparse learning via boolean relaxations.
\newblock \emph{Mathematical Programming (Series B)}, 151\penalty0
  (1):\penalty0 63--87, 2015.

\bibitem[Raskutti et~al.(2010)Raskutti, Wainwright, and
  Yu]{raskutti2010restricted}
G.~Raskutti, M.~J. Wainwright, and B.~Yu.
\newblock Restricted eigenvalue properties for correlated gaussian designs.
\newblock \emph{Journal of Machine Learning Research}, 11\penalty0
  (Aug):\penalty0 2241--2259, 2010.

\bibitem[Robins et~al.(2008)Robins, Orellana, and
  Rotnitzky]{robins2008estimation}
J.~Robins, L.~Orellana, and A.~Rotnitzky.
\newblock Estimation and extrapolation of optimal treatment and testing
  strategies.
\newblock \emph{Statistics in Medicine}, 27\penalty0 (23):\penalty0 4678--4721,
  2008.

\bibitem[Robins et~al.(2000)Robins, Hernan, and Brumback]{robins2000marginal}
J.~M. Robins, M.~A. Hernan, and B.~Brumback.
\newblock Marginal structural models and causal inference in epidemiology,
  2000.

\bibitem[Rusmevichientong and Tsitsiklis(2010)]{rusmevichientong2010linearly}
P.~Rusmevichientong and J.~N. Tsitsiklis.
\newblock Linearly parameterized bandits.
\newblock \emph{Mathematics of Operations Research}, 35\penalty0 (2):\penalty0
  395--411, 2010.

\bibitem[Shamir(2013)]{shamir2013complexity}
O.~Shamir.
\newblock On the complexity of bandit and derivative-free stochastic convex
  optimization.
\newblock In \emph{Proceedings of annual Conference on Learning Theory (COLT)},
  2013.

\bibitem[Shamir(2015)]{shamir2015complexity}
O.~Shamir.
\newblock On the complexity of bandit linear optimization.
\newblock In \emph{Proceedings of annual Conference on Learning Theory (COLT)},
  2015.

\bibitem[Song et~al.(2015)Song, Wang, Zeng, and Kosorok]{song2015penalized}
R.~Song, W.~Wang, D.~Zeng, and M.~R. Kosorok.
\newblock Penalized {Q}-learning for dynamic treatment regimens.
\newblock \emph{Statistica Sinica}, 25\penalty0 (3):\penalty0 901, 2015.

\bibitem[Szepesvari(2016)]{szepesvari2016banditlog}
C.~Szepesvari.
\newblock Lower bounds for stochastic linear bandits.
\newblock
  \url{http://banditalgs.com/2016/10/20/lower-bounds-for-stochastic-linear-bandits/},
  2016.
\newblock Accessed: Oct 15, 2018.

\bibitem[Tibshirani(1996)]{tibshirani1996regression}
R.~Tibshirani.
\newblock Regression shrinkage and selection via the lasso.
\newblock \emph{Journal of the Royal Statistical Society. Series B
  (Methodological)}, pages 267--288, 1996.

\bibitem[Tropp(2012)]{tropp2012user}
J.~A. Tropp.
\newblock User-friendly tail bounds for sums of random matrices.
\newblock \emph{Foundations of Computational Mathematics}, 12\penalty0
  (4):\penalty0 389--434, 2012.

\bibitem[Wainwright(2009)]{wainwright2009sharp}
M.~J. Wainwright.
\newblock Sharp thresholds for high-dimensional and noisy sparsity recovery
  using $\ell_1$-constrained quadratic programming ({Lasso}).
\newblock \emph{IEEE Transactions on Information Theory}, 55\penalty0
  (5):\penalty0 2183--2202, 2009.

\bibitem[Wang et~al.(2017)Wang, Du, Balakrishnan, and
  Singh]{wang2017stochastic}
Y.~Wang, S.~Du, S.~Balakrishnan, and A.~Singh.
\newblock Stochastic zeroth-order optimization in high dimensions.
\newblock In \emph{Proceedings of International Conference on Artificial
  Intelligence and Statistics (AISTATS)}, 2017.

\bibitem[Watkins and Dayan(1992)]{watkins1992q}
C.~J. Watkins and P.~Dayan.
\newblock Q-learning.
\newblock \emph{Machine learning}, 8\penalty0 (3-4):\penalty0 279--292, 1992.

\bibitem[Zhang et~al.(2012{\natexlab{a}})Zhang, Tsiatis, Davidian, Zhang, and
  Laber]{zhang2012estimating}
B.~Zhang, A.~A. Tsiatis, M.~Davidian, M.~Zhang, and E.~Laber.
\newblock Estimating optimal treatment regimes from a classification
  perspective.
\newblock \emph{Stat}, 1\penalty0 (1):\penalty0 103--114, 2012{\natexlab{a}}.

\bibitem[Zhang et~al.(2012{\natexlab{b}})Zhang, Tsiatis, Laber, and
  Davidian]{zhang2012robust}
B.~Zhang, A.~A. Tsiatis, E.~B. Laber, and M.~Davidian.
\newblock A robust method for estimating optimal treatment regimes.
\newblock \emph{Biometrics}, 68\penalty0 (4):\penalty0 1010--1018,
  2012{\natexlab{b}}.

\bibitem[Zhao et~al.(2012)Zhao, Zeng, Rush, and Kosorok]{zhao2012estimating}
Y.~Zhao, D.~Zeng, A.~J. Rush, and M.~R. Kosorok.
\newblock Estimating individualized treatment rules using outcome weighted
  learning.
\newblock \emph{Journal of the American Statistical Association}, 107\penalty0
  (499):\penalty0 1106--1118, 2012.

\bibitem[Zhao et~al.(2014)Zhao, Zeng, Laber, Song, Yuan, and
  Kosorok]{zhao2014doubly}
Y.-Q. Zhao, D.~Zeng, E.~B. Laber, R.~Song, M.~Yuan, and M.~R. Kosorok.
\newblock Doubly robust learning for estimating individualized treatment with
  censored data.
\newblock \emph{Biometrika}, 102\penalty0 (1):\penalty0 151--168, 2014.

\end{thebibliography}


\begin{thebibliography}{2}
\providecommand{\natexlab}[1]{#1}
\providecommand{\url}[1]{\texttt{#1}}
\expandafter\ifx\csname urlstyle\endcsname\relax
  \providecommand{\doi}[1]{doi: #1}\else
  \providecommand{\doi}{doi: \begingroup \urlstyle{rm}\Url}\fi

\bibitem[Abbasi-Yadkori et~al.(2011)Abbasi-Yadkori, P{\'a}l, and
  Szepesv{\'a}ri]{abbasi2011improved}
Y.~Abbasi-Yadkori, D.~P{\'a}l, and C.~Szepesv{\'a}ri.
\newblock Improved algorithms for linear stochastic bandits.
\newblock In \emph{Advances in Neural Information Processing Systems}, pages
  2312--2320, 2011.

\bibitem[Tropp(2012)]{tropp2012user}
J.~A. Tropp.
\newblock User-friendly tail bounds for sums of random matrices.
\newblock \emph{Foundations of Computational Mathematics}, 12\penalty0
  (4):\penalty0 389--434, 2012.

\end{thebibliography}

\end{singlespace}
\newpage








\end{document}